\documentclass{article}
\PassOptionsToPackage{numbers}{natbib}
\usepackage[main, final]{neurips_2026}

\usepackage[utf8]{inputenc} 
\usepackage[T1]{fontenc}    
\usepackage{hyperref}       
\usepackage{url}            
\usepackage{booktabs}       
\usepackage{amsfonts}       
\usepackage{nicefrac}       
\usepackage{microtype}      
\usepackage{xcolor}         

\usepackage{graphicx} 

\title{Multi-Environment POMDPs with \\ Finite-Horizon Objectives}

\author{
L{\'e}onard Brice \\
Institute of Science and Technology Austria \\
Klosterneuburg, Austria \\
\texttt{leonard.brice@ist.ac.at}
\And
Filip Cano \\
Institute of Science and Technology Austria \\
Klosterneuburg, Austria \\
\texttt{filip.cano@ist.ac.at}
\And
Krishnendu Chatterjee \\
Institute of Science and Technology Austria \\
Klosterneuburg, Austria \\
\texttt{krishnendu.chatterjee@ist.ac.at}
\And
Thomas A. Henzinger \\
Institute of Science and Technology Austria \\
Klosterneuburg, Austria \\
\texttt{tah@ist.ac.at}
\And
Stefanie Muroya \\
Institute of Science and Technology Austria \\
Klosterneuburg, Austria \\
\texttt{stefanie.muroyalei@ist.ac.at}
}

\usepackage[english]{babel}
\usepackage{amsfonts}

\usepackage{mathtools}

\usepackage{amsmath}
\usepackage{amsthm}
\usepackage{graphicx}
\usepackage{cleveref}
\usepackage{thm-restate}

\usepackage{caption}
\usepackage{subcaption}

\usepackage{placeins} 

\usepackage{tikz}
\usetikzlibrary{automata, arrows, positioning, decorations.markings, decorations.pathreplacing}
\usetikzlibrary {decorations.pathmorphing, decorations.shapes}
\usetikzlibrary{backgrounds,fit,shapes.geometric}
\usetikzlibrary {graphs, quotes}
\usetikzlibrary{calc}
\usetikzlibrary{shapes,shapes.geometric,fit}
\usepackage[new]{old-arrows}
\usetikzlibrary {graphs,shapes.geometric,quotes}
\usetikzlibrary{decorations.pathreplacing}
\usetikzlibrary{automata}
\usetikzlibrary{calc}
\usetikzlibrary{arrows,automata,decorations.pathmorphing}
\usetikzlibrary{shapes,shapes.geometric,fit,calc,automata}

\usepackage{algorithm}
\usepackage{algorithmicx}
\usepackage[noend]{algpseudocode}

\newtheorem{problem}{Problem}

\newtheorem{theorem}{Theorem}

\newtheorem{definition}{Definition}

\theoremstyle{remark}
\newtheorem{remark}{Remark}
\newtheorem{example}{Example}

\newcommand{\distribution}[1]{\mathcal{D}(#1)}
\newcommand{\supp}[1]{\mathsf{Supp}(#1)}
\newcommand{\horizon}[0]{k}
\newcommand{\threshold}[0]{\lambda}
\newcommand{\prob}[0]{\mathbb{P}}

\newcommand{\NN}{\mathbb{N}}
\newcommand{\RR}{\mathbb{R}}
\newcommand{\ZZ}{\mathbb{Z}}
\newcommand{\QQ}{\mathbb{Q}}

\newcommand{\POMDP}[0]{\mathcal{P}}
\newcommand{\MEPOMDP}[0]{\mathcal{M}}
\newcommand{\POMDPStates}[0]{S}
\newcommand{\POMDPState}[0]{s}
\newcommand{\POMDPStatebis}[0]{t}
\newcommand{\POMDPObs}[0]{\mathcal{O}}

\newcommand{\obs}[0]{o}
\newcommand{\POMDPActions}[0]{A}
\newcommand{\POMDPAction}[0]{a}
\newcommand{\POMDPMapObs}[0]{\mathsf{obs}}
\newcommand{\POMDPTransF}[0]{\delta}

\newcommand{\POMDPTraj}[0]{\tau}

\newcommand{\allTrajectories}[0]{\mathsf{Traj}}

\newcommand{\obsSeq}[0]{\omega}
\newcommand{\allObsSeqs}[0]{\mathsf{ObsSeq}}

\newcommand{\updateOp}[0]{\mathsf{update}}

\newcommand{\reward}{r}
\newcommand{\payoff}[0]{\mathsf{payoff}}
\newcommand{\MaxReward}{R}
\newcommand{\expected}[0]{\mathbb{E}}

\newcommand{\ConstructPolicy}{\textnormal{\scshape ConstructPolicy}}
\newcommand{\ConstructDetPolicy}{\textnormal{\scshape ConstructDetPolicy}}
\newcommand{\main}{\textnormal{\scshape Main}}
\newcommand{\CheckAchievableValue}{\textnormal{\scshape CheckAchievableValue}}
\newcommand{\auxFunc}{\textnormal{\scshape partialSum}}
\newcommand{\buildFrontier}{\textnormal{\scshape BuildFrontier}}
\newcommand{\prune}{\textnormal{\scshape Prune}}

\newcommand{\belief}[0]{\beta}

\newcommand{\policy}[0]{\sigma}
\newcommand{\policies}[1]{\mathsf{Pol}(#1)}
\newcommand{\detPolicies}[1]{\mathsf{DetPol}(#1)}

\newcommand{\multibelief}{\bar{\beta}}
\newcommand{\MEPPoint}[0]{\bar{x}}
\newcommand{\WPPoint}[0]{x}
\newcommand{\MEP}[0]{\mathsf{mep}}

\newcommand{\PSPACE}{\mathsf{PSPACE}}
\newcommand{\NPSPACE}{\mathsf{NPSPACE}}

\newcommand{\NEXPTIME}{\mathsf{NEXPTIME}}
\newcommand{\InputSize}{\mathcal{N}}

\newcommand{\Conv}{\mathsf{Conv}}

\newcommand{\False}{\mathsf{False}}
\newcommand{\True}{\mathsf{True}}
\newcommand{\isXpoint}{\mathsf{is\bar{X}point}}
\renewcommand{\l}{\ell}
\renewcommand{\epsilon}{\varepsilon}
\renewcommand{\phi}{\varphi}

\newcommand{\RS}{\mathsf{RS}}
\newcommand{\RN}{\mathsf{RN}}
\newcommand{\IFF}{\mathsf{IFF}}

 \usepackage{todonotes}

\usepackage[framemethod=TikZ]{mdframed}

\tikzset{st/.style={state,minimum size=20pt,fill, top color=white, bottom color=black!20}}
\tikzset{stoch/.style={state, fill=black, inner sep=1pt, text=white, minimum size=0pt}}

\begin{document}

\maketitle

\begin{abstract}
  Partially Observable Markov Decision Processes (POMDPs) are systems in which one agent interacts with a stochastic environment, and receives only partial information about the current state.
  In a multi-environment POMDP (MEPOMDP), the initial state is unknown, and assumed to be adversarially chosen. In this work we focus on computing the optimal value and policy in MEPOMDPs with finite-horizon objectives. That problem is known to be $\PSPACE$-complete in POMDPs. The extension is nontrivial: MEPOMDPs require reasoning about beliefs and payoffs across environments, and deterministic policies are no longer enough for optimal payoff. While the lower bound transfers immediately, the upper bound requires new techniques: we prove that mixtures of at most $n$ deterministic policies suffice for $n$ environments and show how to explore them without storing exponential-size policy descriptions. Our main results are as follows: (1) we establish that it is also $\PSPACE$-complete in the more general setting of MEPOMDPs; (2) we present a practical algorithm and evaluate it on classical benchmarks, significantly outperforming the only previously known algorithm.
\end{abstract}

\section{Introduction}

Sequential decision-making under uncertainty is a central challenge in artificial intelligence research.
\emph{Markov Decision Processes} (MDPs) constitute a standard theoretical model to capture settings in which an agent interacts with a stochastic environment.
An MDP consists of a finite set of \emph{states} and a finite set of \emph{actions}; given a state, each action choice stochastically updates the state and gives per-stage rewards.
MDPs have already shown a wide range of applications, from management science to robotics~\cite{puterman1994,sutton2018,thrun2005probabilistic}.
They also have many extensions: we discuss two prominent ones. 

MDPs assume known dynamics. In practice, several models may be plausible without a defensible probability distribution over them. Optimising average performance under an arbitrary prior can conceal poor behaviour in one candidate model; a robust formulation instead seeks one policy that performs well across all candidates.

The first extension of MDPs is known as \emph{partial-observation MDPs} (POMDPs).
While in MDPs the agent always has perfect observation about the state, in POMDPs, the agent only receives some information about the state.
POMDPs are remarkably expressive and have found applications in robotics~\citep{thrun2005probabilistic}, healthcare~\citep{hauskrecht2000value}, and autonomous planning more broadly.
POMDPs capture state uncertainty within one model, not uncertainty over the model itself.

The second extension is to \emph{multi-environment MDPs} (MEMDPs), where instead of a single environment, there are multiple possible environments that the agent is unaware of.
This setting was studied for two environments in~\citep{DBLP:conf/fsttcs/RaskinS14} and later for many environments~\citep{DBLP:conf/icalp/Chatterjee0RS25,vegt2023memdp}.
MEMDPs have also been studied in several applications such as recommander systems and robot navigation~\citep{ICAPS20paper104}. 
In an MEMDP, one environment is selected before execution and remains fixed. The agent may learn about it from the trajectory, but must act robustly before it has been identified.

Combining the above two aspects leads to \emph{multi-environment POMDPs} (MEPOMDPs).
Those were introduced in~\cite{bovymulti}, where it was established that they constitute a robust model that is equivalent to various natural settings, such as adversarial-belief POMDPs or multi-objective POMDPs.

MEPOMDPs combine partial state information with discrete model uncertainty, as when a robot has one of finitely many actuator models or experts provide incompatible models. The environment remains fixed, while observations refine state estimates and eliminate incompatible environments. Finite-horizon computation is relevant for tasks with deadlines or resource budgets, and for approximating discounted objectives by truncation. Exact infinite-horizon value computation is undecidable already for POMDPs~\cite{madani1999undecidability}, making finite horizon the natural setting for exact guarantees.

\paragraph{Contributions.}

In this work, we consider MEPOMDPs with finite-horizon objectives, where the agent's goal is to maximise the reward sum for a finite-horizon length $\horizon$.
We consider both exact and approximate computation of the optimal value the agent can achieve.
Our results are as follows.

\emph{Complexity.} We establish that both the exact and approximate problems are in $\PSPACE$, and therefore $\PSPACE$-complete, using the lower bound known for POMDPs.

\emph{Space-efficient algorithm.} While in the complexity result we present a non-deterministic algorithm, whose determinisation using standard techniques leads to an algorithm with substantial (albeit polynomial) space usage, we present a deterministic space-efficient algorithm.

\emph{Practical algorithm and experiments.} We then complement the above theoretical results with a time-efficient practical algorithm, and test its performance and scalability empirically. Our experiments show that our method obtains better performance than the previously existing state of the art tool.

\paragraph{Related works.}
 
\emph{Complexity of POMDPs.}
The landmark paper of Papadimitriou and Tsitsiklis~\citep{papadimitriou} establishes that the decision version of value computation of POMDPs, under finite horizon, is $\PSPACE$-complete.
In~\citep{madani1999undecidability}, Madani, Hanks, and Condon show undecidability for the infinite-horizon case, with undiscounted and discounted payoff, a result that motivates the focus on finite horizon throughout our work. Exact and approximate finite-horizon POMDP solvers are studied extensively in~\cite{cassandra98,smith2004heuristic}, but they optimise one scalar expectation for one known model. Our result shows that adversarial discrete model uncertainty preserves the classical PSPACE boundary.

 
\emph{Multi-environment MDPs.}
The main works about MEMDPs have been described above~\citep{DBLP:conf/fsttcs/RaskinS14,DBLP:conf/icalp/Chatterjee0RS25,vegt2023memdp,ICAPS20paper104}.
None of those papers consider partial observation.
\cite{DBLP:conf/fsttcs/RaskinS14} introduces robust synthesis for fully observable environments, \cite{DBLP:conf/icalp/Chatterjee0RS25} studies efficient analysis and applications, and~\cite{vegt2023memdp,ICAPS20paper104} give complexity results for reachability and parity objectives. We consider quantitative finite-horizon rewards under partial observation.

 
\emph{Robust POMDPs.}
In~\citep{pmlr-v37-osogami15}, Osogami introduces \emph{robust POMDPs}, i.e. POMDPs equipped with a continuous uncertainty set for the transition and/or observation probabilities.
Multi-environment POMDPs can be seen as a subcase of robust POMDPs, where the uncertainty set is finite.
%
%
\cite{pmlr-v37-osogami15} develops value-iteration and point-based methods for continuous uncertainty. 
In~\cite{krale2025evaluating}, Krale et al. study evaluation of a fixed robust-POMDP policy, reducing it to an equivalent MDP and deriving bounds on its optimality gap. These results do not determine the complexity of exact MEPOMDP value computation. Many robust-POMDP methods also rely on convexity, rectangularity, or dynamic uncertainty. An MEPOMDP instead represents finitely many globally coupled environments, one of which remains fixed throughout the execution; allowing stepwise model changes can therefore be unnecessarily conservative.

\emph{Multi-environment POMDPs.}
Under the name MEPOMDPs, the model was studied by Bovy et al.~\cite{bovymulti}, who propose an exact algorithm for the finite-horizon case and an approximate one for the infinite-horizon discounted case.
However, that paper does not provide any complexity result, and the algorithm proposed is significantly slower than the one we present here, as we detail in \Cref{sec:experiments}. 
The method in~\cite{bovymulti} combines POMDP alpha-vector backups with linear programming over adversarial initial beliefs, adapts heuristic search value iteration for the model. We add tight complexity bounds, a deterministic polynomial-space algorithm, and an algorithm with exponential worst-case complexity that proves to be more performant in practical benchmarks adapted from the POMDP literature.

More recently, Galesloot et al.~\cite{galesloot2025robust} study the same fixed-hidden-model principle under the name hidden-model POMDPs. Their method restricts policies to finite-state controllers and combines deductive worst-case policy evaluation with projected subgradient optimisation. It addresses a complementary, scalable optimisation setting, but does not compute the unrestricted exact finite-horizon value or establish its computational complexity.

Finally, MEPOMDPs can also be seen as a special case of partially observable stochastic games~\cite{Kuhn53,posg}, where the opponent has a non-trivial choice to make at the first round only.
To the best of our knowledge, the exact complexity of the equivalent of our problem, when an opponent is added to the picture and selects an action at each round, is still open.

\section{Background and problem statement}\label{sec:background}
Given a set $S$, we denote the set of probability distributions over $S$ as $\distribution{S}$, and the support of a distribution $d$ as $\supp{d}$.
The set of natural numbers (with 0) is $\NN$.
Given $n \in \NN$, we write $[n] = \{1, \dots, n\}$.
For a set $S$, and an indexing set $I$, when we have defined an element $x_i \in S$ for every $i \in I$, we use the bar notation $\bar{x}$ for the tuple $(x_i)_{i \in I} \in X^I$.
Conversely, when we introduce a tuple $\bar{x}$, the notation $x_i$ refers to the element of index $i \in I$.
We write $\bar{x} \leq \bar{y}$ when the tuple $\bar{y}$ \emph{dominates} $\bar{x}$, i.e., when we have $x_i \leq y_i$ for every index $i$.
Given a set $S \subseteq \RR^n$, we say that $S$ dominates $\bar x$, and write $\bar x\leq S$, when there exists $\bar y\in S$ such that $\bar x\leq \bar y$.
We denote by $\Conv(S) \subseteq \RR^n$ the convex closure of $S$, i.e., the set of points $\bar{x}$ such that there exist points $\bar{y}_1, \dots, \bar{y}_k \in S$ and coefficients $\alpha_1, \dots, \alpha_k \in [0,1]$ with $\sum_i \alpha_i = 1$ such that $\sum_i \alpha_i \bar{y}_i = \bar{x}$.

\subsection{POMDPs}

We now define the model that is considered in this paper.

\begin{definition}[Partially Observable Markov Decision Process (POMDP)] A POMDP is a tuple $\POMDP = \langle\POMDPStates, \POMDPActions, \POMDPObs, \POMDPTransF: \POMDPStates \times \POMDPActions \rightarrow \distribution{\POMDPStates}, \POMDPMapObs: \POMDPStates \rightarrow \POMDPObs, \reward: \POMDPStates \rightarrow \ZZ \rangle$ that consists of 
(i)~a finite set $\POMDPStates$ of states, 
(ii)~a finite set $\POMDPActions$ of actions, 
(iii)~a finite set $\POMDPObs$ of observations, 
(iv)~a stochastic transition function $\POMDPTransF$,
(v)~an observation function $\POMDPMapObs$, and
(vi)~a reward function $\reward$.
\end{definition}

\begin{remark}[Equivalence with other models]
\label{rmk:equivalentpomdps}
    POMDPs are often described with more general observation and reward functions, that depend on states and actions, and have probabilistic outcome.
    However, this is equivalent to the simple model that we consider (see Remark~4 in~\citep{CHATTERJEE2016878}).
\end{remark}

A MDP is a special case of POMDPs where each state has a unique observation.

A \emph{trajectory} of length $\horizon \in \NN$ in the POMDP $\POMDP$ is a finite sequence
$\POMDPTraj = (\POMDPState_0, \POMDPAction_0, \POMDPState_1, \dots, \POMDPAction_{\horizon-1}, \POMDPState_{\horizon}) \in (\POMDPStates \cdot \POMDPActions)^* \cdot \POMDPStates$.
The protagonist of a POMDP aims at generating trajectories that maximise her \emph{payoff}, i.e., the sum of the rewards collected on the way, which we write
$\payoff(\POMDPTraj) = \sum_{\l \leq \horizon} \reward(\POMDPState_\l)$.

An \emph{observation sequence} in $\POMDP$ is a finite sequence
$(\obs_0, \POMDPAction_0, \obs_1, \dots, \POMDPAction_{\horizon-1}, \obs_{\horizon}) \in (\POMDPObs \cdot \POMDPActions)^* \cdot \POMDPObs$.

The set of trajectories in the POMDP $\POMDP$ is denoted by $\allTrajectories({\POMDP})$, and the set of observation sequences by $\allObsSeqs({\POMDP})$.
We define $\allTrajectories_\horizon({\POMDP})$ as the set of trajectories of length $\horizon$.
Each trajectory $\POMDPTraj$ is associated with an observation sequence, denoted by
$\POMDPMapObs(\POMDPTraj) = (
    \POMDPMapObs(\POMDPState_0), 
    \POMDPAction_0, 
    \POMDPMapObs(\POMDPState_1), 
    \cdots,
    \POMDPAction_{\horizon-1},
    \POMDPMapObs(\POMDPState_\horizon)
    )$.

\begin{figure}[t]
            \centering
			\begin{tikzpicture}[->,>=latex,shorten >=1pt, initial text={}, scale=1, every node/.style={scale=0.8, transform shape}, node distance=20mm]
				\node[initial above, st] (s1) at (0, 0) {$s_1$};
				\node[initial above, st] (s2) at (7, 0) {$s_2$};
                
                \node[st, below left of=s1] (t1) {$t_1$};
                \node[st, double, below right of=s1] (t2) {$t_2$};
                \node[stoch, left of=s1] (stoch1) {};
                \node[stoch, right of=s1] (stoch2) {};

                \node[st, above left of=stoch1] (t3) {$t_3$};
                \node[st, below left of=stoch1, double] (t4) {$t_4$};
                \node[st, above right of=stoch2] (t5) {$t_5$};
                \node[st, below right of=stoch2, double] (t6) {$t_6$};
                
                \node[st, double, below left of=s2] (t7) {$t_7$};
                \node[st, below right of=s2] (t8) {$t_8$};
                \node[stoch, left of=s2] (stoch3) {};
                \node[stoch, right of=s2] (stoch4) {};

                \node[st, above left of=stoch3] (t9) {$t_9$};
                \node[st, below left of=stoch3, double] (t10) {$t_{10}$};
                \node[st, above right of=stoch4] (t11) {$t_{11}$};
                \node[st, below right of=stoch4, double] (t12) {$t_{12}$};

                \path (s1) edge node[below right] {$a$} (t1);
                \path (s1) edge node[below left] {$b$} (t2);
                \path (s1) edge node[above] {$c$} (stoch1);
                \path (s1) edge node[above] {$d$} (stoch2);
                \path (s2) edge node[below right] {$a$} (t7);
                \path (s2) edge node[below left] {$b$} (t8);
                \path (s2) edge node[above] {$c$} (stoch3);
                \path (s2) edge node[above] {$d$} (stoch4);

                \path (stoch1) edge node[above right] {$0.1$} (t3);
                \path (stoch1) edge node[below right] {$0.9$} (t4);
                \path (stoch2) edge node[above left] {$0.4$} (t5);
                \path (stoch2) edge node[below left] {$0.6$} (t6);
                \path (stoch3) edge node[above right] {$0.4$} (t9);
                \path (stoch3) edge node[below right] {$0.6$} (t10);
                \path (stoch4) edge node[above left] {$0.1$} (t11);
                \path (stoch4) edge node[below left] {$0.9$} (t12);

                \path (t1) edge[loop below] (t1);
                \path (t2) edge[loop below] (t2);
                \path (t3) edge[loop left] (t3);
                \path (t4) edge[loop left] (t4);
                \path (t5) edge[loop right] (t5);
                \path (t6) edge[loop right] (t6);
                \path (t7) edge[loop below] (t7);
                \path (t8) edge[loop below] (t8);
                \path (t9) edge[loop left] (t9);
                \path (t10) edge[loop left] (t10);
                \path (t11) edge[loop right] (t11);
                \path (t12) edge[loop right] (t12);

                \node[draw, dashed, rounded corners, fit=(s1)(s2)(t1)(t2)(t5)(t6)(t7)(t8)(t9)(t10)(t11)(t12), inner xsep=20mm, inner ysep=5mm] {};
                \node[right of=stoch4] (o1) {$o_1~~~~$};

                \node[draw, dashed, rounded corners, fit=(t3)(t4), inner xsep=5mm, inner ysep=5mm] {};
                \node[left of=stoch1] (o2) {$~~~~o_2$};
			\end{tikzpicture}
            \vspace{-2em}
			\caption{An example of a POMDP}
			\label{fig:ex_randomisation}
            \end{figure}
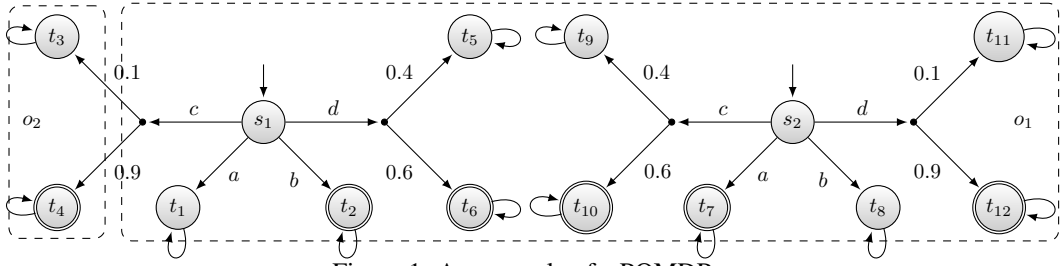

\begin{example}
    \Cref{fig:ex_randomisation} depicts an example of a POMDP.
    In this POMDP, when action $c$ is performed from the state $s_1$, the state $t_3$ is reached with probability $0.1$, and the state $t_4$ with probability $0.9$.
    The dashed boxes indicate that all states have the same observation $o_1$, except the states $t_3$ and $t_4$, with observation $o_2$.
    The double circles indicate states with reward 1; the other ones have reward 0.
\end{example}

\paragraph{Policies.} A {\em policy} in the POMDP $\POMDP$ is a function $\policy: \allObsSeqs({\POMDP})  \to \distribution{\POMDPActions}$.
It is \emph{deterministic} if for every observation sequence $\obsSeq \in \allObsSeqs({\POMDP})$, there is an action $\POMDPAction$ with $\policy(\obsSeq)(\POMDPAction) = 1$.
We denote the set of all policies on $\POMDP$ as $\policies{\POMDP}$, and
the set of deterministic policies as $\detPolicies{\POMDP}$.

Given a POMDP $\POMDP$ and an initial state $\POMDPState \in \POMDPStates$,
each policy $\policy$ defines a probability distribution $\prob^\policy_{\POMDPState}$ over trajectories of length $\horizon$, for each \emph{horizon} $\horizon \in \NN$.
Thus, we can define the expected payoff obtained when following that policy: we write it $\expected_{\POMDPState}^\horizon(\policy)$.
Throughout this paper, we focus on policy behaviors within a well-defined finite horizon $\horizon$: we therefore assume that policies are described only on observations sequences of length bounded by $\horizon$, which uses space $2^{O(\horizon)}$.

\subsection{Beliefs}

A classical tool to study POMDPs is the notion of \emph{belief}: from a known initial state, after having performed a sequence of actions and seen a sequence of observations, the agent can keep track of the probability of being in each state.

\begin{definition}[Belief]
    A \emph{belief} in the POMDP $\POMDP$ is a probability distribution $\belief$ over states, such that  for all $\POMDPState, \POMDPState'$ in $\supp{\belief}$, we have that $\POMDPMapObs(\POMDPState) = \POMDPMapObs(\POMDPState')$.
\end{definition}

We sometimes abuse notations and write $\POMDPMapObs(\belief)$ for the unique observation associated to the belief $\belief$.
When the agent has the belief $\belief$, performs an action $\POMDPAction$ and sees the observation $\obs$, she must update her belief to reflect her knowledge about which state may have now be reached.
We denote that updated belief by $\updateOp(\belief, \POMDPAction, \obs)$, defined for each state $\POMDPState$ by:
$$\updateOp(\belief, \POMDPAction, \obs)(\POMDPState) = 
    \frac{\sum_{\POMDPState' \in \POMDPStates} \belief(\POMDPState') \cdot \POMDPTransF(\POMDPState', \POMDPAction)( \POMDPState)}
    {\sum_{\substack{\POMDPState'' \in \POMDPStates \\   \POMDPMapObs(\POMDPState'')= \obs} } \sum_{\POMDPState' \in \POMDPStates} \belief(\POMDPState') \cdot \POMDPTransF(\POMDPState', \POMDPAction)( \POMDPState'')}$$
if $\POMDPMapObs(\POMDPState) = \obs$, and $\updateOp(\belief, \POMDPAction, \obs)(\POMDPState) = 0$ otherwise.


\subsection{Multi-environment POMDP}

We can now define the model we focus on throughout this paper.

\begin{definition}[Multi-environment POMDP]
    A \emph{multi-environment POMDP} is a pair $\MEPOMDP = \langle \POMDP, (\POMDPState_i)_{i \in [n]} \rangle$, with $n \in \NN$, and where $\POMDPState_i$ is a state of $\POMDP$ for each $i$.
\end{definition}

\begin{remark}[Equivalence with other models]
    There are several ways to define ME-POMDPs.
    A first one would be to consider a single initial state but different transition functions $\delta_1,\dots,\delta_n$. 
    A second one is to consider different reward functions.
    A third one is to consider multiple initial states, or more generally initial beliefs.
    Those three models are introduced in \citep{bovymulti}, under the names \emph{multi-environment}, \emph{multi-objective}, and \emph{adversarial belief POMDPs}, respectively.
    The same paper establishes that all of them are equivalent, hence we chose the technically simplest one, with multiple initial states (i.e., adversarial belief) and refer to it as ME-POMDPs for consistency with the literature of MEMDPs.   
\end{remark}

\subsection{Problem statement}

Given an MEPOMDP $\MEPOMDP$, we define the \emph{value} of $\MEPOMDP$ as $v = \sup_{\policy\in \policies{\POMDP}} \min_{i \in [n]} \expected^\horizon_{s_i}(\policy)$.

\begin{problem}[Value problem] \label{pb:value}
    Given an MEPOMDP $\MEPOMDP$ and a horizon $\horizon$, compute the value $v$.
\end{problem}

We also consider the approximate version of that problem, where we are given a quantity $\epsilon$, and ask for any quantity in the interval $[v - \epsilon, v + \epsilon]$.
For complexity purposes, we also consider the decision version of these problems.

\begin{problem}[Threshold problem] \label{pb:threshold}
    Given an MEPOMDP $\MEPOMDP$, a horizon $\horizon$, and a threshold $\threshold \in \QQ$, do we have $v \geq \lambda$?
\end{problem}

The approximate threshold problem is obtained by considering only instances that come with the guarantee that the quantity $\threshold$ is at distance at least $\epsilon$ from the value---or equivalently, by considering that both answers are acceptable on instances that do not satisfy this guarantee.

\paragraph*{Problem encoding.}
For complexity results, all items in the instance are given with classical encodings, except for the horizon $\horizon$, assumed to be given in unary.
This choice is motivated by the fact that the case $n=1$ was also studied under this hypothesis~\citep{papadimitriou}, and that to our knowledge, the complexity of that problem when $\horizon$ is given in binary remains unknown.
In the rest of this paper, we assume that we are given an instance of this problem, whose total size is denoted by $\InputSize$.

\paragraph*{The need for randomised policies.}
While in POMDPs the value coincides with optimising over $\detPolicies{\POMDP}$, for MEMDPs, randomised policies are more powerful: there exist MEPOMDPs where we have $v > \sup_{\sigma \in \detPolicies{\POMDP}} \expected^\horizon_{s_i}(\policy)$.
This is already known for multi-environment MDPs~\cite{DBLP:conf/fsttcs/RaskinS14}.

\begin{example}
    Consider the POMDP depicted by \Cref{fig:ex_randomisation}, with horizon $1$ and initial states $s_1$ and $s_2$.
    Remember that the double circles indicate a reward $1$, and that the other states give reward $0$.
    With a single initial state, a deterministic policy would be sufficient to maximise the expected payoff: action $b$ from $s_1$, action $a$ from $s_2$.
    However, if the initial vertex can be either $s_1$ or $s_2$, none of those policies is satisfying.
    On the other hand, randomising between actions $a$ and $b$ with probability $0.5$ each guarantees expected payoff $0.5$.
    Better: the policy that randomises uniformly between actions $c$ and $d$ guarantees an expected payoff $0.75$, even though the deterministic policies that play deterministically either $c$ or $d$ are not optimal neither from $s_1$ nor from $s_2$.
\end{example}

In the next section, we define a few conceptual tool that we use to solve this problem.



\section{Multi-belief, multi-expected payoff, and mixtures of policies}\label{sec:tools}
\subsection{Multi-beliefs}

The notion of belief is not sufficient to solve our problem.
Indeed, with an adversarially chosen initial state, the agent must consider separately the $n$ \emph{environments} in which the initial state was $\POMDPState_1, \dots, \POMDPState_n$, respectively; and for each of them, construct a different belief.
When performing actions and seeing observations, the agent can now get information of two kinds: (i) in each environment $i$, she must at each step update her belief; but (ii) some observation sequences may also be incompatible with some initial states, in which case the corresponding environment must be eliminated.
    For example, in the POMDP depicted by \Cref{fig:ex_randomisation}, if after performing action $c$, the agent sees the observation $o_2$, she knows that the initial state was $\POMDPState_1$ and must no longer consider the environment $2$.

\begin{definition}[Multi-belief]
    A \emph{multi-belief} in the MEPOMDP $\MEPOMDP$ is a tuple $\multibelief = (\belief_i)_{i \in [n]}$, where each $\belief_i$ is either a belief or the symbol $\bot$, such that all beliefs $\belief_i \neq \bot$ have the same observation.
\end{definition}

The equality $\belief_i = \bot$ means that the environment $i$ is eliminated, i.e., that the agent learned that $\POMDPState_i$ was not the initial state.
For convenience, we write $\bot(s)=\bot$ for every $s \in S$.
The initial multi-belief is $\multibelief^0$, defined by $\belief^0_i(\POMDPState_i) = 1$ for each $i$.
The update function is extended to multi-beliefs by:
$$\updateOp(\multibelief, \POMDPAction, \obs) = \left(\begin{cases}
    \bot & \text{if } \belief_i = \bot \text{ or } \sum_{\POMDPState \in \POMDPStates} \sum_{\substack{\POMDPStatebis \in \POMDPStates \\ \POMDPMapObs(\POMDPStatebis) = \obs}} \belief_i(\POMDPState) \cdot \POMDPTransF(\POMDPState, \POMDPAction)( \POMDPStatebis) = 0 \\
    \updateOp(\belief_i, \POMDPAction, \obs) & \text{otherwise}
\end{cases}\right)_{i\in [n]}$$

\subsection{Multi-expected payoff}

Similarly, the notion of expected payoff must be replaced by that of \emph{multi-expected payoff}, defined as tuples of real values.
We also allow some coordinates to be set to $\bot$, for policies that are played starting from a multi-belief in which some environments have already been eliminated.
For simplicity, we use the conventions $\bot \cdot 0 = 0$ and $\bot + x = \bot$ and $\bot \cdot x = \bot$ for every other $x$.

\begin{definition}[Multi-expected payoff]
    A \emph{multi-expected payoff} is a tuple in $(\RR \cup \{\bot\})^n$.
    Given a multi-belief $\multibelief$, policy $\policy$, and horizon $\l$, their multi-expected payoff is 
    $\MEP_{\multibelief}^\l(\policy) = \left( \sum_s \beta_i(s) \expected^\l_{s}(\policy)\right)_i$.
\end{definition}

When $\l = \horizon$ is clear from context, we write  $\MEP_{\multibelief}(\policy)$ instead of $\MEP_{\multibelief}^\l(\policy)$.
\Cref{pb:threshold} can now be rephrased as follows: in the POMDP $\POMDP$, does there exists a policy $\policy$ such that $\MEP_{\multibelief^0}(\policy) \in \left[\lambda, +\infty\right)^n$?

\subsection{Mixtures of policies}

Finally, we will often need to define policies as a randomised combination of several policies.

\begin{definition}[Mixture of policies]
    Given a POMDP $\POMDP$, some policies $\policy_1,\cdots, \policy_{\l} \in \detPolicies\POMDP$, and some coefficients $\alpha_1,\cdots, \alpha_\l \in [0,1]$ with $\sum_{i \leq \l} \alpha_i = 1$, we denote by $\sum_{i \leq \l} \alpha_i \policy_i$ the policy $\policy$ that, for each $i$, follows the policy $\policy_i$ with probability $\alpha_i$.
    The policy $\policy$ is then a \emph{mixture} of $\policy_1, \dots, \policy_\l$.
    The \emph{degree} of $\policy$ is the minimal $\l \in \NN$ such that $\policy$ is a mixture of $\l$ deterministic policies.
\end{definition}

\begin{remark}\label{rk_max_degree}
    Every policy has degree at most $|\POMDPActions|^{|\POMDPObs|^\horizon}$.
\end{remark}


\begin{remark}\label{rk:mixed-winning-prob}
    Let $\policy = \sum_{i=1}^{\l}\alpha_i\policy_i$ be a mixture of policies and $\multibelief^0$ be an initial product belief. 
    Then, we have that $\MEP(\policy, \multibelief^0) = \sum_i \alpha_i\cdot \MEP(\policy_i, \multibelief^0)$.
\end{remark}

\section{Complexity}\label{sec:complexity}
In this section, we show that the threshold problem and its approximate version are $\PSPACE$-complete.
The lower bound is a result of~\cite{papadimitriou}, but the upper bound requires new techniques.

First, we note that it is sufficient to consider policies with a degree at most the number of initial states.
This is a consequence of Carathéodory's theorem~\cite{leonard2016geometry}: each policy is a mixture of deterministic policies, and its corresponding multi-expected payoff is a convex combination of their multi-expected payoffs. Since the multi-expected payoff of an arbitrary policy is a point in $\mathbb R^n$, Carathéodory's theorem states that said convex combination can be written using only $n+1$ of the original points.

\begin{restatable}[App.~\ref{app:mixed_policy_existence}]{lemma}{lmMixedPolicyExistence}\label{lemma:mixed_policy_existence}
    Let $\POMDP$ be a POMDP, let $\threshold \in \QQ$, and $\POMDPState_1, \dots, \POMDPState_n$ be $n$ initial states.
    Let us assume that there exists a policy $\policy$ such that from every initial state $\POMDPState_i$, following the policy $\policy$, the expected payoff is at least $\threshold$.
    Then, there exists a policy of degree at most $n$ that offers the same guarantee.
\end{restatable}

This result should convey the intuition of why value computation in MEPOMDPs is not fundamentally harder than in POMDPs: although deterministic strategies are not sufficient here, \Cref{lemma:mixed_policy_existence} shows that only a limited amount of randomisation is necessary.

From this result, a first naive algorithm can be proposed.
By \Cref{lemma:mixed_policy_existence}, to recognize positive instances of \Cref{pb:threshold}, it suffices to guess $n$ deterministic policies $\policy_1, \dots, \policy_n$ and to check that they can be mixed into a satisfying policy, i.e., that we have $\Conv\{\MEPPoint^1, \dots, \MEPPoint^n\} \cap [\threshold, 1]^n \neq \emptyset$.
Every deterministic policy can be described with the list of possible observation sequences, and for each of them, the action that is performed.
Such a description requires exponential space.
Once those policies are guessed, computing the associated multi-expected payoffs amounts to computing the expected payoffs in $n$ Markov chains, which can be done in time polynomial in the size of the Markov chains, i.e., exponential in the size of the original POMDP.
Checking whether we have $\Conv\{\MEPPoint^1, \dots, \MEPPoint^n\} \cap \left[\threshold, +\infty\right)^n \neq \emptyset$ can then be done in polynomial time.
This algorithm shows therefore membership to the class $\NEXPTIME$.
To obtain a better upper bound, we note that we do not need the whole description of each deterministic policy, but just their multi-expected payoff, which we can represent using only polynomial space, as formalized in~\Cref{lem:exponentialpoints}.
Therefore, we can guess policies one action at a time, and propagate the multi-expected payoffs instead of storing the whole policy.
Such an algorithm also requires the following lemma.

\begin{restatable}[App.~\ref{app:exponentialpoints}]{lemma}{lmExponentialPoints}
\label{lem:exponentialpoints}
    Let $\MaxReward = \max_{\POMDPState \in \POMDPStates} |\reward(\POMDPState)|$.
    Let  $X=\{ \MEP_{\multibelief^0}(\policy)   \:\mid\: \policy\in \detPolicies{\POMDP}  \}$.
    There exists a polynomial $\pi$ and an integer $\mathcal C\leq 2^{\pi(\InputSize)}$ such that the set $N=  \left\{ p/\mathcal C \:\left|\: p\in \left\{ - \mathcal C k\MaxReward, \dots, \mathcal C k\MaxReward\right.\right\} \right\}$ contains $X$. Thus, $N$ can be represented using $\pi(\InputSize)+2\InputSize+1$ bits.
\end{restatable}



\begin{figure}[H]
\centering

\begin{minipage}[t]{0.48\textwidth}
\vspace{0pt}
We present now a non-deterministic polynomial-space algorithm to solve our problem, illustrated by \Cref{fig:illu_algo_npspace}.
The core idea is as follows: instead of guessing deterministic policies as a whole, we guess them action by action, in a depth-first fashion.
Then, each time a subpolicy (all actions planned after a given observation sequence) has been entirely guessed, the corresponding multi-expected payoff is computed, and the subpolicy is erased.
This algorithm is described precisely in \Cref{algo:npspace}.
The following result follows from this algorithm.

\begin{restatable}[App.~\ref{app:pspaceeasy}]{lemma}{lmPspaceEasy}\label{lm:pspaceeasy}
    The threshold problem and its approximate version are in $\PSPACE$.
\end{restatable}
\end{minipage}
\hfill
\begin{minipage}[t]{0.48\textwidth}
\vspace{0pt}
\centering
\begin{tikzpicture}[scale=0.7, decorate,
        decoration={brace,amplitude=10pt}, node distance=0.5mm]
        \node[state,minimum size=10pt,fill, top color=white, bottom color=black!20] (s0) at (0,1) {};
        \draw (0,1) node[right] {$~\multibelief^0$};
        \node[stoch] (stoch0) at (0,-1) {};
        \path[->] (s0) edge node[left] {\scriptsize{(1) Guess $\POMDPAction$}} (stoch0);
        
        \node[state,minimum size=10pt,fill, top color=white, bottom color=black!20] (s1) at (-1.5,-3) {};
        \draw (-1.5,-3) node[right] {$~\multibelief^{\obs_1}$};
        \node[state,minimum size=10pt,fill, top color=white, bottom color=black!20] (s2) at (0,-3) {};
        \node (dots1) at (1.5,-3) {\dots};
        \path[->] (stoch0) edge node[above left] {$\obs_1$} (s1);
        \path[->] (stoch0) edge node[right] {$\obs_2$} (s2);
        \draw (-3, -1.5) node {\scriptsize{\shortstack{(2) Enumerate\\all observations}}};
        \draw (-3.5, -3) node {\scriptsize{\shortstack{(3) Compute the\\updated belief}}};

        \draw[dashed] (s1) -- (-2,-5) -- (-1,-5) -- (s1);
        \draw (-1.5, -4.7) node {$\MEPPoint^{\obs_1}$};
        \draw (-3.5, -4.5) node {\scriptsize{(4) Recursive call}};

        \draw[decorate]
        (1.8,1.3) -- (1.8,-5)
        node[midway, right=10pt, align=left] {\scriptsize{\shortstack{(5) Compute the aggregated\\multi-expected payoff,\\return it, and erase $\POMDPAction$}}};
		\end{tikzpicture}
        \vspace{-1em}
	    \caption{An illustration of \Cref{algo:npspace}} \label{fig:illu_algo_npspace}
\end{minipage}
\end{figure}
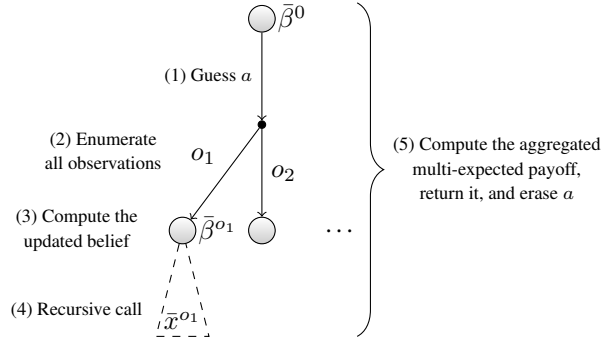

\begin{algorithm}[t]
            \begin{algorithmic}[1]\caption{A non-deterministic polynomial-space algorithm}\label{algo:npspace}
                \Procedure{ConstructPolicy}{$\POMDP, (\POMDPState_i)_{i \in [n]}, \horizon$}
                    \State $\multibelief^0 \gets \bigl(\mathbf{1}_{\{s = \POMDPState_i\}}\bigr)_{i \in [n]}$
                    \For{$i \in \{1, \dots, n\}$}
                        \State $\MEPPoint^i \gets \ConstructDetPolicy(\POMDP, \horizon, \multibelief^0)$
                    \EndFor
                    \State\Return{$\Conv\{\MEPPoint^1, \dots, \MEPPoint^n\} \cap \left[\lambda, +\infty\right)^n \neq  \emptyset$}
                \EndProcedure\\

                \Procedure{ConstructDetPolicy}{$\POMDP, \l, \multibelief$}
                    \If{$\l = 0$}
                        \State \Return{$\left(\sum_{\POMDPState \in \POMDPStates} \belief_i(\POMDPState) \reward(\POMDPState)\right)_{i \in [n]}$}
                    \EndIf
                        \State Guess $\POMDPAction \in \POMDPActions$
                        \Comment{(1)}
                        \For{$\obs \in \POMDPObs$}
                        \Comment{(2)}
                            \State $\multibelief^\obs \gets \updateOp(\multibelief, \POMDPAction, \obs)$
                            \Comment{(3)}
                            \State $\MEPPoint^\obs \gets \ConstructDetPolicy(\POMDP, \l-1, \multibelief^\obs)$
                            \Comment{(4)}
                        \EndFor
                        \State $\MEPPoint \gets \left(
                        \sum_{s\in\POMDPStates} \belief_i(s) \left( \reward(s) + \sum_{o\in \POMDPObs}\sum_{t\in S, \POMDPMapObs(t)=o} \POMDPTransF(s,a)(t)\cdot \bar{x}^o_i\right)\right)_{i\in [n]}
                        $
                        \State \Return{$\MEPPoint$}
                        \Comment{(5)}
                        \label{xupdate}
                \EndProcedure
            \end{algorithmic}
        \end{algorithm}

Finally, the matching lower bound for the exact problem follows from \citep{papadimitriou} where it was shown for $n=1$.
We briefly show that the proof can be adapted to show the same hardness for the approximate version.
We can then conclude on our main theoretical result.

\begin{restatable}[App.~\ref{app:hardness_approx}]{lemma}{lmHardnessApprox}\label{lm:hardness_approx}
    The approximate threshold problem is $\PSPACE$-hard, even when $n=1$.
\end{restatable}

\begin{theorem}
    The threshold problem and its approximate version are $\PSPACE$-complete.
\end{theorem}

\section{Algorithms}\label{sec:algorithms}
We have proved \Cref{pb:threshold} to be $\PSPACE$-complete by providing a non-deterministic algorithm. 
In this section, we present two deterministic algorithms, one that optimises space, and another one that optimizes efficiency in practice for tractable instances, even if it may use exponential time and space in the worst case.

\subsection{A space-efficient deterministic algorithm}
In this section, we assume an enumeration $\POMDPObs = \{o_1, \dots, o_{|\POMDPObs|}\}$.
\Cref{algo:pspace} is a space-efficient determinisation of \Cref{algo:npspace}. Its rationale is as follows.
The $\main$ procedure enumerates all possible multi-expected payoff values $\bar X = (\bar x^1,\dots, \bar x^n)$ achievable by $n$ deterministic policies, mimicking the non-deterministic guess of points in $\ConstructPolicy$.
    The $\CheckAchievableValue$ procedure checks whether the value $\bar x$ is achievable as a 
    multi-expected payoff
    in the POMDP $\POMDP$, from a multi-belief $\multibelief$ and a horizon of $\l$ steps. Here, achievable means that there exists a deterministic policy $\sigma$ such that $\expected^\horizon_{\POMDPState_i}(\policy) = \bar x$. 
    The procedure $\CheckAchievableValue$, in line with $\ConstructDetPolicy$, checks whether there is a policy that achieves a multi-expected payoff $\bar x$ from the multi-belief $\bar \beta$ in $\l$ steps by checking, for every pair $(\POMDPAction, \obs)\in \POMDPActions\times\POMDPObs$
    the multi-expected payoffs that can be achieved with a horizon $\l-1$ and a multi-belief $\bar\beta^\obs=\updateOp(\bar\beta, \POMDPAction, \obs)$, and aggregating them according to the Bellman update (Eq.~\eqref{eq:bellmanupdate}).
    To do so, it constructs partial sums of Eq.~\ref{eq:bellmanupdate} recursively, checking, for every $j\in \{1,\dots, |\POMDPObs|\}$ and every possible value $\bar x_{\mathsf{rem}}\in X^n$, whether $\bar x_{\mathsf{rem}}$ can be obtained as the Bellman update partial sum from $o_j$ onward, i.e., whether it can be achieved as a multi-expected payoff in a horizon of $\l-1$ steps from the updated multi-belief $\updateOp(\multibelief,a,o_j)$, and for those values where it is possible, whether there is a way to fill up the remaining $|\POMDPObs|-j+1$ terms in the partial sum to make it up to $\bar x_{\mathsf{rem}}$.\footnote{Note that when $\bar{x}_\mathsf{rem}$ is updated, for each $i$ such that $x_i = \bot$, we have $y_i = 0$ and the coordinate $x_{\mathsf{rem}i}$ is left unchanged.}



\begin{algorithm}
            \begin{algorithmic}[1]\caption{A deterministic polynomial-space algorithm}\label{algo:pspace}
                \Procedure{Main}{$\POMDP, (\POMDPState_i)_{i \in [n]}, \horizon$}
                    \State $\multibelief^0 \gets \bigl(\mathbf{1}_{\{s = \POMDPState_i\}}\bigr)_{i \in [n]}$
                    \For{$\bar X = (\bar x^1,\dots, \bar x^n)\in N^{n^2}$}
                    \State $\isXpoint\gets \forall i\in [n],  \CheckAchievableValue(\POMDP, k, \multibelief^0, \bar x^i)$
                    \If{$\isXpoint \:\land\: (\Conv\{\MEPPoint^1, \dots, \MEPPoint^n\} \cap [\lambda, +\infty)^n \neq \emptyset) $}
                    \Return{$\True$}
                    \EndIf
                    \EndFor
                    \State \Return{$\False$}
                \EndProcedure\\

                \Procedure{CheckAchievableValue}{$\POMDP, \l, \multibelief, \bar x$}
                 \If{$\l = 0$}
                        \Return{ $\bar x =\left(\sum_{\POMDPState \in \POMDPStates} \reward(\POMDPState)\belief_i(\POMDPState)\right)_{i\in [n]}$}
                    \EndIf
                    \For{$a\in A$}
                        \If{$\auxFunc(1,\l,\multibelief,a, \bar x)$} \Return{$\True$}
                        \EndIf
                    \EndFor
                    \State\Return{$\False$}
                \EndProcedure\\

                \Procedure{\auxFunc}{$j, \l, \multibelief, a,\bar x_{\mathsf{rem}}$}
                    
                    \If{$j=|\POMDPObs|+1$}
                    \Return{$\bar x_{\mathsf{rem}}=\left(\sum_{\POMDPState \in \POMDPStates} \reward(\POMDPState)\belief_i(\POMDPState)\right)_{i\in [n]}$}
                    \EndIf
                    \For{$\bar x\in (N \cup \{\bot\})^n$}
                    \If{$\CheckAchievableValue(\POMDP, \l-1, \updateOp(\multibelief, a,o_j), \bar x)$}
                    \State $\bar y \gets \left(
                            \sum_{\POMDPState \in \POMDPStates}  
                            \sum_{\POMDPStatebis \in \POMDPStates,\: \POMDPMapObs(\POMDPStatebis) = \obs_j} \belief_i(\POMDPState) \cdot \POMDPTransF(\POMDPState, \POMDPAction)(\POMDPStatebis) \cdot x_i \right)_{i\in [n]}
                        $
                    \If{$\auxFunc(j+1, \l, \multibelief, a, \bar x_{\mathsf{rem}}-\bar y)$}
                    \Return{$\True$}
                    \EndIf
                    \EndIf
                    \EndFor
                    
                    \State\Return{$\False$}
                \EndProcedure

    \end{algorithmic}
\end{algorithm}

                    
                    
                    

\begin{restatable}[App.~\ref{app:space-pspace}]{theorem}{thmSpacePspace}\label{thm:space-pspace}
    \Cref{algo:pspace} decides \Cref{pb:threshold} and uses $O(|\POMDPStates|\cdot |\POMDPObs|\cdot n \cdot \horizon \cdot  \pi(\InputSize))$ space.
\end{restatable}

\begin{remark}
    The space requirement of \Cref{algo:npspace} is $O(|\POMDPStates|\cdot n \horizon \pi(\InputSize))$, so the space requirement of the standard determinization using Savitch's theorem~\cite{SAVITCH1970177} would be $O(|\POMDPStates|^2\cdot n^2 \horizon^2 \pi(\InputSize)^2)$, significantly larger than that of \Cref{algo:pspace}.
\end{remark}

Algorithm~\ref{algo:pspace} is thus efficient in terms of space.
However, this comes at the cost of several imbricated loops over all values in $N$, which results in practice in unfeasibly large computation times for most instances.
For that reason, we now propose a second method that uses exponential space in the worst case, but is much more time-efficient in practice.


\subsection{Practically efficient algorithm}
\label{sec:efficient_alg}

Algorithm~\ref{algo:efficient} implements a different approach: computing the set of multi-expected payoffs bottom-up for increasing horizons, and pruning, at each step, the points that are dominated by other points in the set, corresponding to families of policies that do not need to be explored.

Consider the set $X^\l(\multibelief) = \{ \MEP_{\multibelief}^\l(\policy)  \:\mid\: \policy\in \detPolicies{\POMDP}   \}$.
We can construct $X^\l$ recursively as:
\begin{equation}
\label{eq:Xrecursive}
    X^\l(\multibelief) = \Bigg\{  \Bigg(\sum_{s\in S} \belief_i(s) \bigg( \reward(s) + \sum_{o \in \POMDPObs} \hspace{-0.5em} \sum_{\substack{t\in S,\\\POMDPMapObs(t)=o}} \hspace{-0.5em}   \POMDPTransF(s,a)(t) x^o_i\bigg)\Bigg)_{i}
    \:\Bigg|\:
    \begin{matrix}
        \bar x^o \in X^{\l-1}(\multibelief^{a,o}),\\
        \multibelief^{a,o} = \updateOp(\multibelief, \POMDPAction, \obs),
        \\
        a\in\POMDPActions
    \end{matrix}
     \Bigg\}
\end{equation}




As a result of Lemma~\ref{lemma:mixed_policy_existence}, solving \Cref{pb:threshold} is equivalent to checking whether $\Conv(X^\horizon(\multibelief^0))\cap [\lambda,\infty)^n\neq \emptyset$.
Therefore, it is sufficient to keep only a representation of the convex hull of $X^\horizon(\multibelief^0)$.
Since we are interested in the intersection with $[\lambda,\infty)^n$, we can also do without any point convexly dominated by points in $X^\l(\multibelief^0)$.
Also, when building $X^\l(\multibelief)$ from corresponding $X^{\l-1}(\multibelief')$ following Eq.~\ref{eq:Xrecursive}, we can prune all dominated points in an $\l-1$ depth before building the set for depth $\l$.

\begin{algorithm}
            \begin{algorithmic}[1]\caption{An efficient algorithm}\label{algo:efficient}
                \Procedure{BuildFrontier}{$\l, \multibelief$}
                \If{$\l=0$}
                \Return{$\{\left(\sum_{\POMDPState \in \POMDPStates} \reward(\POMDPState)\belief_i(\POMDPState)\right)_{i\in [n]}\}$}
                \EndIf
                \State $V\gets \emptyset$
                \For{$a\in \POMDPActions$}
                \For{$o\in\POMDPObs$}
                \State $V_{o} \gets \buildFrontier(\l-1, \updateOp(\multibelief,a,o))$
                \EndFor
                \For{$(\bar x^o)_{o\in \POMDPObs}\in \prod_{o\in \POMDPObs} V_o $}
                \State $\bar y \gets  \left(\sum_{s\in\POMDPStates} \belief_i(s) \left( \reward(s) + \sum_{o\in \POMDPObs}\sum_{t\in S, \POMDPMapObs(t)=o} \POMDPTransF(s,a)(t)\cdot \bar{x}^o_i\right)\right)_{i\in [n]}$
                \State $V\gets \prune(V,\bar y)$
                \EndFor
                \EndFor
                \State \Return $V$
                \EndProcedure\\
                \Procedure{Prune}{$V, \bar y$}
                \For{$\bar x\in V$}
                \If{$\bar y$ is dominated by $\bar x$}
                \Return $V$
                \EndIf
                \If{$\bar x$ is dominated by $\bar y$}
                $V\gets V\setminus \{\bar x\}$
                \EndIf
                \EndFor
                \State \Return{$V\cup \{\bar y\}$}
                \EndProcedure
            \end{algorithmic}
\end{algorithm}

\section{Experimental Evaluation}\label{sec:experiments}

In our experimental evaluation, we study the practical efficiency and scalability of our efficient method (Algorithm~\ref{algo:efficient}) and compare it to the state of the art. 
We tackle two research questions.


\hspace{1em} \textbf{(RQ1) Baseline comparison.} How efficient is our algorithm compared to the state of the art~\cite{bovymulti}? 

\hspace{1em} \textbf{(RQ2) Scalability.} How does the computational cost of our method scale with different parameters, such as model size and horizon?

In both the tables and the figures, we report $|\POMDPStates|$ to be only the number of reachable states within the set horizon.
All the experiments were executed using 16GB of RAM and a processor Intel Xeon CPU E5-2680 v3 2.50GHz. Timeout is set at $3600s$.

\paragraph{Benchmarks.} We consider three benchmarks in our experiments. The first one is the extension of \emph{Rock Sample} to ME-POMDPs introduced in~\cite{bovymulti}.
For the second and third benchmarks, we extend two classical problems from~\cite{cassandra98} to ME-POMDPs: the \emph{Robot Navigation} and the \emph{Identification (Friend or Foe)} problem.
Instances of the Rock Sample problem are denoted by $\RS_{m,g,t}$, where $m$ is the size of the grid, $g$ is the number of good rocks, and $t$ is the total number of rocks. 
Instances of the Robot Navigation problem are denoted by $\RN_{\textsc{Map},d}$, where $\textsc{Map}$ indicates the map or layout, and $d$ indicates the maximum distance of initial states to a goal.
Instances of the Identification problem are denoted by $\IFF_{d_1, d_2,v_1,v_2}$, where $d_i$ and $v_i$ indicate the distance and visibility, respectively, of the two initial foe states.
The full description of the three benchmarks is in Appendix~\ref{sec:experiments-more}.

\paragraph{Comparison to baseline.}
To address RQ1, we compare our method with the one proposed in~\cite{bovymulti}, using their own tool. Since~\cite{bovymulti} is an approximate algorithm, we set it to have no approximation error, to match our exact algorithm.
In the appendix, Sec.~\ref{sec:comparisonepsilon},
we show that relaxing the approximation error of~\cite{bovymulti} has no significant effect in runtime.
We report the results on a representative subset of the three benchmarks in Tables~\ref{tab:rock-sample-main} and~\ref{tab:cassandra-main}.
Observe that our method is in general faster than~\cite{bovymulti}, and manages to successfully solve problems with up to 574 states for horizon up to $\horizon=5$.

\begin{table}[t]
\centering
\small

\begin{minipage}{0.48\textwidth}
\centering
\caption{Rock Sampling.}
\label{tab:rock-sample-main}
\begin{tabular}{l@{\hspace{3pt}}r@{\hspace{3pt}}r@{\hspace{3pt}}r@{\hspace{3pt}}r@{\hspace{3pt}}r@{\hspace{3pt}}r@{\hspace{4pt}}r@{\hspace{4pt}}}
\toprule
Model & $|\POMDPStates|$ & $|\POMDPActions|$ & $|\POMDPObs|$ & $n$ & $\horizon$ & Time \cite{bovymulti} (s) & Time ours (s) \\
\midrule
$\RS_{3, 1, 2}$ & 26 & 7 & 3 & 2 & 4 & 1.736 & 0.2335 \\
$\RS_{3, 1, 3}$ & 34 & 8 & 3 & 3 & 3 & 2.434 & 0.05812 \\
$\RS_{4, 1, 2}$ & 32 & 7 & 3 & 2 & 7 & \textsc{to} & 154.2 \\
$\RS_{6, 1, 2}$ & 35 & 7 & 3 & 2 & 6 & 43.17 & 18.68 \\
$\RS_{3, 2, 7}$ & 246 & 12 & 3 & 21 & 5 & \textsc{to} & 459.9 \\
$\RS_{3, 4, 7}$ & 574 & 12 & 3 & 35 & 5 & \textsc{to} & 783.8 \\
$\RS_{3, 6, 7}$ & 192 & 12 & 3 & 7 & 6 & \textsc{to} & 3477 \\
\bottomrule
\end{tabular}

\end{minipage}
\hfill
\begin{minipage}{0.48\textwidth}
\centering
\caption{Robot Navigation and Identification.}
\label{tab:cassandra-main}
\begin{tabular}{l@{\hspace{3pt}}r@{\hspace{3pt}}r@{\hspace{3pt}}r@{\hspace{3pt}}r@{\hspace{3pt}}r@{\hspace{10pt}}r@{\hspace{6pt}}r@{\hspace{3pt}}}
\toprule
Model & $|\POMDPStates|$ & $|\POMDPActions|$ & $|\POMDPObs|$ & $n$ & $\horizon$ & \cite{bovymulti} (s) & Ours (s) \\
\midrule
$\RN_{\textsc{cit}, 3}$ & 48 & 4 & 28 & 2 & 2 & 304.2 & 0.0223 \\
$\RN_{\textsc{cit}, 3}$ & 48 & 4 & 28 & 2 & 6 & \textsc{to} & 914.1 \\
$\RN_{\textsc{mit}, 3}$ & 36 & 4 & 28 & 2 & 5 & \textsc{to} & 70.02 \\
$\RN_{\textsc{pen.}, 2}$ & 35 & 4 & 28 & 2 & 6 & \textsc{to} & 902.1 \\
$\IFF_{1, 2, 2, 4}$ & 15 & 4 & 22 & 3 & 6 & \textsc{to} & 271 \\
$\IFF_{1, 2, 4, 2}$ & 17 & 4 & 22 & 3 & 1 & \textsc{to} & 0.1388 \\
$\IFF_{1, 3, 0, 2}$ & 24 & 4 & 22 & 3 & 3 & 213.5 & 0.6066 \\
\bottomrule
\end{tabular}

\end{minipage}

\end{table}

\paragraph{Scalability.}
To investigate RQ2, we compare computation times in terms of horizon, number of states and number of initial states. In Fig.~\ref{fig:fig1} (left), we show computation time for increasing horizons for the models in Table~\ref{tab:rock-sample-main}. In Fig.~\ref{fig:fig1} (right), we report a scatter plot of computation time and horizon for the whole RockSample benchmark (see Appendix~\ref{sec:experiments-more} for description), color-coded by number of initial states. 
While there is a general trend that computation time increases slightly both with number of states and number of initial states, the strongest predictor of computation time (and thus, the most salient challenge for scalability) is horizon ($\horizon$).

\begin{figure}
     \centering
     \begin{subfigure}[b]{0.48\textwidth}
         \centering
         \includegraphics[width=\textwidth]{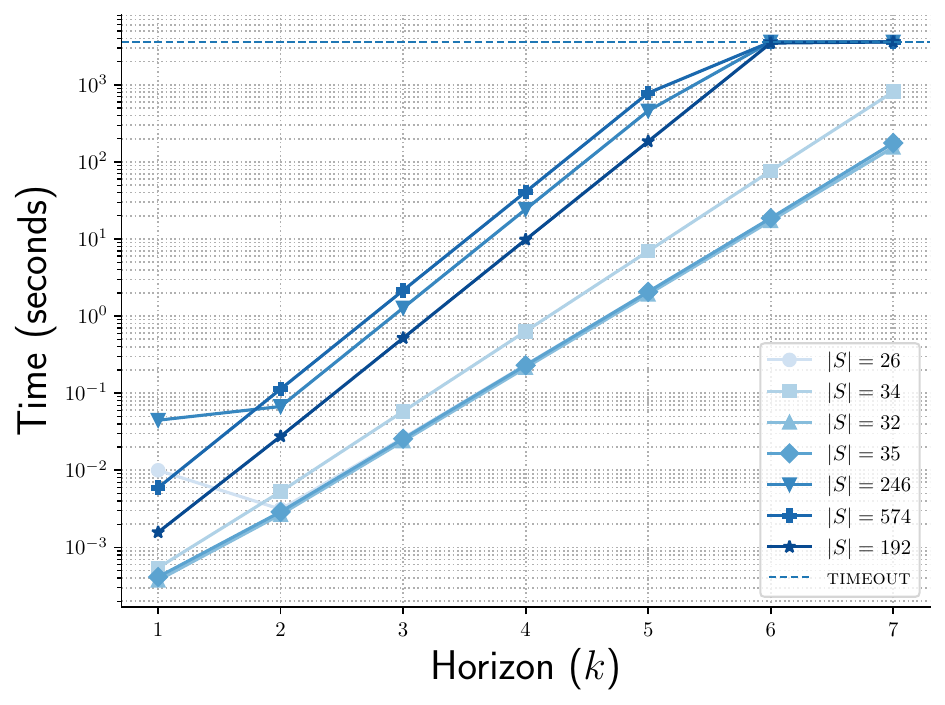}
     \end{subfigure}
     \begin{subfigure}[b]{0.48\textwidth}
         \centering
         \includegraphics[width=\textwidth]{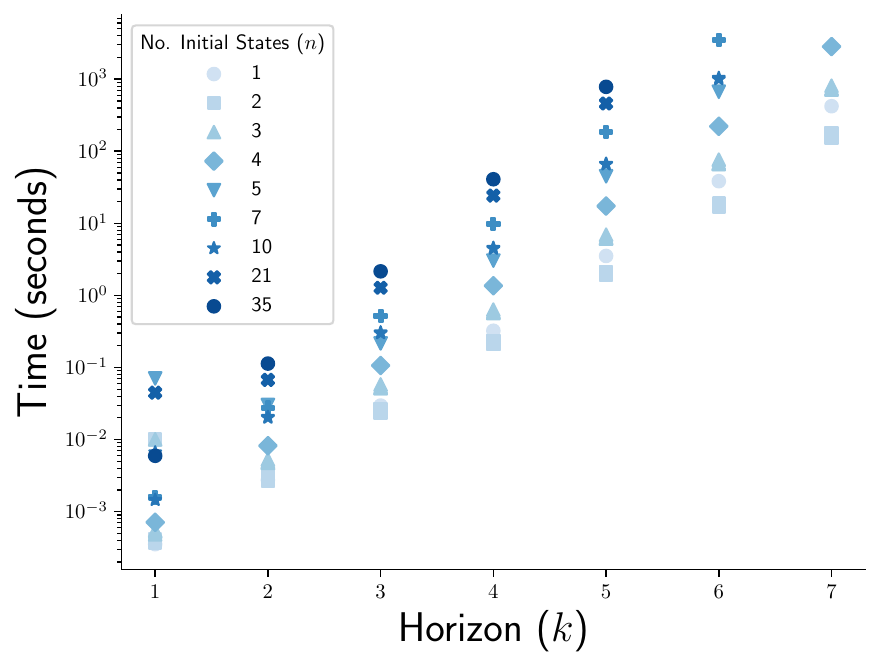}
     \end{subfigure}
        \caption{Horizon vs. Computation time in the RockSample benchmark, for different number of states (left), and different number of initial states (right).}
        \label{fig:fig1}
\end{figure}

\section{Conclusion}

We have considered the value computation problem in MEPOMDPs with finite-horizon objectives, 
providing tight complexity bounds, a space efficient algorithm, and a time efficient algorithm that proves to be empirically more performant than previously existing tools.

\paragraph{Limitations.}
From a theoretical perspective, our methods present a usual tradeoff between theoretical and practical performance:
Algorithm~\ref{algo:pspace}, while asymptotically optimal, is unfeasibly slow in practice; and Algorithm~\ref{algo:efficient}, while efficient in practice, uses exponential space in the worst case.
From a practical perspective, the main challenge for Algorithm~\ref{algo:efficient} is scalability. In particular, the dependence on the horizon is quite strong, with the solved instances with longest horizon at $\horizon =7$.

\paragraph{Future work.}
Along with finite-horizon objectives, one can also consider discounted-sum objectives.
The exact problem for discounted-sum objectives is already undecidable for POMDPs~\cite{madani1999undecidability}, but its approximate version can be reduced to finite-horizon objectives: our algorithm can therefore easily be applied (up to adding a discount factor) to solve that problem as well.
We leave the theoretical and practical complexities of this problem as future work.

\begin{ack}
This work was supported by the European Research Council under Grant No.: ERC-2020-AdG 101020093 (VAMOS), by ERC CoG 863818 (ForMSMArt), and Austrian Science Fund (FWF) 10.55776/COE12. This work was also supported by the ISTA Responsible AI Program, made possible through the support of Garrett Camp and the Camp Foundation.
\end{ack}

\bibliography{biblio}

@article{papadimitriou,
  author    = {Papadimitriou, Christos H. and Tsitsiklis, John N.},
  title     = {The Complexity of {M}arkov Decision Processes},
  journal   = {Mathematics of Operations Research},
  volume    = {12},
  number    = {3},
  pages     = {441--450},
  year      = {1987},
  publisher = {INFORMS},
  ISSN      = {0364765X, 15265471},
  URL       = {http://www.jstor.org/stable/3689975},
}

@inproceedings{bovymulti,
  author    = {Bovy, Eline and Probine, Caleb and Suilen, Marnix and Topcu, Ufuk and Jansen, Nils},
  title     = {Multi-Environment {POMDP}s: Discrete Model Uncertainty Under Partial Observability},
  booktitle = {The Thirty-ninth Annual Conference on Neural Information Processing Systems},
  year      = {2025},
}

@inproceedings{Kuhn53,
        AUTHOR = "H. Kuhn",
        TITLE = "Extensive games and the problem of information",
        BOOKTITLE = "Contributions to the Theory of Games",
 	EDITOR = "H. Kuhn and A. Tucker",
     	PUBLISHER = "Princeton University Press",
        YEAR = 1953,
        PAGES = "193-216"
        }

@article{posg,
author = {Hansen, Eric and Bernstein, Daniel and Zilberstein, Shlomo},
year = {2004},
month = {06},
pages = {},
title = {Dynamic Programming for Partially Observable Stochastic Games},
journal = {AAAI Workshop - Technical Report}
}

@article{thrun2005probabilistic,
  author       = {Nicholas Roy and
                  Geoffrey J. Gordon and
                  Sebastian Thrun},
  title        = {Finding Approximate {POMDP} solutions Through Belief Compression},
  journal      = {J. Artif. Intell. Res.},
  volume       = {23},
  pages        = {1--40},
  year         = {2005},
  url          = {https://doi.org/10.1613/jair.1496},
  doi          = {10.1613/JAIR.1496},
  bibsource    = {dblp computer science bibliography, https://dblp.org}
}

@article{hauskrecht2000value,
  author       = {Milos Hauskrecht and
                  Hamish S. F. Fraser},
  title        = {Planning treatment of ischemic heart disease with partially observable
                  {M}arkov decision processes},
  journal      = {Artif. Intell. Medicine},
  volume       = {18},
  number       = {3},
  pages        = {221--244},
  year         = {2000},
  url          = {https://doi.org/10.1016/S0933-3657(99)00042-1},
  doi          = {10.1016/S0933-3657(99)00042-1},
  bibsource    = {dblp computer science bibliography, https://dblp.org}
}

@inproceedings{DBLP:conf/fsttcs/RaskinS14,
  author       = {Jean{-}Fran{\c{c}}ois Raskin and
                  Ocan Sankur},
  editor       = {Venkatesh Raman and
                  S. P. Suresh},
  title        = {Multiple-Environment {M}arkov Decision Processes},
  booktitle    = {34th International Conference on Foundation of Software Technology
                  and Theoretical Computer Science, {FSTTCS} 2014, New Delhi, India,
                  December 15-17, 2014},
  series       = {LIPIcs},
  pages        = {531--543},
  publisher    = {Schloss Dagstuhl - Leibniz-Zentrum f{\"{u}}r Informatik},
  year         = {2014},
  url          = {https://doi.org/10.4230/LIPIcs.FSTTCS.2014.531},
  doi          = {10.4230/LIPICS.FSTTCS.2014.531},
  bibsource    = {dblp computer science bibliography, https://dblp.org}
}

@article{madani1999undecidability,
title = {On the undecidability of probabilistic planning and related stochastic optimization problems},
journal = {Artificial Intelligence},
volume = {147},
number = {1},
pages = {5-34},
year = {2003},
note = {Planning with Uncertainty and Incomplete Information},
issn = {0004-3702},
doi = {https://doi.org/10.1016/S0004-3702(02)00378-8},
url = {https://www.sciencedirect.com/science/article/pii/S0004370202003788},
author = {Omid Madani and Steve Hanks and Anne Condon}
}

@InProceedings{vegt2023memdp,
author="van der Vegt, Marck
and Jansen, Nils
and Junges, Sebastian",
editor="Sankaranarayanan, Sriram
and Sharygina, Natasha",
title="Robust Almost-Sure Reachability in Multi-Environment {MDP}s",
booktitle="Tools and Algorithms for the Construction and Analysis of Systems",
year="2023",
publisher="Springer Nature Switzerland",
address="Cham",
pages="508--526",
isbn="978-3-031-30823-9"
}

@InProceedings{pmlr-v37-osogami15,
  title = 	 {Robust partially observable {M}arkov decision process},
  author = 	 {Osogami, Takayuki},
  booktitle = 	 {Proceedings of the 32nd International Conference on Machine Learning},
  pages = 	 {106--115},
  year = 	 {2015},
  editor = 	 {Bach, Francis and Blei, David},
  volume = 	 {37},
  series = 	 {Proceedings of Machine Learning Research},
  address = 	 {Lille, France},
  month = 	 {07--09 Jul},
  publisher = {PMLR},
  url = 	 {https://proceedings.mlr.press/v37/osogami15.html}
}

@inproceedings{krale2025evaluating,
  title={On Evaluating Policies for Robust {POMDP}s},
  author={Krale, Merlijn and Bovy, Eline M and Galesloot, Maris FL and Sim{\~a}o, Thiago D and Jansen, Nils},
  booktitle={The Thirty-ninth Annual Conference on Neural Information Processing Systems},
  year={2025}
}

@phdthesis{cassandra98,
    author = {Cassandra, Anthony Rocco},
    advisor = {Kaelbling, Leslie Pack},
    title = {Exact and approximate algorithms for partially observable {M}arkov decision processes},
    school = {Brown University, {USA}},
    year = {1998},
    isbn = {0591833220},
    publisher = {Brown University}
}

@inproceedings{DBLP:conf/icalp/Chatterjee0RS25,
  author       = {Krishnendu Chatterjee and
                  Laurent Doyen and
                  Jean{-}Fran{\c{c}}ois Raskin and
                  Ocan Sankur},
  editor       = {Keren Censor{-}Hillel and
                  Fabrizio Grandoni and
                  Jo{\"{e}}l Ouaknine and
                  Gabriele Puppis},
  title        = {The Value Problem for Multiple-Environment {MDP}s with Parity Objective},
  booktitle    = {52nd International Colloquium on Automata, Languages, and Programming,
                  {ICALP} 2025, Aarhus, Denmark, July 8-11, 2025},
  series       = {LIPIcs},
  pages        = {150:1--150:17},
  publisher    = {Schloss Dagstuhl - Leibniz-Zentrum f{\"{u}}r Informatik},
  year         = {2025},
  url          = {https://doi.org/10.4230/LIPIcs.ICALP.2025.150},
  doi          = {10.4230/LIPICS.ICALP.2025.150},
  bibsource    = {dblp computer science bibliography, https://dblp.org}
}

@inproceedings{ICAPS20paper104,
  title     = {Multiple-Environment {M}arkov Decision Processes: Efficient Analysis and Applications},
  author    = {Krishnendu Chatterjee and Martin Chmel{\'{\i}}k and Deep Karkhanis and Petr Novotn{\'{y}} and Am{\'{e}}lie Royer},
  booktitle = {Proceedings of the 30th International Conference on Automated Planning and Scheduling ({ICAPS})},
  publisher = {{AAAI} Press},
  pages     = {48--56},
  year      = {2020}
}

@book{puterman1994,
  author    = {Puterman, Martin L.},
  title     = {{M}arkov Decision Processes: Discrete Stochastic Dynamic
               Programming},
  publisher = {John Wiley \& Sons},
  address   = {New York},
  year      = {1994},
  isbn      = {978-0-471-72782-8},
  doi       = {10.1002/9780470316887},
}

@book{sutton2018,
  author    = {Sutton, Richard S. and Barto, Andrew G.},
  title     = {Reinforcement Learning: An Introduction},
  edition   = {2nd},
  publisher = {The MIT Press},
  address   = {Cambridge, Massachusetts},
  year      = {2018},
  isbn      = {978-0-262-03924-6},
  url       = {http://incompleteideas.net/book/the-book-2nd.html},
}

@article{CHATTERJEE2016878,
title = {What is decidable about partially observable {M}arkov decision processes with $\omega$-regular objectives},
journal = {Journal of Computer and System Sciences},
volume = {82},
number = {5},
pages = {878-911},
year = {2016},
issn = {0022-0000},
doi = {https://doi.org/10.1016/j.jcss.2016.02.009},
url = {https://www.sciencedirect.com/science/article/pii/S0022000016000246},
author = {Krishnendu Chatterjee and Martin Chmelík and Mathieu Tracol}
}

@article{SAVITCH1970177,
title = {Relationships between nondeterministic and deterministic tape complexities},
journal = {Journal of Computer and System Sciences},
volume = {4},
number = {2},
pages = {177-192},
year = {1970},
issn = {0022-0000},
doi = {https://doi.org/10.1016/S0022-0000(70)80006-X},
author = {Walter J. Savitch}
}

@book{leonard2016geometry,
  author       = {I. E. Leonard and J. E. Lewis},
  title        = {Geometry of Convex Sets},
  publisher    = {Wiley-Blackwell},
  address      = {Hoboken, New Jersey, United States},
  year         = {2016},
  edition      = {1st},
  isbn13       = {9781119022664},
  isbn10       = {1119022665},
  pages        = {336},
}

@inproceedings{smith2004heuristic,
  title={Heuristic search value iteration for {POMDP}s},
  author={Smith, Trey and Simmons, Reid},
  booktitle={Proceedings of the 20th conference on Uncertainty in artificial intelligence},
  pages={520--527},
  year={2004}
}

@article{galesloot2025robust,
  title={Robust finite-memory policy gradients for hidden-model POMDPs},
  author={Galesloot, Maris FL and Andriushchenko, Roman and {\v{C}}e{\v{s}}ka, Milan and Junges, Sebastian and Jansen, Nils},
  journal={International Joint Conference on Artificial Intelligence (IJCAI)},
  year={2025}
}

\appendix

 \section{Proof of \Cref{lm:hardness_approx}} \label{app:hardness_approx}

\lmHardnessApprox*

\begin{proof}
    The proof in \citep{papadimitriou} goes by reducing the problem QBF to \Cref{pb:threshold} in polynomial time.
    From a formula of the form $\phi = \exists x_1 \forall x_2 \dots \forall x_p \psi$, where $\psi$ is a quantifier-free formula in disjunctive normal form, it constructs a POMDP in which getting expected payoff at least $0$ is possible if and only if the formula $\phi$ is true.
    A careful reading of this construction also shows that if $\phi$ is not true, then every policy gives expected payoff not only negative, but at most $-\left(1/2\right)^{p/2}$.
    By defining $\epsilon = \left(1/2\right)^{p/2}$ as equal to this quantity, we have a polynomial-time reduction to the approximate threshold problem.
\end{proof}

\section{Proof of \Cref{lemma:mixed_policy_existence}}\label{app:mixed_policy_existence}

\lmMixedPolicyExistence*

\begin{proof}
    By \Cref{rk_max_degree}, the policy $\policy$ is a mixture of a finite number of deterministic policies.
    Using \Cref{rk:mixed-winning-prob}, we can therefore write:
    $$\MEP(\policy) \in \Conv\left(\MEP\left(\detPolicies{\POMDP}\right)\right).$$
    Now, by Carathéodory's theorem, there exist $n+1$ deterministic policies $\policy'_1, \dots, \policy'_{n+1} \in \detPolicies{\POMDP}$ and coefficients $\beta_1, \dots, \beta_{n+1}$ such that $\sum_{i=1}^{n+1} \beta_i = 1$ and $\sum_{i=1}^{n+1} \beta_i \MEP(\policy'_i) = \MEP(\policy)$.

    Consider now the polyhedron $X = \Conv\{ \MEP(\policy'_1), \dots, \MEP(\policy'_{n+1})\}$.
    The mapping $f: \MEPPoint \mapsto \min_i \WPPoint_i$ finds its maximum on a facet of $X$, i.e., the convex hull of at most $n$ of its vertices.
    We can therefore choose $n$ deterministic policies $\policy_1, \dots, \policy_n \in \{\policy'_1, \dots, \policy'_{n+1}\}$, and $n$ coefficients $\alpha_1, \dots, \alpha_n$ with $\sum_i \alpha_i = 1$, such that we have $\min_i \sum_{i=1}^n \alpha_i \MEP(\policy_i) \geq \lambda$.
    By \Cref{rk:mixed-winning-prob}, the policy $\policy^\star = \sum_i \alpha_i \policy_i$ offers then the desired guarantees.
\end{proof}

\section{Proof of \Cref{lem:exponentialpoints}} \label{app:exponentialpoints}

\lmExponentialPoints*

\begin{proof}
    For each triple $(s,a,s')\in \POMDPStates\times \POMDPActions\times S$ with a non-zero transition probability,
    consider the standard representation of the transition probability $\POMDPTransF(s,a)[s'] = p_{s,a,s'}/q_{s,a,s'}$, where $p_{s,a,s'}$ and $q_{s,a,s'}$ are positive coprime integers. Let $Q = \{ q_{s,a,s'}\:\mid\: (s,a,s')\in \POMDPStates\times \POMDPActions\times \POMDPStates \}$, $M=\max Q$ and $C$ be the least common multiple of $Q$.
    Since the size of the input is $\InputSize$, $M$ has at most $\InputSize$ bits, so it is at most $2^\InputSize$.
    Note that 
    \[C\leq \prod_{q\in Q}q \leq M^{|\POMDPStates|^2\cdot|\POMDPActions|} \leq 2^{\InputSize \cdot |\POMDPStates|^2\cdot|\POMDPActions|}\leq 2^{\InputSize^4}.
    \]
    Every transition probability can be rewritten as 
    $\delta(s,a)(s')=\tilde p_{s,a,s'}/C$, for suitable integers 
    $\tilde p_{s,a,s'} \leq C$.

    Consider a deterministic policy $\policy\in\detPolicies{P}$, a belief $\belief$, and a trajectory $\POMDPTraj = (\POMDPState_0, \POMDPAction_0, \dots, \POMDPState_{\horizon}) \in \allTrajectories_\horizon(\POMDP)$. 
    The probability of $\POMDPTraj$ is $\prob^{\policy}_{\belief}(\POMDPTraj)=\belief(\POMDPState_0)
    \prod_{i=0}^{\horizon-1}\POMDPTransF(\POMDPState_i, \POMDPAction_i)(\POMDPState_{i+1})$ if $\POMDPTraj$ is compatible with $\policy$, and $0$ otherwise. 
    If $\belief = \belief^0_i$ for some $i\in [n]$,
    then $\belief(\POMDPState_0)\in \{0,1\}$, and 
    since all factors can be expressed with denominator $C$, the probability of $\POMDPTraj$ is of the form 
    \[
    \prob^{\policy}_{\belief}(\POMDPTraj)=\frac{m_\POMDPTraj}{C^k},
    \]
    for some integer $m_\POMDPTraj \leq C^k$, and the associated payoff can be expressed as $\payoff(\tau)\frac{m_\POMDPTraj}{C^k}$. 
    Since all rewards are integers, $\payoff(\tau)$ is an integer, $-kR \leq \payoff(\tau) \leq kR$. 

    Let $\mathcal C = C^k$.
    The expected payoff $\expected^\horizon_{\belief}(\policy)$  is the sum of probabilities of trajectories times their payoff, so it can be expressed as a rational $p/\mathcal C$, where $p$ is an integer between $-kRC^k$ and $kRC^k$.
    Since $\mathcal C \leq 2^{k\cdot \InputSize^4} \leq 2^{\InputSize^5}$, and $kR\leq 2^{2\InputSize}$ we have that $X\subseteq N$ with $\pi(\InputSize) = \InputSize^5$, and all numbers in $N$ are expressed by an integer in $[0,\mathcal C]$, and integer in $[0,\horizon R]$, and a sign, for a total of $\pi(\InputSize)+2\InputSize+1$.
\end{proof}

\section{Proof of \Cref{lm:pspaceeasy}} \label{app:pspaceeasy}

\lmPspaceEasy*

\begin{proof}
    We show that \Cref{algo:npspace} is correct, and that it uses only polynomial memory.

    If we have a positive instance of \Cref{pb:threshold}, then by \Cref{lemma:mixed_policy_existence}, there exists a policy $\policy$ of degree at most $n$ such that $\MEP(\policy) \in \left[\lambda, +\infty\right)^n$; in other words, there exist $n$ deterministic policies $\policy_1, \dots, \policy_n$ such that $\Conv\{\MEP(\policy_1), \dots, \MEP(\policy_n)\} \cap \left[\lambda, +\infty\right)^n \neq \emptyset$.
    Consider the run of \Cref{algo:npspace} where, in the main procedure, the $i$th call to the procedure $\ConstructDetPolicy$ constructs the policy $\policy_i$, in the following sense.
    Each recursive call corresponds to a given observation sequence $\obsSeq$: consider that the action guessed is then $\POMDPAction = \policy(\obsSeq)$.
    Then, in the main procedure, the tuples $\MEPPoint_1, \dots, \MEPPoint_n$ that are obtained are such that $\MEPPoint_i = \MEP(\policy_i)$ for each $i$.
    Since we have $\Conv\{\MEP(\policy_1), \dots, \MEP(\policy_n)\} \cap \left[\lambda, +\infty\right)^n \neq \emptyset$, that run of the algorithm returns $\True$.

    Conversely, if there is a run of \Cref{algo:npspace} that returns $\True$, then the actions guessed along that run define $n$ deterministic policies $\policy_1, \dots, \policy_n$ such that we have $\Conv\{\MEP(\policy_1), \dots, \MEP(\policy_n)\} \cap \left[\lambda, +\infty\right)^n \neq \emptyset$.
    As a consequence, we have a positive instance of \Cref{pb:threshold}.

    Finally, this algorithm uses polynomial space.
    Indeed, the size of the recursion stack is $\horizon$ (remember that $\horizon$ is given in unary), and each recursive call uses a polynomial amount of information: the POMDP $\POMDP$ and the integer $n$, an integer $\l \leq \horizon$, and a multi-belief $\multibelief$ which grows linearly with $\l$.

    We therefore have a non-deterministic polynomial-space algorithm that decides \Cref{pb:threshold}, and consequently also its approximate version.
    Since $\NPSPACE = \PSPACE$, the conclusion follows.
\end{proof}

\section{Proof of \Cref{thm:space-pspace}} \label{app:space-pspace}

\thmSpacePspace*

\begin{proof}
    The correctness of \Cref{algo:pspace} follows from the correctness of \Cref{algo:npspace}, as it uses the same approach, enumerating multi-expected payoffs instead of guessing them.
    
    The $\main$ procedure enumerates all possible multi-expected payoff values $\bar X = (\bar x^1,\dots, \bar x^n)$ achievable by $n$ deterministic policies, and returns true whenever one is found that is achievable by the POMDP (i.e., it corresponds with a list of $n$ deterministic strategies), and it satisfies the condition of Lemma~\ref{lemma:mixed_policy_existence}.
    The $\CheckAchievableValue$ procedure checks whether the value $\bar x$ is achievable as a 
    multi-expected payoff
    in the POMDP $\POMDP$, from a multi-belief $\multibelief$ and a horizon of $\l$ steps. Here, achievable means that there exists a deterministic strategy $\sigma$ such that $\expected^\horizon_{\POMDPState_i}(\policy) = \bar x$. 
    A key difference between with respect to \Cref{algo:npspace} is that for every action, \Cref{algo:npspace} can build $\bar x^o$ by guessing a policy in the next recursive level, while \Cref{algo:pspace} cannot use non-determinism, so it has to enumerate all possible values of $\bar x^o$ obtained by following the recursion. 
    Consider the update in line 14 of \Cref{algo:npspace}, 
    where we enumerate $\POMDPObs =\{\obs_1,\dots,\obs_{|\POMDPObs|}\}$:
    \begin{equation}
    \label{eq:bellmanupdate}
    \MEPPoint \:\:\gets\:\:
    \sum_{\POMDPState \in \POMDPStates}\bigg(
                            r(s)\belief_i(s) \:\: + \sum_{j=1}^{|\POMDPObs|}\:\:\sum_{\POMDPStatebis \in \POMDPStates,\, \POMDPMapObs(\POMDPStatebis) = \obs_j} \belief_i(\POMDPState) \cdot \POMDPTransF(\POMDPState, \POMDPAction)(\POMDPStatebis) \cdot x_i^{\obs_j}\bigg)_{i\in[n]}.    
    \end{equation}
    The procedure $\CheckAchievableValue$, in line with $\ConstructDetPolicy$, checks whether there is a strategy that achieves a multi-expected payoff $\bar x$ from the multi-belief $\bar \beta$ in $\l$ steps by checking, for every pair $(\POMDPAction, \obs)\in \POMDPActions\times\POMDPObs$
    the multi-expected payoffs that can be achieved with a horizon $\l-1$ and a multi-belief $\bar\beta^\obs=\updateOp(\bar\beta, \POMDPAction, \obs)$, and aggregating them according to the Bellman update (Eq.~\eqref{eq:bellmanupdate}).
    To do so, it constructs partial sums of Eq.~\ref{eq:bellmanupdate} recursively.
    The $\auxFunc(j,\l,\multibelief,a, \bar x_{\mathsf{rem}})$ procedure checks whether there is a strategy that starts with action $a$ and such that
    \[
    \bar x_{\mathsf{rem}} = \bigg(\sum_{j'=j}^{|\POMDPObs|}
    \sum_{\POMDPState \in \POMDPStates}
                            r(s)\belief_i(s) \:\: + \sum_{\POMDPStatebis \in \POMDPStates,\, \POMDPMapObs(\POMDPStatebis) = \obs_{j'}} \belief_i(\POMDPState) \cdot \POMDPTransF(\POMDPState, \POMDPAction)(\POMDPStatebis) \cdot x_i^{\obs_{j'}}\bigg)_{i\in[n]}.
    \]
    To do so, it checks, for every possible value in $N$, whether it can be achieved as a multi-expected payoff in a horizon of $\l-1$ steps from the updated multi-belief $\updateOp(\multibelief,a,o_j)$, and for those values where it is possible, whether there is a way to fill up the remaining $|\POMDPObs|-j+1$ terms in the partial sum to make it up to $\bar x_{\mathsf{rem}}$.

    To check the space requirements, we need to check the space used by each procedure, and the size of the recursion stack.
    Let $D=\pi(\InputSize)$ (recall \Cref{lem:exponentialpoints}).
    
    First, observe that the main procedure just requires storing $n$ multi-expected payoff points, plus a high-level call to the $\CheckAchievableValue$ procedure. 

    Second, consider the space requirement for a single frame of both recursive procedures. In a single frame of the $\CheckAchievableValue$ procedure, the space requirement is that required to store the current multi-belief $\multibelief$ ($|\POMDPStates|\cdot n D$) 
    and the current multi-expected payoff value $\bar x$ ($n^2 D \leq |\POMDPStates|\cdot n D$), 
    plus the space required for a high-level call to $\auxFunc$.
    In a single frame of $\auxFunc$, the space requirement is the same as for $\CheckAchievableValue$: a multibelief and a multi-expected payoff.
    
    Third, consider the size of the recursion stack.
    If we consider the index $(\l, j)$ shared between the procedures $\CheckAchievableValue$ and $\auxFunc$ (considering that $\CheckAchievableValue$ has always $j=0$), at every recursive call, either $j$ increases with a call to $\auxFunc$ (lines 15 and 29), keeping $\l$ fixed, or $\l$ decreases with a call to $\CheckAchievableValue$ (line 27), resetting $j$ to $0$. Therefore, the recursion follows all the values of $(\l,j)$ in lexicographical order (descending for $\l$ and ascending for $j$), so the size of the recursion stack is always at most $|\POMDPObs|\cdot \horizon$.

    Finally, the space requirement is the space requirement of each frame, times the size of the recursion stack, plus the space required for the main procedure: $|\POMDPObs|\cdot \horizon |\POMDPStates|\cdot n\cdot D + n^2\cdot D$. 
\end{proof}

\section{Experimental Details}
\label{sec:experiments-more}

In this appendix, we provide further details on our experimental evaluation.

\subsection{Algorithm choice}
As noted in Sec.~\ref{sec:algorithms}, Algorithm~\ref{algo:pspace} is very inefficient in terms of computation time, so we only consider Algorithm~\ref{algo:efficient} for our empirical evaluation.

As for the comparison with~\cite{bovymulti}, a feature in their algorithm is that the user can specify a \emph{gap} or \emph{tolerance}, explicitly indicating how precise the approximation of the actual result is allowed to be.
In our experience with~\cite{bovymulti}'s tool, making this gap more or less restrictive had almost no effect in practice, both in the obtained value of the MEPOMDP, as well as in the computation time required.
We explain this phenomenon as a consequence of~\cite{bovymulti} being optimized for the infinite horizon discounted problem, where, according to the results presented in~\cite{bovymulti}, the gap does play a significant role.

\subsection{Benchmark description}

As benchmarks, we used modified versions of standard problems in the POMDP literature: RockSample, introduced in~\cite{smith2004heuristic}, with the modifications in~\cite{bovymulti}, and Robot Navigation and Aircraft Identification, both introduced in~\cite{cassandra98}.
For consistency with the sources of our models, we consider POMDPs with stochasticity in the observation function, that is, $\POMDPMapObs\colon \POMDPStates\times\POMDPActions\to\mathcal D(\POMDPObs)$. 
As mentioned in Remark~\ref{rmk:equivalentpomdps}, these models are equivalent, in the sense that one can systematically build a POMDP with deterministic observation from one with stochastic observation, that has the same value for any horizon.
\begin{figure}[t]
    \centering
    \includegraphics[width=0.8\linewidth]{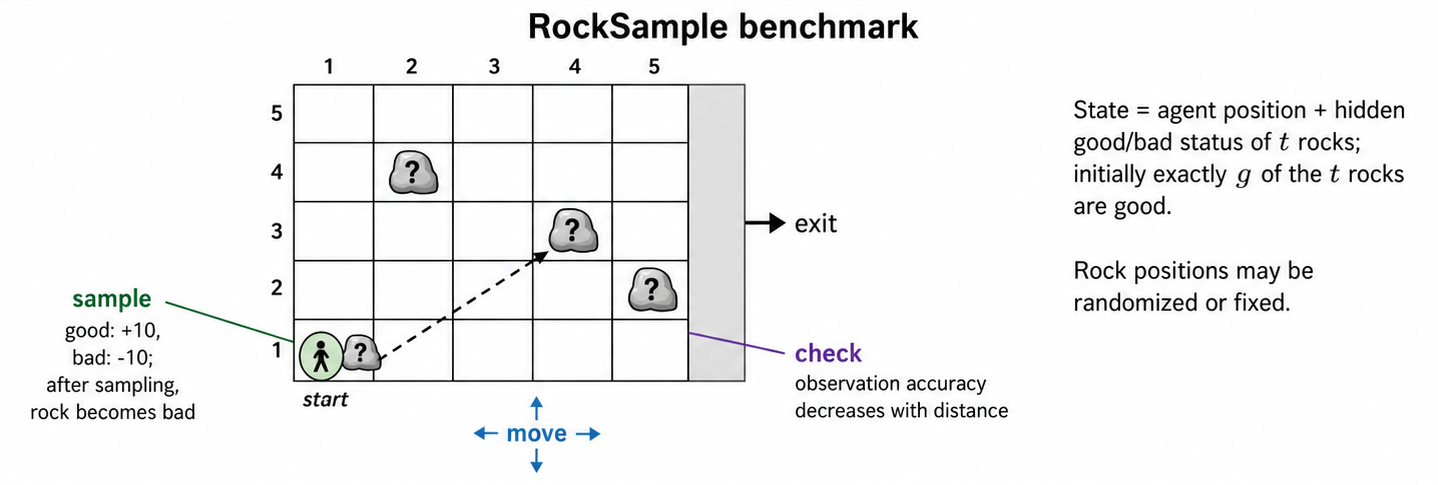}
    \caption{RockSample problem.}
    \label{fig:rocksample-illustration}
\end{figure}

\subsubsection{Rock Sample}
The RockSample problem was introduced in~\cite{smith2004heuristic}, and extended to the ME-POMDP setting in~\cite{bovymulti}. An instance consists of an $m\times m$ grid with $t$ rocks at known positions. The agent starts in the bottom-left cell, moves in the four cardinal directions, can exit through the right boundary, and may sample a rock when located on its cell. Rock qualities are hidden: initially exactly $g$ of the $t$ rocks are good and the remaining $t-g$ rocks are bad. Sampling a good rock gives reward $+10$ and changes that rock to bad; sampling a bad rock gives reward $-10$. For each rock, a check action returns a noisy good/bad observation whose correctness decreases with the distance between the agent and the queried rock. 
We denote instances of the problem as $\RS_{m,g,t}$. 
Note that, in our description, we use the instances corresponding to $\RS^{c}_{m,g,t}$ in~\cite{bovymulti}, that is, with fixed rock positions.
We illustrate the problem in Figure~\ref{fig:rocksample-illustration}.
The initial states are those states in which the agent is located in the bottom-left corner and g out of t rocks are good.
The multiple initial states, i.e., the position of the rocks, are sampled in the same way as~\cite{bovymulti}.

\subsubsection{Robot navigation}
\label{app:robot-navigation-benchmarks}

The Robot Navigation problem~\cite[Appendix H.5]{cassandra98} describes a robot moving through a room, that has to reach different goals.
We use the \textsc{CIT}, \textsc{MIT}, \textsc{Pentagon}, and \textsc{SUNYSB} robot-navigation maps, illustrated in Figure~\ref{fig:RobotNavigationMaps}.  Each map is a discretized indoor environment.  A nonterminal state is a pair
\(
    s=(x,\theta),
\)
where \(x\) is a traversable map location and \(\theta\) is one of the four cardinal orientations.  The source construction also includes an absorbing terminal state, entered after the robot declares that it is at the goal.  The concrete four-action instances use
$\POMDPActions_{\mathrm{nav}}
    = \{\mathsf{forward},\mathsf{left},\mathsf{right},\mathsf{declare}\}$.
The first three actions are noisy motion actions, and \(\mathsf{declare}\) moves the process to the absorbing terminal state.  The reward is zero except for \(\mathsf{declare}\): declaring at the goal gives reward \(+1\), while declaring away from the goal gives reward \(-1\).

The motion model is specified by primitive outcomes.  Let \(N\) denote no movement, \(F\) a one-cell forward move, \(L\) a \(90^\circ\) left turn, \(R\) a \(90^\circ\) right turn, and \(A\) absorption.  The noisy outcomes are
\[
\begin{array}{c|l}
\text{Action} & \text{Primitive outcome distribution} \\
\hline
\mathsf{forward} & N:0.11,\quad F:0.88,\quad FF:0.01 \\
\mathsf{left}    & N:0.05,\quad L:0.90,\quad LL:0.05 \\
\mathsf{right}   & N:0.05,\quad R:0.90,\quad RR:0.05 \\
\mathsf{declare} & A:1.00 .
\end{array}
\]

\begin{figure}
     \centering
     \begin{subfigure}[b]{0.24\textwidth}
         \centering
         \includegraphics[width=\textwidth]{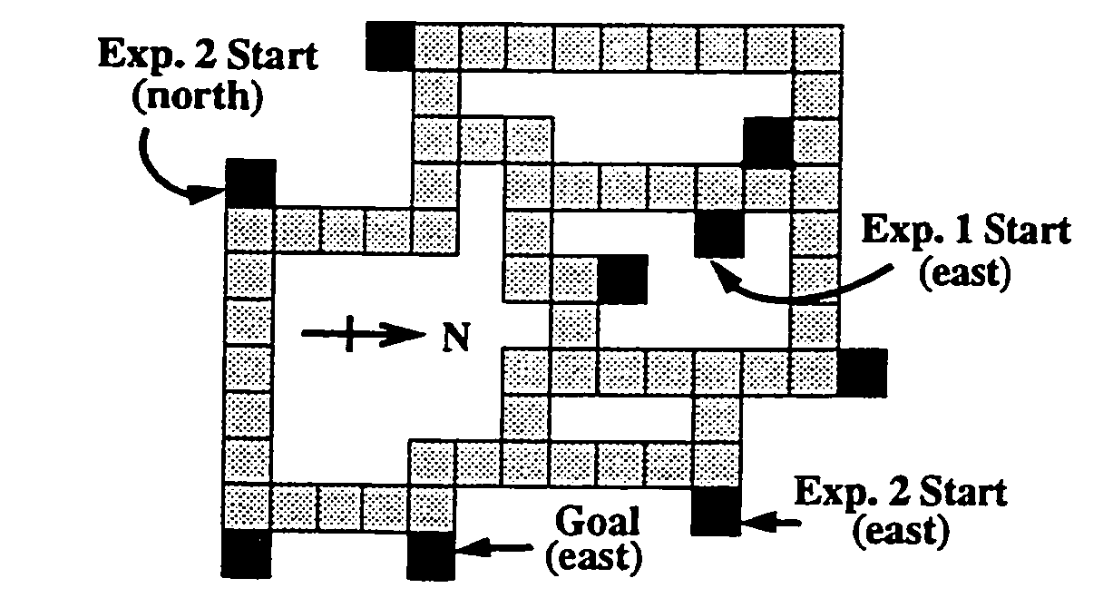}
         \caption{\textsc{CIT} (281)}
     \end{subfigure}
     \begin{subfigure}[b]{0.24\textwidth}
         \centering
         \includegraphics[width=\textwidth]{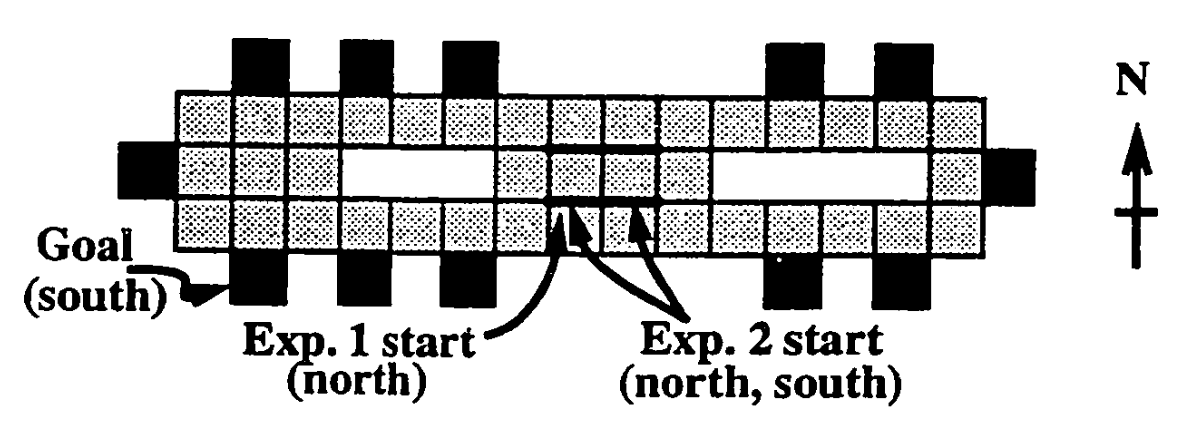}
         \caption{\textsc{MIT} (201)}
     \end{subfigure}
     \begin{subfigure}[b]{0.24\textwidth}
         \centering
         \includegraphics[width=\textwidth]{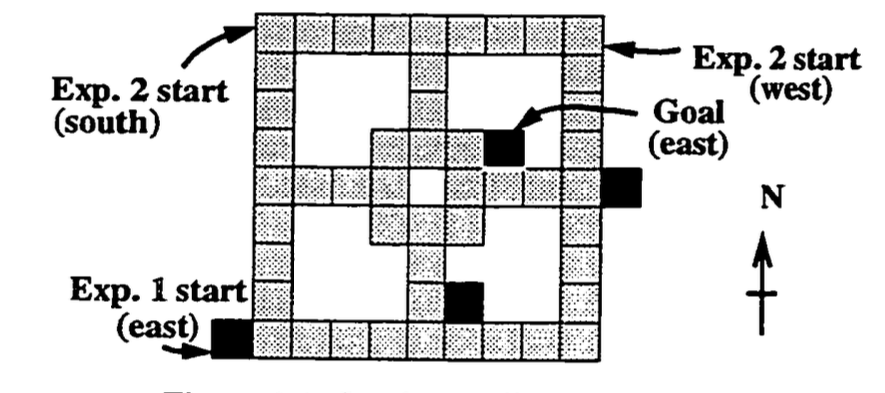}
         \caption{\textsc{Pentagon} (209)} 
     \end{subfigure}
     \begin{subfigure}[b]{0.24\textwidth}
         \centering
         \includegraphics[width=\textwidth]{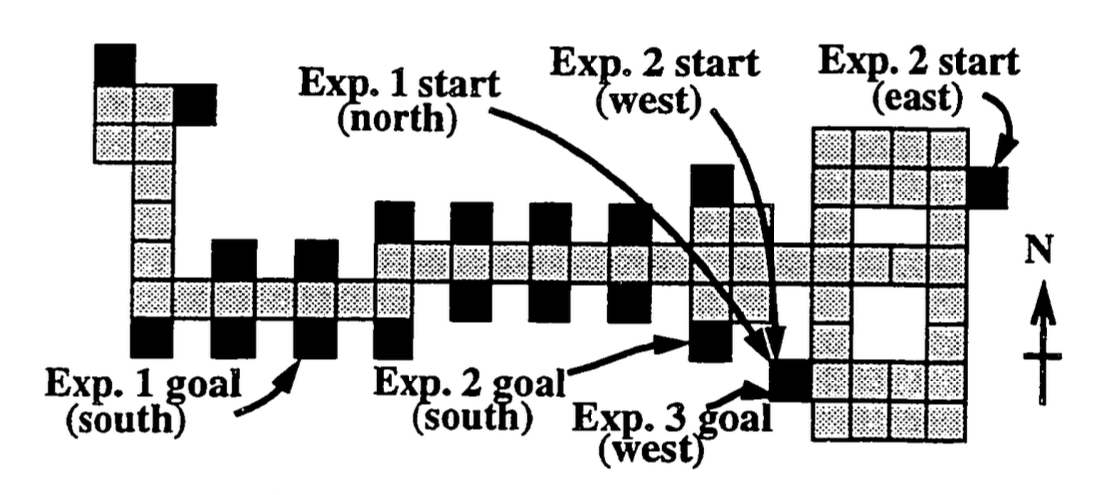}
         \caption{\textsc{SUNYSB} (297)}
     \end{subfigure}
        \caption{Robot Navigation Maps. \textsc{Name} (No. states). Figures from~\cite[Figures 6.1--6.4]{cassandra98}.}
        \label{fig:RobotNavigationMaps}
\end{figure}

For each nonterminal state \(s\), let
\[
    \sigma(s)=(\sigma_{\mathrm{front}}(s),\sigma_{\mathrm{left}}(s),\sigma_{\mathrm{right}}(s))
        \in \{\mathsf{open},\mathsf{wall},\mathsf{doorway}\}^3
\]
be the true local observation in the directions in front of the robot, to its left, and to its right.  We replace the source observation table by the alphabet
\[
    \POMDPObs_{\mathrm{nav}}
      = \{\mathsf{open},\mathsf{wall},\mathsf{doorway}\}^3
        \cup \{(\mathsf{undet},\mathsf{undet},\mathsf{undet})\},
\]
so \(|\POMDPObs_{\mathrm{nav}}|=28\).  Observations are independent of the action, and depend only on how many sides of the robot have open space, wall, or doorway.  Let
\[
    \alpha_{\mathrm{ow}}(s)
      = \#\{i \in \{\mathrm{front},\mathrm{left},\mathrm{right}\}: \sigma_i(s)\in\{\mathsf{open},\mathsf{wall}\}\},
\]
and let \(\alpha_{\mathrm{door}}(s)=3-\alpha_{\mathrm{ow}}(s)\).  The probability of observing the correct local signature is
\[
    p_{\mathrm{corr}}(s)
      = 0.9^{\alpha_{\mathrm{ow}}(s)} 0.7^{\alpha_{\mathrm{door}}(s)} .
\]
Thus, for every action \(a\), the probability to observe $o$ is $p_{\mathrm{corr}}(s)$ if $o=\sigma(s)$, $1- p_{\mathrm{corr}}(s)$ if $o=(\mathsf{undet},\mathsf{undet},\mathsf{undet})$, and $0$ otherwise.

For each map we choose a pair of initial states that are at a distance at most $d$ to the goal and maximize the observational diversity. 
Let \(G_E\) be the set of goal states for map \(E\), and let \(d_E\) be shortest-path distance in the nonterminal state graph induced by positive-probability motion transitions, with the \(\mathsf{declare}\) transition excluded.  The candidate set is
\[
    C_E = \{s\in S_E\setminus\{\bot\}: d_E(s,G_E)\le d\} . 
\]
In our test cases, the parameter $d$ is either $1$, $2$, or $3$.
As initial states, we select the set  $I_E=\{s_E^{(1)},s_E^{(2)}\}$ defined as
$
    \{s_E^{(1)},s_E^{(2)}\}
      \in \arg\max_{\{s,t\}\subseteq C_E,\,s\ne t}
      D_{\mathrm{JS}}(\POMDPMapObs(s),\POMDPMapObs(t)),
$
where
\[
    D_{\mathrm{JS}}(P,Q)
      = \frac{1}{2}D_{\mathrm{KL}}\!\left(P\middle\|\frac{P+Q}{2}\right)
        +\frac{1}{2}D_{\mathrm{KL}}\!\left(Q\middle\|\frac{P+Q}{2}\right)
\]
is the Jensen-Shannon divergence over \(\POMDPObs_{\mathrm{nav}}\). 
In our test cases, the largest Jensen-Shannon divergence we observed is $\approx 0.204$.

\subsubsection{Identification, Friend or Foe}
\label{app:aircraft-identification-benchmark}

The IFF problem (also referred to as \emph{aircraft-identification} problem) models an incoming aircraft that may be a friend or a foe.  A nonterminal state is $s=(\tau,d,v)$,
where \(\tau\in\{\mathsf{friend},\mathsf{foe}\}\) is the aircraft type, \(d\in\{0,\ldots,D-1\}\) is the discretized distance to the base, and \(v\in\{0,\ldots,4\}\) is the visibility of the base to the aircraft, i.e., how well does the aircraft see the base.  We use \(D=10\), so there are \(2D\cdot 5=100\) nonterminal states.  
As in the source benchmark, there are four absorbing states:
\[
    \mathsf{base\mbox{-}safe},\quad
    \mathsf{base\mbox{-}destroyed},\quad
    \mathsf{foe\mbox{-}destroyed},\quad
    \mathsf{friend\mbox{-}destroyed},
\]
for a total of \(44\) states.  The action set is
$
    \POMDPActions_{\mathrm{air}}
      = \{\mathsf{active},\mathsf{passive},\mathsf{noop},\mathsf{attack}\} $.
The actions \(\mathsf{active}\) and \(\mathsf{passive}\) are sensing actions; \(\mathsf{active}\) is more accurate but tends to increase visibility more than \(\mathsf{passive}\).  The action \(\mathsf{noop}\) uses no sensor, and \(\mathsf{attack}\) attempts to destroy the aircraft.

On nonterminal transitions that do not enter an absorbing state, the aircraft type \(\tau\) is unchanged.  For \(d>0\), the aircraft advances one distance bin toward the base with probability \(0.8\), and remains at the same distance with probability \(0.2\).
The distance transition is independent of \(\tau\), \(v\), and the action, conditional on not entering an absorbing state.
We modify the visibility transition by forbidding visibility decreases.  Let the source visibility-increment probabilities over \(\Delta\in\{-1,0,+1\}\) be
\[
\begin{array}{c|ccc}
\text{Action} & q_a(-1) & q_a(0) & q_a(+1) \\
\hline
\mathsf{noop}    & 0.25 & 0.75 & 0.00 \\
\mathsf{passive} & 0.00 & 0.90 & 0.10 \\
\mathsf{active}  & 0.00 & 0.05 & 0.95 \\
\mathsf{attack}  & 0.00 & 0.20 & 0.80 .
\end{array}
\]
For our benchmark, transitions to lower visibility are removed and the remaining probabilities are renormalized:
\[
    P(v'=v+\Delta\mid v,a)
      =
      \frac{
        q_a(\Delta)\,\mathbf 1\{\Delta\ge 0\}\,\mathbf 1\{0\le v+\Delta\le 4\}
      }{
        \sum_{\delta\in\{-1,0,+1\}}
        q_a(\delta)\,\mathbf 1\{\delta\ge 0\}\,\mathbf 1\{0\le v+\delta\le 4\}
      } .
\]
Consequently, for \(v<4\), \(\mathsf{noop}\) leaves visibility unchanged with probability \(1\), while \(\mathsf{passive}\), \(\mathsf{active}\), and \(\mathsf{attack}\) use the probabilities shown in the \(0\) and \(+1\) columns above.  At \(v=4\), all actions leave visibility at \(4\) with probability \(1\).

The attack action may enter an absorbing state before the normal distance and visibility update is applied.  At distance \(d\), the probability of destroying the aircraft is
\[
    p_{\mathrm{hit}}(d)=\frac{(D-d)^2}{D^2} .
\]
If the attack succeeds, the successor is \(\mathsf{foe\mbox{-}destroyed}\) when \(\tau=\mathsf{foe}\), and \(\mathsf{friend\mbox{-}destroyed}\) when \(\tau=\mathsf{friend}\).  If the attack fails, the usual distance and modified visibility transitions are applied, conditional on failure.  When a non-attacked friend is already at \(d=0\), the next state is \(\mathsf{base\mbox{-}safe}\) with probability \(1\).  When a non-attacked foe is already at \(d=0\), the next state is \(\mathsf{base\mbox{-}destroyed}\) with probability \(0.25+0.1v\), and \(\mathsf{base\mbox{-}safe}\) otherwise.
The only nonzero rewards are the immediate rewards for entering absorbing states:
\[
\begin{array}{c|r}
\text{Absorbing state} & \text{Reward} \\
\hline
\mathsf{base\mbox{-}safe}       & 0 \\
\mathsf{base\mbox{-}destroyed}  & -100 \\
\mathsf{foe\mbox{-}destroyed}   & 20 \\
\mathsf{friend\mbox{-}destroyed}& -30 .
\end{array}
\]

The observation alphabet is
\[
    \POMDPObs_{\mathrm{air}}
      = \bigl(\{\mathsf{friend},\mathsf{foe}\}\times\{0,\ldots,D-1\}\bigr)
        \cup \{\mathsf{nothing},\mathsf{absorb}\},
\]
so \(|\POMDPObs_{\mathrm{air}}|=2D+2=22\).  The action \(\mathsf{noop}\) emits \(\mathsf{nothing}\) with probability \(1\), and absorbing states emit \(\mathsf{absorb}\) with probability \(1\).  For sensing, let \(\bar\tau\) denote the opposite type and let \(d^+=\min\{d+1,D-1\}\).  We use the following modified observation probabilities:
\[
\begin{aligned}
    \POMDPMapObs((\tau,d)\mid (\tau,d,v),\mathsf{active}) &= 0.9, &
    \POMDPMapObs((\bar\tau,d^+)\mid (\tau,d,v),\mathsf{active}) &= 0.1,\\
    \POMDPMapObs((\tau,d)\mid (\tau,d,v),\mathsf{passive}) &= 0.8, &
    \POMDPMapObs((\bar\tau,d^+)\mid (\tau,d,v),\mathsf{passive}) &= 0.2 .
\end{aligned}
\]
All other type-distance observations have probability zero.  For \(d=D-1\), the reported distance \(d^+\) is clipped to the largest distance bin, matching the boundary convention of the discretized observation space.  The attack action follows the source-domain convention: successful terminal outcomes emit \(\mathsf{absorb}\), while unsuccessful nonterminal outcomes use the active-sensing observation model above.

Initial states in the aircraft-identification experiments are chosen from two foe states and one friend state.  For parameters
\[
    0\le d_1 < d_2 \le 4,
    \qquad
    v_1,v_2\in\{0,\ldots,4\},
\]
we use the initial-state set
\[
    I_{d_1,d_2,v_1,v_2}
      = \{(\mathsf{foe},d_1,v_1),\ (\mathsf{foe},d_2,v_2),\ (\mathsf{friend},d_2,\bar v)\},
\]
where \(\bar v\) is any fixed visibility level for the friend state.  The value of \(\bar v\) is immaterial for our initialization because the friend's visibility does not affect the modified observation probabilities or the terminal reward associated with the friend/foe distinction.

Problem instances are labelled according to their initial state parameters, i.e.,  $\IFF_{d_1,d_2,v_1,v_2}$.

\subsection{Scalability: Further results}

In Figure~\ref{fig:RockSample_scalability2}, we extend our study of the scalability of our approach in the RockSample problem by plotting computation time in terms of total number of states and number of initial states. In both cases, we observe that computation time increases slightly with the size of the POMDP, but the major scalability issue is the horizon.

In Figure~\ref{fig:Cassandra_scalability}, we show analogous scatter plots for the robot navigation and identification benchmarks. As before, we observe that the dominant factor for scalability is horizon.

\begin{figure}[t]
     \centering
     \begin{subfigure}[b]{0.48\textwidth}
         \centering
         \includegraphics[width=\textwidth]{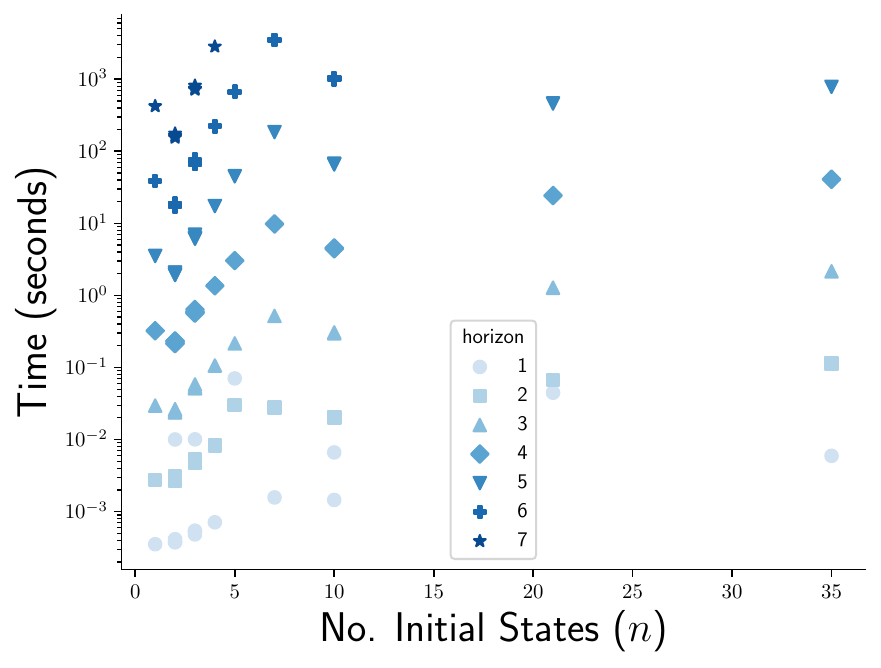}
     \end{subfigure}
     \begin{subfigure}[b]{0.48\textwidth}
         \centering
         \includegraphics[width=\textwidth]{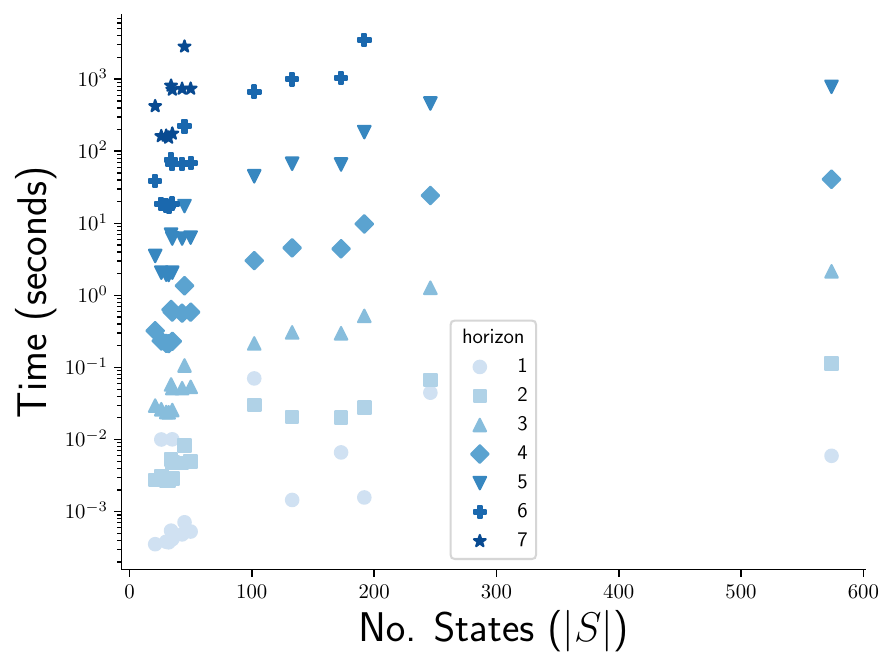}
     \end{subfigure}
        \caption{Scalability of Rock Sampling benchmark}
        \label{fig:RockSample_scalability2}
\end{figure}

\begin{figure}[t]
     \centering
     \begin{subfigure}[b]{0.48\textwidth}
         \centering
         \includegraphics[width=\textwidth]{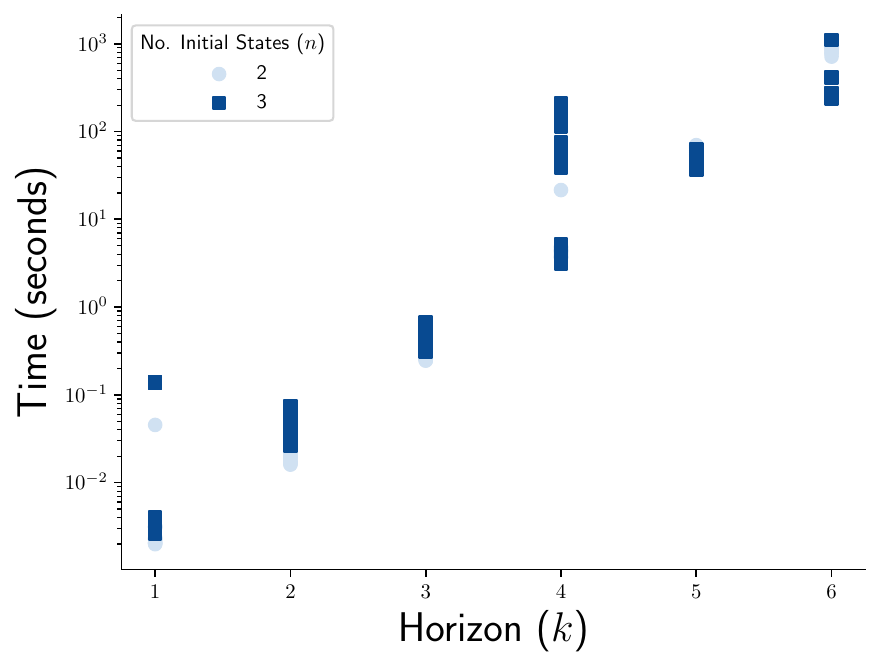}
     \end{subfigure}
     \begin{subfigure}[b]{0.48\textwidth}
         \centering
         \includegraphics[width=\textwidth]{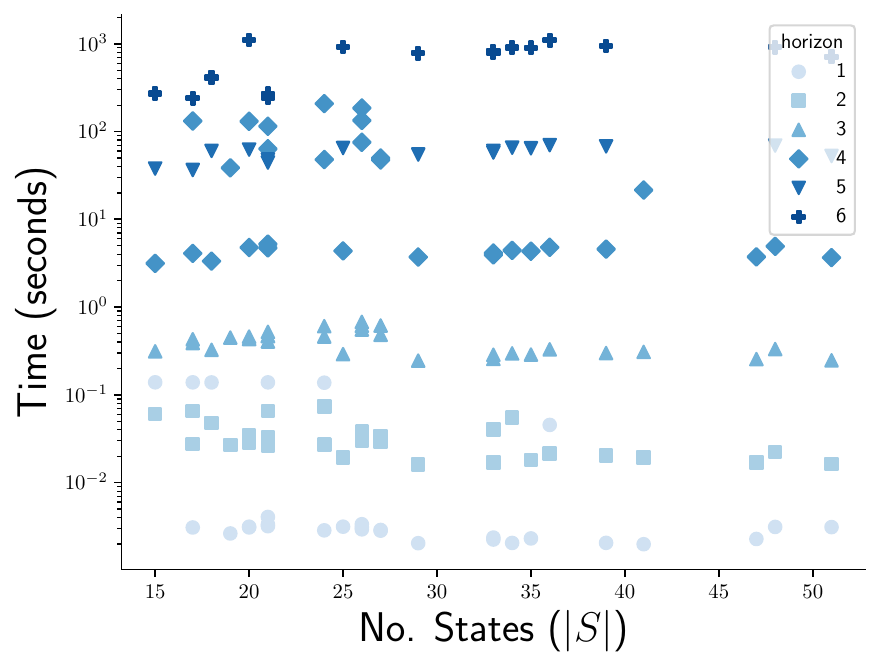}
     \end{subfigure}
        \caption{Scalability of the Robot Navigation and Identification Benchmarks}
        \label{fig:Cassandra_scalability}
\end{figure}

\FloatBarrier 

\subsection{Comparison to baseline: Further results}

In Tables~\ref{tab:rock-sample-full-0}~to~\ref{tab:rock-sample-full-3}, we extend the results in Table~\ref{tab:rock-sample-main} by reporting the whole suite of experiments on the RockSample problem, comparing our runtime with the runtime of~\cite{bovymulti}.
In Tables~\ref{tab:robot-navigation-full-0} to~\ref{tab:robot-navigation-full-3} we extend the results for the robot navigation problem, and in Tables~\ref{tab:iff-full-0} to~\ref{tab:iff-full-2}, for the identification problem.
The results show that our method is faster almost all the time, with an average ratio for the non-timeout instances of $1s$ of computation (ours) for $14537s$ of computation (\cite{bovymulti}).

\begin{table}[t]
\centering
\small
\caption{Rock sampling full benchmark results (Part 1 of 5).}
\label{tab:rock-sample-full-0}
\begin{tabular}{lrrrrrrr}
\toprule
Model & $|\POMDPStates|$ & $|\POMDPActions|$ & $|\POMDPObs|$ & $n$ & $\horizon$ & Time \cite{bovymulti} (s) & Time ours (s) \\
\midrule
$\RS_{3, 1, 2}$ & 26 & 7 & 3 & 2 & 1 & 0.3738 & 0.01002 \\
$\RS_{3, 1, 2}$ & 26 & 7 & 3 & 2 & 2 & 0.7799 & 0.003126 \\
$\RS_{3, 1, 2}$ & 26 & 7 & 3 & 2 & 3 & 1.229 & 0.02633 \\
$\RS_{3, 1, 2}$ & 26 & 7 & 3 & 2 & 4 & 1.736 & 0.2335 \\
$\RS_{3, 1, 2}$ & 26 & 7 & 3 & 2 & 5 & 2.281 & 2.068 \\
$\RS_{3, 1, 2}$ & 26 & 7 & 3 & 2 & 6 & \textsc{to} & 18.45 \\
$\RS_{3, 1, 2}$ & 26 & 7 & 3 & 2 & 7 & \textsc{to} & 162 \\
$\RS_{3, 1, 3}$ & 34 & 8 & 3 & 3 & 1 & 0.7684 & 0.000544 \\
$\RS_{3, 1, 3}$ & 34 & 8 & 3 & 3 & 2 & 1.532 & 0.005291 \\
$\RS_{3, 1, 3}$ & 34 & 8 & 3 & 3 & 3 & 2.434 & 0.05812 \\
$\RS_{3, 1, 3}$ & 34 & 8 & 3 & 3 & 4 & \textsc{to} & 0.636 \\
$\RS_{3, 1, 3}$ & 34 & 8 & 3 & 3 & 5 & \textsc{to} & 6.96 \\
$\RS_{3, 1, 3}$ & 34 & 8 & 3 & 3 & 6 & \textsc{to} & 76.41 \\
$\RS_{3, 1, 3}$ & 34 & 8 & 3 & 3 & 7 & \textsc{to} & 812 \\
$\RS_{3, 1, 4}$ & 45 & 9 & 3 & 4 & 1 & 1.267 & 0.000711 \\
$\RS_{3, 1, 4}$ & 45 & 9 & 3 & 4 & 2 & 2.694 & 0.008213 \\
$\RS_{3, 1, 4}$ & 45 & 9 & 3 & 4 & 3 & 4.234 & 0.1065 \\
$\RS_{3, 1, 4}$ & 45 & 9 & 3 & 4 & 4 & \textsc{to} & 1.362 \\
$\RS_{3, 1, 4}$ & 45 & 9 & 3 & 4 & 5 & \textsc{to} & 17.37 \\
$\RS_{3, 1, 4}$ & 45 & 9 & 3 & 4 & 6 & \textsc{to} & 222.1 \\
$\RS_{3, 1, 4}$ & 45 & 9 & 3 & 4 & 7 & \textsc{to} & 2834 \\
$\RS_{3, 2, 3}$ & 50 & 8 & 3 & 3 & 1 & 2.158 & 0.000529 \\
$\RS_{3, 2, 3}$ & 50 & 8 & 3 & 3 & 2 & 14.93 & 0.004932 \\
$\RS_{3, 2, 3}$ & 50 & 8 & 3 & 3 & 3 & 7.246 & 0.05386 \\
$\RS_{3, 2, 3}$ & 50 & 8 & 3 & 3 & 4 & \textsc{to} & 0.5853 \\
$\RS_{3, 2, 3}$ & 50 & 8 & 3 & 3 & 5 & \textsc{to} & 6.34 \\
$\RS_{3, 2, 3}$ & 50 & 8 & 3 & 3 & 6 & \textsc{to} & 68.53 \\
$\RS_{3, 2, 3}$ & 50 & 8 & 3 & 3 & 7 & \textsc{to} & 738.4 \\
\bottomrule
\end{tabular}
\end{table}

\begin{table}[t]
\centering
\small
\caption{Rock sampling full benchmark results (Part 2 of 5).}
\label{tab:rock-sample-full-1}
\begin{tabular}{lrrrrrrr}
\toprule
Model & $|\POMDPStates|$ & $|\POMDPActions|$ & $|\POMDPObs|$ & $n$ & $\horizon$ & Time \cite{bovymulti} (s) & Time ours (s) \\
\midrule
$\RS_{4, 1, 2}$ & 32 & 7 & 3 & 2 & 1 & 1.111 & 0.000375 \\
$\RS_{4, 1, 2}$ & 32 & 7 & 3 & 2 & 2 & 2.344 & 0.002675 \\
$\RS_{4, 1, 2}$ & 32 & 7 & 3 & 2 & 3 & 14.24 & 0.02388 \\
$\RS_{4, 1, 2}$ & 32 & 7 & 3 & 2 & 4 & 5.263 & 0.2143 \\
$\RS_{4, 1, 2}$ & 32 & 7 & 3 & 2 & 5 & 6.878 & 1.917 \\
$\RS_{4, 1, 2}$ & 32 & 7 & 3 & 2 & 6 & 24.33 & 17.2 \\
$\RS_{4, 1, 2}$ & 32 & 7 & 3 & 2 & 7 & \textsc{to} & 154.2 \\
\bottomrule
\end{tabular}
\end{table}

\begin{table}[t]
\centering
\small
\caption{Rock sampling full benchmark results (Part 3 of 5).}
\label{tab:rock-sample-full-2}
\begin{tabular}{lrrrrrrr}
\toprule
Model & $|\POMDPStates|$ & $|\POMDPActions|$ & $|\POMDPObs|$ & $n$ & $\horizon$ & Time \cite{bovymulti} (s) & Time ours (s) \\
\midrule
$\RS_{5, 1, 2}$ & 30 & 7 & 3 & 2 & 1 & 2.659 & 0.000382 \\
$\RS_{5, 1, 2}$ & 30 & 7 & 3 & 2 & 2 & 5.628 & 0.002711 \\
$\RS_{5, 1, 2}$ & 30 & 7 & 3 & 2 & 3 & 8.921 & 0.0241 \\
$\RS_{5, 1, 2}$ & 30 & 7 & 3 & 2 & 4 & 12.58 & 0.2163 \\
$\RS_{5, 1, 2}$ & 30 & 7 & 3 & 2 & 5 & 16.55 & 1.938 \\
$\RS_{5, 1, 2}$ & 30 & 7 & 3 & 2 & 6 & 20.9 & 17.61 \\
$\RS_{5, 1, 2}$ & 30 & 7 & 3 & 2 & 7 & 25.42 & 164.3 \\
$\RS_{6, 1, 2}$ & 35 & 7 & 3 & 2 & 1 & 5.432 & 0.000414 \\
$\RS_{6, 1, 2}$ & 35 & 7 & 3 & 2 & 2 & 11.58 & 0.002896 \\
$\RS_{6, 1, 2}$ & 35 & 7 & 3 & 2 & 3 & 18.32 & 0.02568 \\
$\RS_{6, 1, 2}$ & 35 & 7 & 3 & 2 & 4 & 26.05 & 0.2304 \\
$\RS_{6, 1, 2}$ & 35 & 7 & 3 & 2 & 5 & \textsc{to} & 2.066 \\
$\RS_{6, 1, 2}$ & 35 & 7 & 3 & 2 & 6 & 43.17 & 18.68 \\
$\RS_{6, 1, 2}$ & 35 & 7 & 3 & 2 & 7 & 52.68 & 175.8 \\
\bottomrule
\end{tabular}
\end{table}

\begin{table}[t]
\centering
\small
\caption{Rock sampling full benchmark results (Part 4 of 5).}
\label{tab:rock-sample-full-3}
\begin{tabular}{lrrrrrrr}
\toprule
Model & $|\POMDPStates|$ & $|\POMDPActions|$ & $|\POMDPObs|$ & $n$ & $\horizon$ & Time \cite{bovymulti} (s) & Time ours (s) \\
\midrule
$\RS_{3, 1, 3}$ & 35 & 8 & 3 & 3 & 1 & 0.7359 & 0.01007 \\
$\RS_{3, 1, 3}$ & 35 & 8 & 3 & 3 & 2 & 5.004 & 0.004771 \\
$\RS_{3, 1, 3}$ & 35 & 8 & 3 & 3 & 3 & 2.441 & 0.05178 \\
$\RS_{3, 1, 3}$ & 35 & 8 & 3 & 3 & 4 & 3.436 & 0.5878 \\
$\RS_{3, 1, 3}$ & 35 & 8 & 3 & 3 & 5 & \textsc{to} & 6.156 \\
$\RS_{3, 1, 3}$ & 35 & 8 & 3 & 3 & 6 & \textsc{to} & 66.51 \\
$\RS_{3, 1, 3}$ & 35 & 8 & 3 & 3 & 7 & \textsc{to} & 716 \\
$\RS_{3, 2, 3}$ & 43 & 8 & 3 & 3 & 1 & 2.14 & 0.000483 \\
$\RS_{3, 2, 3}$ & 43 & 8 & 3 & 3 & 2 & 4.526 & 0.004735 \\
$\RS_{3, 2, 3}$ & 43 & 8 & 3 & 3 & 3 & 7.225 & 0.05159 \\
$\RS_{3, 2, 3}$ & 43 & 8 & 3 & 3 & 4 & \textsc{to} & 0.5671 \\
$\RS_{3, 2, 3}$ & 43 & 8 & 3 & 3 & 5 & \textsc{to} & 6.15 \\
$\RS_{3, 2, 3}$ & 43 & 8 & 3 & 3 & 6 & \textsc{to} & 66.86 \\
$\RS_{3, 2, 3}$ & 43 & 8 & 3 & 3 & 7 & \textsc{to} & 736.1 \\
$\RS_{3, 3, 3}$ & 21 & 8 & 3 & 1 & 1 & 2.797 & 0.000353 \\
$\RS_{3, 3, 3}$ & 21 & 8 & 3 & 1 & 2 & 5.977 & 0.002736 \\
$\RS_{3, 3, 3}$ & 21 & 8 & 3 & 1 & 3 & 9.338 & 0.02959 \\
$\RS_{3, 3, 3}$ & 21 & 8 & 3 & 1 & 4 & \textsc{to} & 0.3233 \\
$\RS_{3, 3, 3}$ & 21 & 8 & 3 & 1 & 5 & \textsc{to} & 3.525 \\
$\RS_{3, 3, 3}$ & 21 & 8 & 3 & 1 & 6 & 60.94 & 38.46 \\
$\RS_{3, 3, 3}$ & 21 & 8 & 3 & 1 & 7 & \textsc{to} & 421.8 \\
$\RS_{3, 2, 5}$ & 133 & 10 & 3 & 10 & 1 & 13.97 & 0.001451 \\
$\RS_{3, 2, 5}$ & 133 & 10 & 3 & 10 & 2 & \textsc{to} & 0.02059 \\
$\RS_{3, 2, 5}$ & 133 & 10 & 3 & 10 & 3 & 46.86 & 0.3065 \\
$\RS_{3, 2, 5}$ & 133 & 10 & 3 & 10 & 4 & 200.7 & 4.567 \\
$\RS_{3, 2, 5}$ & 133 & 10 & 3 & 10 & 5 & \textsc{to} & 67.1 \\
$\RS_{3, 2, 5}$ & 133 & 10 & 3 & 10 & 6 & \textsc{to} & 997.7 \\
$\RS_{3, 2, 5}$ & 133 & 10 & 3 & 10 & 7 & \textsc{to} & \textsc{to} \\
\bottomrule
\end{tabular}
\end{table}

\begin{table}[t]
\centering
\small
\caption{Rock sampling full benchmark results (Part 5 of 5).}
\label{tab:rock-sample-full-4}
\begin{tabular}{lrrrrrrr}
\toprule
Model & $|\POMDPStates|$ & $|\POMDPActions|$ & $|\POMDPObs|$ & $n$ & $\horizon$ & Time \cite{bovymulti} (s) & Time ours (s) \\
\midrule
$\RS_{3, 3, 5}$ & 173 & 10 & 3 & 10 & 1 & 36.18 & 0.006625 \\
$\RS_{3, 3, 5}$ & 173 & 10 & 3 & 10 & 2 & \textsc{to} & 0.02006 \\
$\RS_{3, 3, 5}$ & 173 & 10 & 3 & 10 & 3 & 121.7 & 0.2983 \\
$\RS_{3, 3, 5}$ & 173 & 10 & 3 & 10 & 4 & \textsc{to} & 4.436 \\
$\RS_{3, 3, 5}$ & 173 & 10 & 3 & 10 & 5 & \textsc{to} & 65.58 \\
$\RS_{3, 3, 5}$ & 173 & 10 & 3 & 10 & 6 & \textsc{to} & 1041 \\
$\RS_{3, 3, 5}$ & 173 & 10 & 3 & 10 & 7 & \textsc{to} & \textsc{to} \\
$\RS_{3, 4, 5}$ & 102 & 10 & 3 & 5 & 1 & 51.55 & 0.07048 \\
$\RS_{3, 4, 5}$ & 102 & 10 & 3 & 5 & 2 & 356 & 0.03017 \\
$\RS_{3, 4, 5}$ & 102 & 10 & 3 & 5 & 3 & 172.6 & 0.2155 \\
$\RS_{3, 4, 5}$ & 102 & 10 & 3 & 5 & 4 & \textsc{to} & 3.038 \\
$\RS_{3, 4, 5}$ & 102 & 10 & 3 & 5 & 5 & \textsc{to} & 44.98 \\
$\RS_{3, 4, 5}$ & 102 & 10 & 3 & 5 & 6 & \textsc{to} & 669.2 \\
$\RS_{3, 4, 5}$ & 102 & 10 & 3 & 5 & 7 & \textsc{to} & \textsc{to} \\
$\RS_{3, 2, 7}$ & 246 & 12 & 3 & 21 & 1 & 54.27 & 0.04466 \\
$\RS_{3, 2, 7}$ & 246 & 12 & 3 & 21 & 2 & 115.9 & 0.06734 \\
$\RS_{3, 2, 7}$ & 246 & 12 & 3 & 21 & 3 & 185.3 & 1.277 \\
$\RS_{3, 2, 7}$ & 246 & 12 & 3 & 21 & 4 & \textsc{to} & 24.3 \\
$\RS_{3, 2, 7}$ & 246 & 12 & 3 & 21 & 5 & \textsc{to} & 459.9 \\
$\RS_{3, 2, 7}$ & 246 & 12 & 3 & 21 & 6 & \textsc{to} & \textsc{to} \\
$\RS_{3, 2, 7}$ & 246 & 12 & 3 & 21 & 7 & \textsc{to} & \textsc{to} \\
$\RS_{3, 4, 7}$ & 574 & 12 & 3 & 35 & 1 & 627.6 & 0.005945 \\
$\RS_{3, 4, 7}$ & 574 & 12 & 3 & 35 & 2 & 1339 & 0.1132 \\
$\RS_{3, 4, 7}$ & 574 & 12 & 3 & 35 & 3 & 2129 & 2.157 \\
$\RS_{3, 4, 7}$ & 574 & 12 & 3 & 35 & 4 & \textsc{to} & 40.9 \\
$\RS_{3, 4, 7}$ & 574 & 12 & 3 & 35 & 5 & \textsc{to} & 783.8 \\
$\RS_{3, 4, 7}$ & 574 & 12 & 3 & 35 & 6 & \textsc{to} & \textsc{to} \\
$\RS_{3, 4, 7}$ & 574 & 12 & 3 & 35 & 7 & \textsc{to} & \textsc{to} \\
$\RS_{3, 6, 7}$ & 192 & 12 & 3 & 7 & 1 & 1029 & 0.001572 \\
$\RS_{3, 6, 7}$ & 192 & 12 & 3 & 7 & 2 & \textsc{to} & 0.02758 \\
$\RS_{3, 6, 7}$ & 192 & 12 & 3 & 7 & 3 & 3493 & 0.5194 \\
$\RS_{3, 6, 7}$ & 192 & 12 & 3 & 7 & 4 & \textsc{to} & 9.806 \\
$\RS_{3, 6, 7}$ & 192 & 12 & 3 & 7 & 5 & \textsc{to} & 184.6 \\
$\RS_{3, 6, 7}$ & 192 & 12 & 3 & 7 & 6 & \textsc{to} & 3477 \\
$\RS_{3, 6, 7}$ & 192 & 12 & 3 & 7 & 7 & \textsc{to} & \textsc{to} \\
\bottomrule
\end{tabular}
\end{table}

\begin{table}[t]
\centering
\small
\caption{Robot navigation full benchmark results (Part 1 of 4).}
\label{tab:robot-navigation-full-0}
\begin{tabular}{lrrrrrrr}
\toprule
Model & $|\POMDPStates|$ & $|\POMDPActions|$ & $|\POMDPObs|$ & $n$ & $\horizon$ & Time \cite{bovymulti} (s) & Time ours (s) \\
\midrule
$\RN_{\textsc{cit}, 1}$ & 33 & 4 & 28 & 2 & 1 & 151.8 & 0.002365 \\
$\RN_{\textsc{cit}, 1}$ & 33 & 4 & 28 & 2 & 2 & \textsc{to} & 0.04038 \\
$\RN_{\textsc{cit}, 1}$ & 33 & 4 & 28 & 2 & 3 & \textsc{to} & 0.2864 \\
$\RN_{\textsc{cit}, 1}$ & 33 & 4 & 28 & 2 & 4 & \textsc{to} & 4.136 \\
$\RN_{\textsc{cit}, 1}$ & 33 & 4 & 28 & 2 & 5 & \textsc{to} & 60.3 \\
$\RN_{\textsc{cit}, 1}$ & 33 & 4 & 28 & 2 & 6 & \textsc{to} & 826.8 \\
$\RN_{\textsc{cit}, 2}$ & 33 & 4 & 28 & 2 & 1 & 150.9 & 0.002247 \\
$\RN_{\textsc{cit}, 2}$ & 33 & 4 & 28 & 2 & 2 & 305.3 & 0.017 \\
$\RN_{\textsc{cit}, 2}$ & 33 & 4 & 28 & 2 & 3 & \textsc{to} & 0.2568 \\
$\RN_{\textsc{cit}, 2}$ & 33 & 4 & 28 & 2 & 4 & \textsc{to} & 3.935 \\
$\RN_{\textsc{cit}, 2}$ & 33 & 4 & 28 & 2 & 5 & \textsc{to} & 57.73 \\
$\RN_{\textsc{cit}, 2}$ & 33 & 4 & 28 & 2 & 6 & \textsc{to} & 786.3 \\
$\RN_{\textsc{cit}, 3}$ & 48 & 4 & 28 & 2 & 1 & 150.3 & 0.003128 \\
$\RN_{\textsc{cit}, 3}$ & 48 & 4 & 28 & 2 & 2 & 304.2 & 0.0223 \\
$\RN_{\textsc{cit}, 3}$ & 48 & 4 & 28 & 2 & 3 & 457.3 & 0.3325 \\
$\RN_{\textsc{cit}, 3}$ & 48 & 4 & 28 & 2 & 4 & \textsc{to} & 4.921 \\
$\RN_{\textsc{cit}, 3}$ & 48 & 4 & 28 & 2 & 5 & \textsc{to} & 69.29 \\
$\RN_{\textsc{cit}, 3}$ & 48 & 4 & 28 & 2 & 6 & \textsc{to} & 914.1 \\
\bottomrule
\end{tabular}
\end{table}

\begin{table}[t]
\centering
\small
\caption{Robot navigation full benchmark results (Part 2 of 4).}
\label{tab:robot-navigation-full-1}
\begin{tabular}{lrrrrrrr}
\toprule
Model & $|\POMDPStates|$ & $|\POMDPActions|$ & $|\POMDPObs|$ & $n$ & $\horizon$ & Time \cite{bovymulti} (s) & Time ours (s) \\
\midrule
$\RN_{\textsc{mit}, 1}$ & 47 & 4 & 28 & 2 & 1 & 77.93 & 0.00228 \\
$\RN_{\textsc{mit}, 1}$ & 47 & 4 & 28 & 2 & 2 & \textsc{to} & 0.01699 \\
$\RN_{\textsc{mit}, 1}$ & 47 & 4 & 28 & 2 & 3 & \textsc{to} & 0.2556 \\
$\RN_{\textsc{mit}, 1}$ & 47 & 4 & 28 & 2 & 4 & \textsc{to} & 3.737 \\
$\RN_{\textsc{mit}, 1}$ & 47 & 4 & 28 & 2 & 5 & \textsc{to} & \textsc{to} \\
$\RN_{\textsc{mit}, 1}$ & 47 & 4 & 28 & 2 & 6 & \textsc{to} & \textsc{to} \\
$\RN_{\textsc{mit}, 2}$ & 41 & 4 & 28 & 2 & 1 & 77.54 & 0.001988 \\
$\RN_{\textsc{mit}, 2}$ & 41 & 4 & 28 & 2 & 2 & 155.4 & 0.01937 \\
$\RN_{\textsc{mit}, 2}$ & 41 & 4 & 28 & 2 & 3 & \textsc{to} & 0.3086 \\
$\RN_{\textsc{mit}, 2}$ & 41 & 4 & 28 & 2 & 4 & \textsc{to} & 21.56 \\
$\RN_{\textsc{mit}, 2}$ & 41 & 4 & 28 & 2 & 5 & \textsc{to} & \textsc{to} \\
$\RN_{\textsc{mit}, 2}$ & 41 & 4 & 28 & 2 & 6 & \textsc{to} & \textsc{to} \\
$\RN_{\textsc{mit}, 3}$ & 36 & 4 & 28 & 2 & 1 & 78.39 & 0.04544 \\
$\RN_{\textsc{mit}, 3}$ & 36 & 4 & 28 & 2 & 2 & 156.8 & 0.02158 \\
$\RN_{\textsc{mit}, 3}$ & 36 & 4 & 28 & 2 & 3 & 238.3 & 0.3299 \\
$\RN_{\textsc{mit}, 3}$ & 36 & 4 & 28 & 2 & 4 & \textsc{to} & 4.802 \\
$\RN_{\textsc{mit}, 3}$ & 36 & 4 & 28 & 2 & 5 & \textsc{to} & 70.02 \\
$\RN_{\textsc{mit}, 3}$ & 36 & 4 & 28 & 2 & 6 & \textsc{to} & 1096 \\
\bottomrule
\end{tabular}
\end{table}

\begin{table}[t]
\centering
\small
\caption{Robot navigation full benchmark results (Part 3 of 4).}
\label{tab:robot-navigation-full-2}
\begin{tabular}{lrrrrrrr}
\toprule
Model & $|\POMDPStates|$ & $|\POMDPActions|$ & $|\POMDPObs|$ & $n$ & $\horizon$ & Time \cite{bovymulti} (s) & Time ours (s) \\
\midrule
$\RN_{\textsc{pen.}, 1}$ & 34 & 4 & 28 & 2 & 1 & 84.61 & 0.002054 \\
$\RN_{\textsc{pen.}, 1}$ & 34 & 4 & 28 & 2 & 2 & \textsc{to} & 0.05489 \\
$\RN_{\textsc{pen.}, 1}$ & 34 & 4 & 28 & 2 & 3 & \textsc{to} & 0.2966 \\
$\RN_{\textsc{pen.}, 1}$ & 34 & 4 & 28 & 2 & 4 & \textsc{to} & 4.433 \\
$\RN_{\textsc{pen.}, 1}$ & 34 & 4 & 28 & 2 & 5 & \textsc{to} & 65.52 \\
$\RN_{\textsc{pen.}, 1}$ & 34 & 4 & 28 & 2 & 6 & \textsc{to} & 905.4 \\
$\RN_{\textsc{pen.}, 2}$ & 35 & 4 & 28 & 2 & 1 & 85.13 & 0.002313 \\
$\RN_{\textsc{pen.}, 2}$ & 35 & 4 & 28 & 2 & 2 & 171.3 & 0.01818 \\
$\RN_{\textsc{pen.}, 2}$ & 35 & 4 & 28 & 2 & 3 & \textsc{to} & 0.2877 \\
$\RN_{\textsc{pen.}, 2}$ & 35 & 4 & 28 & 2 & 4 & \textsc{to} & 4.33 \\
$\RN_{\textsc{pen.}, 2}$ & 35 & 4 & 28 & 2 & 5 & \textsc{to} & 64.54 \\
$\RN_{\textsc{pen.}, 2}$ & 35 & 4 & 28 & 2 & 6 & \textsc{to} & 902.1 \\
$\RN_{\textsc{pen.}, 3}$ & 25 & 4 & 28 & 2 & 1 & 85.39 & 0.003148 \\
$\RN_{\textsc{pen.}, 3}$ & 25 & 4 & 28 & 2 & 2 & 170 & 0.01943 \\
$\RN_{\textsc{pen.}, 3}$ & 25 & 4 & 28 & 2 & 3 & 261.1 & 0.2908 \\
$\RN_{\textsc{pen.}, 3}$ & 25 & 4 & 28 & 2 & 4 & \textsc{to} & 4.37 \\
$\RN_{\textsc{pen.}, 3}$ & 25 & 4 & 28 & 2 & 5 & \textsc{to} & 64.95 \\
$\RN_{\textsc{pen.}, 3}$ & 25 & 4 & 28 & 2 & 6 & \textsc{to} & 919.8 \\
\bottomrule
\end{tabular}
\end{table}

\begin{table}[t]
\centering
\small
\caption{Robot navigation full benchmark results (Part 4 of 4).}
\label{tab:robot-navigation-full-3}
\begin{tabular}{lrrrrrrr}
\toprule
Model & $|\POMDPStates|$ & $|\POMDPActions|$ & $|\POMDPObs|$ & $n$ & $\horizon$ & Time \cite{bovymulti} (s) & Time ours (s) \\
\midrule
$\RN_{\textsc{sunysb}, 1}$ & 39 & 4 & 28 & 2 & 1 & 170 & 0.002059 \\
$\RN_{\textsc{sunysb}, 1}$ & 39 & 4 & 28 & 2 & 2 & \textsc{to} & 0.02027 \\
$\RN_{\textsc{sunysb}, 1}$ & 39 & 4 & 28 & 2 & 3 & \textsc{to} & 0.2996 \\
$\RN_{\textsc{sunysb}, 1}$ & 39 & 4 & 28 & 2 & 4 & \textsc{to} & 4.574 \\
$\RN_{\textsc{sunysb}, 1}$ & 39 & 4 & 28 & 2 & 5 & \textsc{to} & 67.88 \\
$\RN_{\textsc{sunysb}, 1}$ & 39 & 4 & 28 & 2 & 6 & \textsc{to} & 947.4 \\
$\RN_{\textsc{sunysb}, 2}$ & 29 & 4 & 28 & 2 & 1 & 168.6 & 0.002042 \\
$\RN_{\textsc{sunysb}, 2}$ & 29 & 4 & 28 & 2 & 2 & 338.9 & 0.01603 \\
$\RN_{\textsc{sunysb}, 2}$ & 29 & 4 & 28 & 2 & 3 & \textsc{to} & 0.2455 \\
$\RN_{\textsc{sunysb}, 2}$ & 29 & 4 & 28 & 2 & 4 & \textsc{to} & 3.717 \\
$\RN_{\textsc{sunysb}, 2}$ & 29 & 4 & 28 & 2 & 5 & \textsc{to} & 54.77 \\
$\RN_{\textsc{sunysb}, 2}$ & 29 & 4 & 28 & 2 & 6 & \textsc{to} & 778.2 \\
$\RN_{\textsc{sunysb}, 3}$ & 51 & 4 & 28 & 2 & 1 & 169.7 & 0.003113 \\
$\RN_{\textsc{sunysb}, 3}$ & 51 & 4 & 28 & 2 & 2 & 342.6 & 0.01639 \\
$\RN_{\textsc{sunysb}, 3}$ & 51 & 4 & 28 & 2 & 3 & 516.6 & 0.2474 \\
$\RN_{\textsc{sunysb}, 3}$ & 51 & 4 & 28 & 2 & 4 & \textsc{to} & 3.66 \\
$\RN_{\textsc{sunysb}, 3}$ & 51 & 4 & 28 & 2 & 5 & \textsc{to} & 52.69 \\
$\RN_{\textsc{sunysb}, 3}$ & 51 & 4 & 28 & 2 & 6 & \textsc{to} & 714 \\
\bottomrule
\end{tabular}
\end{table}

\begin{table}[t]
\centering
\small
\caption{Identification (friend or foe) full benchmark results (Part 1 of 3).}
\label{tab:iff-full-0}
\begin{tabular}{lrrrrrrr}
\toprule
Model & $|\POMDPStates|$ & $|\POMDPActions|$ & $|\POMDPObs|$ & $n$ & $\horizon$ & Time \cite{bovymulti} (s) & Time ours (s) \\
\midrule
$\IFF_{1, 2, 0, 2}$ & 20 & 4 & 22 & 3 & 1 & \textsc{to} & 0.003146 \\
$\IFF_{1, 2, 0, 2}$ & 20 & 4 & 22 & 3 & 2 & error & 0.03503 \\
$\IFF_{1, 2, 0, 2}$ & 20 & 4 & 22 & 3 & 3 & \textsc{to} & 0.4615 \\
$\IFF_{1, 2, 0, 2}$ & 20 & 4 & 22 & 3 & 4 & \textsc{to} & 4.773 \\
$\IFF_{1, 2, 0, 2}$ & 20 & 4 & 22 & 3 & 5 & \textsc{to} & 62.33 \\
$\IFF_{1, 2, 0, 2}$ & 20 & 4 & 22 & 3 & 6 & \textsc{to} & 1099 \\
$\IFF_{1, 2, 0, 4}$ & 18 & 4 & 22 & 3 & 1 & \textsc{to} & 0.1386 \\
$\IFF_{1, 2, 0, 4}$ & 18 & 4 & 22 & 3 & 2 & \textsc{to} & 0.04797 \\
$\IFF_{1, 2, 0, 4}$ & 18 & 4 & 22 & 3 & 3 & \textsc{to} & 0.3258 \\
$\IFF_{1, 2, 0, 4}$ & 18 & 4 & 22 & 3 & 4 & \textsc{to} & 3.342 \\
$\IFF_{1, 2, 0, 4}$ & 18 & 4 & 22 & 3 & 5 & \textsc{to} & 60.24 \\
$\IFF_{1, 2, 0, 4}$ & 18 & 4 & 22 & 3 & 6 & \textsc{to} & 412.6 \\
$\IFF_{1, 2, 2, 0}$ & 21 & 4 & 22 & 3 & 1 & \textsc{to} & 0.003343 \\
$\IFF_{1, 2, 2, 0}$ & 21 & 4 & 22 & 3 & 2 & error & 0.03263 \\
$\IFF_{1, 2, 2, 0}$ & 21 & 4 & 22 & 3 & 3 & \textsc{to} & 0.4695 \\
$\IFF_{1, 2, 2, 0}$ & 21 & 4 & 22 & 3 & 4 & \textsc{to} & 5.215 \\
$\IFF_{1, 2, 2, 0}$ & 21 & 4 & 22 & 3 & 5 & \textsc{to} & 48.06 \\
$\IFF_{1, 2, 2, 0}$ & 21 & 4 & 22 & 3 & 6 & \textsc{to} & 269.5 \\
$\IFF_{1, 2, 2, 4}$ & 15 & 4 & 22 & 3 & 1 & \textsc{to} & 0.139 \\
$\IFF_{1, 2, 2, 4}$ & 15 & 4 & 22 & 3 & 2 & error & 0.0605 \\
$\IFF_{1, 2, 2, 4}$ & 15 & 4 & 22 & 3 & 3 & \textsc{to} & 0.3132 \\
$\IFF_{1, 2, 2, 4}$ & 15 & 4 & 22 & 3 & 4 & \textsc{to} & 3.149 \\
$\IFF_{1, 2, 2, 4}$ & 15 & 4 & 22 & 3 & 5 & \textsc{to} & 37.9 \\
$\IFF_{1, 2, 2, 4}$ & 15 & 4 & 22 & 3 & 6 & \textsc{to} & 271 \\
$\IFF_{1, 2, 4, 0}$ & 21 & 4 & 22 & 3 & 1 & \textsc{to} & 0.1388 \\
$\IFF_{1, 2, 4, 0}$ & 21 & 4 & 22 & 3 & 2 & error & 0.06559 \\
$\IFF_{1, 2, 4, 0}$ & 21 & 4 & 22 & 3 & 3 & \textsc{to} & 0.4053 \\
$\IFF_{1, 2, 4, 0}$ & 21 & 4 & 22 & 3 & 4 & \textsc{to} & 4.727 \\
$\IFF_{1, 2, 4, 0}$ & 21 & 4 & 22 & 3 & 5 & \textsc{to} & 44.38 \\
$\IFF_{1, 2, 4, 0}$ & 21 & 4 & 22 & 3 & 6 & \textsc{to} & 244.1 \\
$\IFF_{1, 2, 4, 2}$ & 17 & 4 & 22 & 3 & 1 & \textsc{to} & 0.1388 \\
$\IFF_{1, 2, 4, 2}$ & 17 & 4 & 22 & 3 & 2 & error & 0.06575 \\
$\IFF_{1, 2, 4, 2}$ & 17 & 4 & 22 & 3 & 3 & \textsc{to} & 0.39 \\
$\IFF_{1, 2, 4, 2}$ & 17 & 4 & 22 & 3 & 4 & \textsc{to} & 4.096 \\
$\IFF_{1, 2, 4, 2}$ & 17 & 4 & 22 & 3 & 5 & \textsc{to} & 36.52 \\
$\IFF_{1, 2, 4, 2}$ & 17 & 4 & 22 & 3 & 6 & \textsc{to} & 239.5 \\
\bottomrule
\end{tabular}
\end{table}

\begin{table}[t]
\centering
\small
\caption{Identification (friend or foe) full benchmark results (Part 2 of 3).}
\label{tab:iff-full-1}
\begin{tabular}{lrrrrrrr}
\toprule
Model & $|\POMDPStates|$ & $|\POMDPActions|$ & $|\POMDPObs|$ & $n$ & $\horizon$ & Time \cite{bovymulti} (s) & Time ours (s) \\
\midrule
$\IFF_{1, 3, 0, 2}$ & 24 & 4 & 22 & 3 & 1 & \textsc{to} & 0.1376 \\
$\IFF_{1, 3, 0, 2}$ & 24 & 4 & 22 & 3 & 2 & 141.6 & 0.07371 \\
$\IFF_{1, 3, 0, 2}$ & 24 & 4 & 22 & 3 & 3 & 213.5 & 0.6066 \\
$\IFF_{1, 3, 0, 2}$ & 24 & 4 & 22 & 3 & 4 & error & 208 \\
$\IFF_{1, 3, 0, 2}$ & 24 & 4 & 22 & 3 & 5 & error & \textsc{to} \\
$\IFF_{1, 3, 0, 2}$ & 24 & 4 & 22 & 3 & 6 & error & \textsc{to} \\
$\IFF_{1, 3, 0, 4}$ & 20 & 4 & 22 & 3 & 1 & \textsc{to} & 0.003086 \\
$\IFF_{1, 3, 0, 4}$ & 20 & 4 & 22 & 3 & 2 & 141.9 & 0.02816 \\
$\IFF_{1, 3, 0, 4}$ & 20 & 4 & 22 & 3 & 3 & 408.1 & 0.4326 \\
$\IFF_{1, 3, 0, 4}$ & 20 & 4 & 22 & 3 & 4 & \textsc{to} & 130.8 \\
$\IFF_{1, 3, 0, 4}$ & 20 & 4 & 22 & 3 & 5 & \textsc{to} & \textsc{to} \\
$\IFF_{1, 3, 0, 4}$ & 20 & 4 & 22 & 3 & 6 & error & \textsc{to} \\
$\IFF_{1, 3, 2, 0}$ & 26 & 4 & 22 & 3 & 1 & \textsc{to} & 0.002936 \\
$\IFF_{1, 3, 2, 0}$ & 26 & 4 & 22 & 3 & 2 & 141.9 & 0.03683 \\
$\IFF_{1, 3, 2, 0}$ & 26 & 4 & 22 & 3 & 3 & 514.2 & 0.6137 \\
$\IFF_{1, 3, 2, 0}$ & 26 & 4 & 22 & 3 & 4 & error & 134 \\
$\IFF_{1, 3, 2, 0}$ & 26 & 4 & 22 & 3 & 5 & \textsc{to} & \textsc{to} \\
$\IFF_{1, 3, 2, 0}$ & 26 & 4 & 22 & 3 & 6 & \textsc{to} & \textsc{to} \\
$\IFF_{1, 3, 2, 4}$ & 17 & 4 & 22 & 3 & 1 & \textsc{to} & 0.003083 \\
$\IFF_{1, 3, 2, 4}$ & 17 & 4 & 22 & 3 & 2 & 142.5 & 0.02751 \\
$\IFF_{1, 3, 2, 4}$ & 17 & 4 & 22 & 3 & 3 & 515.4 & 0.4297 \\
$\IFF_{1, 3, 2, 4}$ & 17 & 4 & 22 & 3 & 4 & \textsc{to} & 131.8 \\
$\IFF_{1, 3, 2, 4}$ & 17 & 4 & 22 & 3 & 5 & \textsc{to} & \textsc{to} \\
$\IFF_{1, 3, 2, 4}$ & 17 & 4 & 22 & 3 & 6 & \textsc{to} & \textsc{to} \\
$\IFF_{1, 3, 4, 0}$ & 26 & 4 & 22 & 3 & 1 & 7.666 & 0.003359 \\
$\IFF_{1, 3, 4, 0}$ & 26 & 4 & 22 & 3 & 2 & error & 0.02967 \\
$\IFF_{1, 3, 4, 0}$ & 26 & 4 & 22 & 3 & 3 & error & 0.5556 \\
$\IFF_{1, 3, 4, 0}$ & 26 & 4 & 22 & 3 & 4 & error & 185.3 \\
$\IFF_{1, 3, 4, 0}$ & 26 & 4 & 22 & 3 & 5 & error & \textsc{to} \\
$\IFF_{1, 3, 4, 0}$ & 26 & 4 & 22 & 3 & 6 & error & \textsc{to} \\
$\IFF_{1, 3, 4, 2}$ & 21 & 4 & 22 & 3 & 1 & \textsc{to} & 0.004064 \\
$\IFF_{1, 3, 4, 2}$ & 21 & 4 & 22 & 3 & 2 & 165.2 & 0.03064 \\
$\IFF_{1, 3, 4, 2}$ & 21 & 4 & 22 & 3 & 3 & error & 0.5227 \\
$\IFF_{1, 3, 4, 2}$ & 21 & 4 & 22 & 3 & 4 & \textsc{to} & 114.9 \\
$\IFF_{1, 3, 4, 2}$ & 21 & 4 & 22 & 3 & 5 & \textsc{to} & \textsc{to} \\
$\IFF_{1, 3, 4, 2}$ & 21 & 4 & 22 & 3 & 6 & \textsc{to} & \textsc{to} \\
\bottomrule
\end{tabular}
\end{table}

\begin{table}[t]
\centering
\small
\caption{Identification (friend or foe) full benchmark results (Part 3 of 3).}
\label{tab:iff-full-2}
\begin{tabular}{lrrrrrrr}
\toprule
Model & $|\POMDPStates|$ & $|\POMDPActions|$ & $|\POMDPObs|$ & $n$ & $\horizon$ & Time \cite{bovymulti} (s) & Time ours (s) \\
\midrule
$\IFF_{2, 3, 0, 2}$ & 26 & 4 & 22 & 3 & 1 & \textsc{to} & 0.003104 \\
$\IFF_{2, 3, 0, 2}$ & 26 & 4 & 22 & 3 & 2 & 137.3 & 0.03796 \\
$\IFF_{2, 3, 0, 2}$ & 26 & 4 & 22 & 3 & 3 & \textsc{to} & 0.6793 \\
$\IFF_{2, 3, 0, 2}$ & 26 & 4 & 22 & 3 & 4 & \textsc{to} & 75.48 \\
$\IFF_{2, 3, 0, 2}$ & 26 & 4 & 22 & 3 & 5 & \textsc{to} & \textsc{to} \\
$\IFF_{2, 3, 0, 2}$ & 26 & 4 & 22 & 3 & 6 & \textsc{to} & \textsc{to} \\
$\IFF_{2, 3, 0, 4}$ & 24 & 4 & 22 & 3 & 1 & \textsc{to} & 0.002859 \\
$\IFF_{2, 3, 0, 4}$ & 24 & 4 & 22 & 3 & 2 & 136.7 & 0.02701 \\
$\IFF_{2, 3, 0, 4}$ & 24 & 4 & 22 & 3 & 3 & \textsc{to} & 0.4624 \\
$\IFF_{2, 3, 0, 4}$ & 24 & 4 & 22 & 3 & 4 & \textsc{to} & 47.96 \\
$\IFF_{2, 3, 0, 4}$ & 24 & 4 & 22 & 3 & 5 & \textsc{to} & \textsc{to} \\
$\IFF_{2, 3, 0, 4}$ & 24 & 4 & 22 & 3 & 6 & \textsc{to} & \textsc{to} \\
$\IFF_{2, 3, 2, 0}$ & 27 & 4 & 22 & 3 & 1 & \textsc{to} & 0.002836 \\
$\IFF_{2, 3, 2, 0}$ & 27 & 4 & 22 & 3 & 2 & 138.9 & 0.03382 \\
$\IFF_{2, 3, 2, 0}$ & 27 & 4 & 22 & 3 & 3 & \textsc{to} & 0.6166 \\
$\IFF_{2, 3, 2, 0}$ & 27 & 4 & 22 & 3 & 4 & \textsc{to} & 50 \\
$\IFF_{2, 3, 2, 0}$ & 27 & 4 & 22 & 3 & 5 & \textsc{to} & \textsc{to} \\
$\IFF_{2, 3, 2, 0}$ & 27 & 4 & 22 & 3 & 6 & \textsc{to} & \textsc{to} \\
$\IFF_{2, 3, 2, 4}$ & 19 & 4 & 22 & 3 & 1 & \textsc{to} & 0.002635 \\
$\IFF_{2, 3, 2, 4}$ & 19 & 4 & 22 & 3 & 2 & 142.7 & 0.02685 \\
$\IFF_{2, 3, 2, 4}$ & 19 & 4 & 22 & 3 & 3 & error & 0.4497 \\
$\IFF_{2, 3, 2, 4}$ & 19 & 4 & 22 & 3 & 4 & \textsc{to} & 38.59 \\
$\IFF_{2, 3, 2, 4}$ & 19 & 4 & 22 & 3 & 5 & \textsc{to} & \textsc{to} \\
$\IFF_{2, 3, 2, 4}$ & 19 & 4 & 22 & 3 & 6 & \textsc{to} & \textsc{to} \\
$\IFF_{2, 3, 4, 0}$ & 27 & 4 & 22 & 3 & 1 & \textsc{to} & 0.002888 \\
$\IFF_{2, 3, 4, 0}$ & 27 & 4 & 22 & 3 & 2 & 138.8 & 0.02922 \\
$\IFF_{2, 3, 4, 0}$ & 27 & 4 & 22 & 3 & 3 & \textsc{to} & 0.4841 \\
$\IFF_{2, 3, 4, 0}$ & 27 & 4 & 22 & 3 & 4 & \textsc{to} & 47.38 \\
$\IFF_{2, 3, 4, 0}$ & 27 & 4 & 22 & 3 & 5 & \textsc{to} & \textsc{to} \\
$\IFF_{2, 3, 4, 0}$ & 27 & 4 & 22 & 3 & 6 & \textsc{to} & \textsc{to} \\
$\IFF_{2, 3, 4, 2}$ & 21 & 4 & 22 & 3 & 1 & \textsc{to} & 0.003189 \\
$\IFF_{2, 3, 4, 2}$ & 21 & 4 & 22 & 3 & 2 & 139.1 & 0.02644 \\
$\IFF_{2, 3, 4, 2}$ & 21 & 4 & 22 & 3 & 3 & \textsc{to} & 0.4669 \\
$\IFF_{2, 3, 4, 2}$ & 21 & 4 & 22 & 3 & 4 & \textsc{to} & 63.63 \\
$\IFF_{2, 3, 4, 2}$ & 21 & 4 & 22 & 3 & 5 & \textsc{to} & \textsc{to} \\
$\IFF_{2, 3, 4, 2}$ & 21 & 4 & 22 & 3 & 6 & \textsc{to} & \textsc{to} \\
\bottomrule
\end{tabular}
\end{table}

\subsubsection{Comparison with approximate solution}
\label{sec:comparisonepsilon}
In this section, we study running times
when setting~\cite{bovymulti}'s tool with a positive tolerance ($\varepsilon>0$), that is, allowing for approximate solutions with an approximation error below $\varepsilon$.
In Table~\ref{tab:rocks_gaps} we report the running times for the full benchmark of the Rock sampling problem. For reference, we also include the running time for $\varepsilon=0$ (last column).   
Table~\ref{tab:gaps_robot} reports running times for different tolerance values for the robot navigation problem, and in Tables~\ref{tab:gaps_iff1} to~\ref{tab:gaps_iff5}, for the identification problem.

\begin{table}[t]
\centering
\small
\caption{Meaningful tolerances $\varepsilon>0$ that improve running time on the rock sampling benchmark.}
\label{tab:rocks_gaps}
\begin{tabular}{lrrrr}
\toprule
Model & $\horizon$ & $\varepsilon$ & Time \cite{bovymulti} (s) ($\varepsilon>0$) &Time \cite{bovymulti} (s) ($\varepsilon=0$) \\
\midrule$\RS_{3, 2, 3}$&7&10.0&214.734942&{\sc to}\\
$\RS_{3, 2, 3}$&7&16.666637&162.974267&{\sc to}\\
$\RS_{3, 2, 3}$&7&20.0&12.911721
&{\sc to}\\
$\RS_{3, 2, 3}$&7&14.0&205.202806&{\sc to}\\
$\RS_{3, 2, 3}$&7&16.666637&153.575063&{\sc to}\\
$\RS_{3, 2, 3}$&7&20.0&12.718293
&{\sc to}\\
$\RS_{3, 3, 3}$&7&10.0&65.583179&{\sc to}\\
$\RS_{3, 3, 3}$&7&20.0&16.556184
&{\sc to}\\
$\RS_{3, 2, 5}$&5&10.0&289.43802&{\sc to}\\
$\RS_{3, 2, 5}$&5&20.0&60.398253
&{\sc to}\\
$\RS_{3, 2, 5}$&6&9.999977&2850.84934&{\sc to}\\
$\RS_{3, 2, 5}$&6&10.0&989.714879&{\sc to}\\
$\RS_{3, 2, 5}$&6&15.0&580.956723&{\sc to}\\
$\RS_{3, 2, 5}$&6&20.0&73.022158
&{\sc to}\\
$\RS_{3, 2, 5}$&7&10.0&1349.932869&{\sc to}\\
$\RS_{3, 2, 5}$&7&15.0&1076.938035&{\sc to}\\
$\RS_{3, 2, 5}$&7&20.0&84.886007
&{\sc to}\\
$\RS_{3, 3, 5}$&6&15.0&1509.046957&{\sc to}\\
$\RS_{3, 3, 5}$&6&20.0&189.186391
&{\sc to}\\
$\RS_{3, 3, 5}$&7&19.375&3354.298341&{\sc to}\\
$\RS_{3, 3, 5}$&7&20.0&219.097685
&{\sc to}\\
$\RS_{3, 4, 5}$&5&6.666667&2290.876245&{\sc to}\\
$\RS_{3, 4, 5}$&5&10.0&222.28128
&{\sc to}\\
$\RS_{3, 4, 5}$&6&16.666667&3404.402459&{\sc to}\\
$\RS_{3, 4, 5}$&6&18.333333&2406.558198&{\sc to}\\
$\RS_{3, 4, 5}$&6&20.0&267.271925
&{\sc to}\\
$\RS_{3, 4, 5}$&7&13.333303&3427.793512&{\sc to}\\
$\RS_{3, 4, 5}$&7&20.0&311.159537
&{\sc to}\\
\bottomrule
\end{tabular}
\end{table}

\begin{table}[t]
\centering
\small
\caption{Meaningful tolerances $\varepsilon>0$ that improve running time on the robot navigation benchmark.}
\label{tab:gaps_robot}
\begin{tabular}{lrrrr}
\toprule
Model & $\horizon$ & $\varepsilon$ & Time \cite{bovymulti} (s) ($\varepsilon>0$) &Time \cite{bovymulti} (s) ($\varepsilon=0$) \\
\midrule
$\RN_{\textsc{cit}, 1}$&2&0.00271&1036.496869&{\sc to}\\
$\RN_{\textsc{cit}, 1}$&2&0.01&298.317198
&{\sc to}\\
$\RN_{\textsc{cit}, 1}$&3&0.054964&2651.807021&{\sc to}\\
$\RN_{\textsc{cit}, 1}$&3&0.7943&447.394049
&{\sc to}\\
$\RN_{\textsc{cit}, 1}$&4&0.010657&3560.988061&{\sc to}\\
$\RN_{\textsc{cit}, 1}$&4&0.966725&583.901878
&{\sc to}\\
$\RN_{\textsc{cit}, 2}$&3&0.00572&3386.211035&{\sc to}\\
$\RN_{\textsc{cit}, 2}$&3&0.009&446.538798
&{\sc to}\\
$\RN_{\textsc{cit}, 3}$&4&0.0001&3548.543892&{\sc to}\\
$\RN_{\textsc{cit}, 3}$&4&0.71289&594.832288
&{\sc to}\\
$\RN_{\textsc{mit}, 1}$&2&0.00271&518.342156&{\sc to}\\
$\RN_{\textsc{mit}, 1}$&2&0.01&150.553949
&{\sc to}\\
$\RN_{\textsc{mit}, 1}$&3&0.054379&2503.895705&{\sc to}\\
$\RN_{\textsc{mit}, 1}$&3&0.054964&1369.080751&{\sc to}\\
$\RN_{\textsc{mit}, 1}$&3&0.7943&230.148927
&{\sc to}\\
$\RN_{\textsc{mit}, 1}$&4&0.010592&3373.050409&{\sc to}\\
$\RN_{\textsc{mit}, 1}$&4&0.010657&1862.766506&{\sc to}\\
$\RN_{\textsc{mit}, 1}$&4&0.966725&311.284316
&{\sc to}\\
$\RN_{\textsc{mit}, 1}$&5&0.008201&2329.608765&{\sc to}\\
$\RN_{\textsc{mit}, 1}$&5&0.995168&383.502956
&{\sc to}\\
$\RN_{\textsc{mit}, 1}$&6&0.008869&2814.605176&{\sc to}\\
$\RN_{\textsc{mit}, 1}$&6&0.99934&460.80324
&{\sc to}\\
$\RN_{\textsc{pen}, 1}$&2&0.00271&1858.005867&{\sc to}\\
$\RN_{\textsc{pen}, 1}$&2&0.01&999.201917&{\sc to}\\
$\RN_{\textsc{pen}, 1}$&2&0.89&164.502555
&{\sc to}\\
$\RN_{\textsc{pen}, 1}$&3&0.054379&2777.488685&{\sc to}\\
$\RN_{\textsc{pen}, 1}$&3&0.054964&1513.262795&{\sc to}\\
$\RN_{\textsc{pen}, 1}$&3&0.9879&250.414842
&{\sc to}\\
$\RN_{\textsc{pen}, 1}$&4&0.069944&2026.016335&{\sc to}\\
$\RN_{\textsc{pen}, 1}$&4&0.998669&327.671769
&{\sc to}\\
$\RN_{\textsc{pen}, 1}$&5&0.080113&2548.471477&{\sc to}\\
$\RN_{\textsc{pen}, 1}$&5&0.999854&413.853363
&{\sc to}\\
$\RN_{\textsc{pen}, 1}$&6&0.082247&3070.93714&{\sc to}\\
$\RN_{\textsc{pen}, 1}$&6&0.999984&498.117854
&{\sc to}\\
$\RN_{\textsc{pen}, 2}$&3&0.0177&1492.731108&{\sc to}\\
$\RN_{\textsc{pen}, 2}$&3&0.801&251.091594
&{\sc to}\\
$\RN_{\textsc{pen}, 2}$&4&0.71081&2023.965727&{\sc to}\\
$\RN_{\textsc{pen}, 2}$&4&0.96921&335.280305
&{\sc to}\\
$\RN_{\textsc{pen}, 2}$&5&0.938894&2549.003754&{\sc to}\\
$\RN_{\textsc{pen}, 2}$&5&0.995723&418.483646
&{\sc to}\\
$\RN_{\textsc{pen}, 2}$&6&0.987306&3085.625036&{\sc to}\\
$\RN_{\textsc{pen}, 2}$&6&0.999441&506.778236
&{\sc to}\\
$\RN_{\textsc{pen}, 3}$&4&0.010124&2833.602887&{\sc to}\\
$\RN_{\textsc{pen}, 3}$&4&0.01593&332.856869
&{\sc to}\\
$\RN_{\textsc{sunysb}, 1}$&2&0.00271&1156.637073&{\sc to}\\
$\RN_{\textsc{sunysb}, 1}$&2&0.01&331.517546
&{\sc to}\\
$\RN_{\textsc{sunysb}, 1}$&3&0.054964&2993.939058&{\sc to}\\
$\RN_{\textsc{sunysb}, 1}$&3&0.7943&500.289369
&{\sc to}\\
$\RN_{\textsc{sunysb}, 2}$&3&0.000135&2124.813691&{\sc to}\\
$\RN_{\textsc{sunysb}, 2}$&3&0.009&493.375661
&{\sc to}\\
\bottomrule
\end{tabular}
\end{table}

\begin{table}[t]
\centering
\small
\caption{Meaningful tolerances $\varepsilon>0$ that improve running time on the identification (friend of foe) benchmark (Part 1 of 5).}
\label{tab:gaps_iff1}
\begin{tabular}{lrrrr}
\toprule
Model & $\horizon$ & $\varepsilon$ & Time \cite{bovymulti} (s) ($\varepsilon>0$) &Time \cite{bovymulti} (s) ($\varepsilon=0$) \\
\midrule
$\IFF_{1, 2, 0, 2}$&2&0.451222&76.474819&error\\
$\IFF_{1, 2, 0, 2}$&2&10.335019&11.986798
&error\\
$\IFF_{1, 2, 0, 2}$&3&0.064845&304.878326&{\sc to}\\
$\IFF_{1, 2, 0, 2}$&3&0.065956&208.809091&{\sc to}\\
$\IFF_{1, 2, 0, 2}$&3&0.434824&113.755056&{\sc to}\\
$\IFF_{1, 2, 0, 2}$&3&10.366268&18.006563
&{\sc to}\\
$\IFF_{1, 2, 0, 2}$&4&0.057703&372.86874&{\sc to}\\
$\IFF_{1, 2, 0, 2}$&4&0.062779&245.994885&{\sc to}\\
$\IFF_{1, 2, 0, 2}$&4&0.425804&119.642974&{\sc to}\\
$\IFF_{1, 2, 0, 2}$&4&10.406565&23.824224
&{\sc to}\\
$\IFF_{1, 2, 0, 2}$&5&0.425539&126.517056&{\sc to}\\
$\IFF_{1, 2, 0, 2}$&5&10.409824&29.661885
&{\sc to}\\
$\IFF_{1, 2, 0, 2}$&6&0.05706&418.433589&{\sc to}\\
$\IFF_{1, 2, 0, 2}$&6&0.06602&323.31635&{\sc to}\\
$\IFF_{1, 2, 0, 2}$&6&0.428498&228.396439&{\sc to}\\
$\IFF_{1, 2, 0, 2}$&6&10.410062&35.603996
&{\sc to}\\
$\IFF_{1, 2, 0, 4}$&2&2.686185&75.570873&{\sc to}\\
$\IFF_{1, 2, 0, 4}$&2&10.335019&11.990939
&{\sc to}\\
$\IFF_{1, 2, 0, 4}$&3&2.625939&177.047695&{\sc to}\\
$\IFF_{1, 2, 0, 4}$&3&10.335019&113.646608&{\sc to}\\
$\IFF_{1, 2, 0, 4}$&3&10.407664&17.965753
&{\sc to}\\
$\IFF_{1, 2, 0, 4}$&4&2.639832&151.483333&{\sc to}\\
$\IFF_{1, 2, 0, 4}$&4&10.335019&88.296378&{\sc to}\\
$\IFF_{1, 2, 0, 4}$&4&10.410011&23.893436
&{\sc to}\\
$\IFF_{1, 2, 0, 4}$&5&2.639779&157.624318&{\sc to}\\
$\IFF_{1, 2, 0, 4}$&5&10.335019&94.745284&{\sc to}\\
$\IFF_{1, 2, 0, 4}$&5&10.410096&29.889675
&{\sc to}\\
$\IFF_{1, 2, 0, 4}$&6&2.639779&164.085923&{\sc to}\\
$\IFF_{1, 2, 0, 4}$&6&10.335019&101.312377&{\sc to}\\
$\IFF_{1, 2, 0, 4}$&6&10.41008&35.72717
&{\sc to}\\
$\IFF_{1, 2, 2, 0}$&2&0.451222&76.950439&error\\
$\IFF_{1, 2, 2, 0}$&2&10.335019&12.051624
&error\\
$\IFF_{1, 2, 2, 0}$&3&0.05911&306.731205&{\sc to}\\
$\IFF_{1, 2, 2, 0}$&3&0.060956&210.168062&{\sc to}\\
$\IFF_{1, 2, 2, 0}$&3&0.434824&114.148018&{\sc to}\\
$\IFF_{1, 2, 2, 0}$&3&10.366268&17.919671
&{\sc to}\\
$\IFF_{1, 2, 2, 0}$&4&0.058062&375.554867&{\sc to}\\
$\IFF_{1, 2, 2, 0}$&4&0.058322&247.797874&{\sc to}\\
$\IFF_{1, 2, 2, 0}$&4&0.425804&120.340503&{\sc to}\\
$\IFF_{1, 2, 2, 0}$&4&10.406565&24.000758
&{\sc to}\\
$\IFF_{1, 2, 2, 0}$&5&0.058477&254.674578&{\sc to}\\
$\IFF_{1, 2, 2, 0}$&5&0.425439&127.239174&{\sc to}\\
$\IFF_{1, 2, 2, 0}$&5&10.409824&30.169275
&{\sc to}\\
$\IFF_{1, 2, 2, 0}$&6&0.05822&229.273519&{\sc to}\\
$\IFF_{1, 2, 2, 0}$&6&0.425478&133.777427&{\sc to}\\
$\IFF_{1, 2, 2, 0}$&6&10.410062&35.914999
&{\sc to}\\
$\IFF_{1, 2, 2, 4}$&2&0.451222&76.717783&error\\
$\IFF_{1, 2, 2, 4}$&2&10.335019&12.056523
&error\\
$\IFF_{1, 2, 2, 4}$&3&0.065162&307.229459&{\sc to}\\
$\IFF_{1, 2, 2, 4}$&3&0.066234&210.491118&{\sc to}\\
$\IFF_{1, 2, 2, 4}$&3&0.434824&114.234266&{\sc to}\\
$\IFF_{1, 2, 2, 4}$&3&10.366268&18.074022
&{\sc to}\\
$\IFF_{1, 2, 2, 4}$&4&0.057943&375.2&{\sc to}\\
$\IFF_{1, 2, 2, 4}$&4&0.063019&247.464541&{\sc to}\\
$\IFF_{1, 2, 2, 4}$&4&0.425804&120.09296&{\sc to}\\
$\IFF_{1, 2, 2, 4}$&4&10.406565&23.917012
&{\sc to}\\
$\IFF_{1, 2, 2, 4}$&5&0.062817&381.827551&{\sc to}\\
$\IFF_{1, 2, 2, 4}$&5&0.063177&254.211061&{\sc to}\\
$\IFF_{1, 2, 2, 4}$&5&0.425539&126.935491&{\sc to}\\
$\IFF_{1, 2, 2, 4}$&5&10.409824&29.968624
&{\sc to}\\
\bottomrule
\end{tabular}
\end{table}

\begin{table}[t]
\centering
\small
\caption{Meaningful tolerances $\varepsilon>0$ that improve running time on the identification (friend of foe) benchmark (Part 2 of 5).}
\label{tab:gaps_iff2}
\begin{tabular}{lrrrr}
\toprule
Model & $\horizon$ & $\varepsilon$ & Time \cite{bovymulti} (s) ($\varepsilon>0$) &Time \cite{bovymulti} (s) ($\varepsilon=0$) \\
\midrule
$\IFF_{1, 2, 2, 4}$&6&0.06292&324.694133&{\sc to}\\
$\IFF_{1, 2, 2, 4}$&6&0.06566&228.996656&{\sc to}\\
$\IFF_{1, 2, 2, 4}$&6&0.425638&133.358436&{\sc to}\\
$\IFF_{1, 2, 2, 4}$&6&10.410062&35.923125
&{\sc to}\\
$\IFF_{1, 2, 4, 0}$&2&0.451222&76.964088&error\\
$\IFF_{1, 2, 4, 0}$&2&10.335019&12.070612
&error\\
$\IFF_{1, 2, 4, 0}$&3&0.05911&306.711583&{\sc to}\\
$\IFF_{1, 2, 4, 0}$&3&0.060956&209.902358&{\sc to}\\
$\IFF_{1, 2, 4, 0}$&3&0.434824&113.962958&{\sc to}\\
$\IFF_{1, 2, 4, 0}$&3&10.366268&17.928714
&{\sc to}\\
$\IFF_{1, 2, 4, 0}$&4&0.058062&375.746958&{\sc to}\\
$\IFF_{1, 2, 4, 0}$&4&0.058322&247.673112&{\sc to}\\
$\IFF_{1, 2, 4, 0}$&4&0.425804&119.923301&{\sc to}\\
$\IFF_{1, 2, 4, 0}$&4&10.406565&23.746487
&{\sc to}\\
$\IFF_{1, 2, 4, 0}$&5&0.058477&255.065468&{\sc to}\\
$\IFF_{1, 2, 4, 0}$&5&0.425439&127.138802&{\sc to}\\
$\IFF_{1, 2, 4, 0}$&5&10.409824&30.059509
&{\sc to}\\
$\IFF_{1, 2, 4, 0}$&6&0.05822&229.04807&{\sc to}\\
$\IFF_{1, 2, 4, 0}$&6&0.425478&133.323256&{\sc to}\\
$\IFF_{1, 2, 4, 0}$&6&10.410062&35.868922
&{\sc to}\\
$\IFF_{1, 2, 4, 2}$&2&0.451222&78.266124&error\\
$\IFF_{1, 2, 4, 2}$&2&10.335019&12.128613
&error\\
$\IFF_{1, 2, 4, 2}$&3&0.064845&311.671375&{\sc to}\\
$\IFF_{1, 2, 4, 2}$&3&0.065956&213.738168&{\sc to}\\
$\IFF_{1, 2, 4, 2}$&3&0.434824&116.052927&{\sc to}\\
$\IFF_{1, 2, 4, 2}$&3&10.366268&18.208854
&{\sc to}\\
$\IFF_{1, 2, 4, 2}$&4&0.057703&382.914156&{\sc to}\\
$\IFF_{1, 2, 4, 2}$&4&0.062779&252.617002&{\sc to}\\
$\IFF_{1, 2, 4, 2}$&4&0.425804&123.428096&{\sc to}\\
$\IFF_{1, 2, 4, 2}$&4&10.406565&24.658906
&{\sc to}\\
$\IFF_{1, 2, 4, 2}$&5&0.425539&130.144402&{\sc to}\\
$\IFF_{1, 2, 4, 2}$&5&10.409824&31.097878
&{\sc to}\\
$\IFF_{1, 2, 4, 2}$&6&0.05706&428.610111&{\sc to}\\
$\IFF_{1, 2, 4, 2}$&6&0.06602&331.335254&{\sc to}\\
$\IFF_{1, 2, 4, 2}$&6&0.428498&234.21308&{\sc to}\\
$\IFF_{1, 2, 4, 2}$&6&10.410062&36.590249
&{\sc to}\\
$\IFF_{1, 3, 0, 2}$&2&1.467944&77.059827&141.641348\\
$\IFF_{1, 3, 0, 2}$&2&10.335019&11.872165
&141.641348\\
$\IFF_{1, 3, 0, 2}$&3&0.381255&116.006194&213.48041\\
$\IFF_{1, 3, 0, 2}$&3&10.335019&17.900276
&213.48041\\
$\IFF_{1, 3, 0, 2}$&4&0.388136&154.878372&error\\
$\IFF_{1, 3, 0, 2}$&4&10.335019&23.726062
&error\\
$\IFF_{1, 3, 0, 2}$&5&0.38856&161.283435&error\\
$\IFF_{1, 3, 0, 2}$&5&10.335019&29.889981
&error\\
$\IFF_{1, 3, 0, 2}$&6&0.20606&136.066608&error\\
$\IFF_{1, 3, 0, 2}$&6&10.335019&35.956934
&error\\
$\IFF_{1, 3, 0, 4}$&2&0.383022&77.237247&141.917578\\
$\IFF_{1, 3, 0, 4}$&2&10.335019&11.878174
&141.917578\\
$\IFF_{1, 3, 0, 4}$&3&0.317788&115.205143&408.103823\\
$\IFF_{1, 3, 0, 4}$&3&10.335019&17.901375
&408.103823\\
$\IFF_{1, 3, 0, 4}$&4&0.005135&799.717785&{\sc to}\\
$\IFF_{1, 3, 0, 4}$&4&0.00579&671.655286&{\sc to}\\
$\IFF_{1, 3, 0, 4}$&4&0.041855&281.414459&{\sc to}\\
$\IFF_{1, 3, 0, 4}$&4&0.362956&152.851509&{\sc to}\\
$\IFF_{1, 3, 0, 4}$&4&10.335019&23.868235
&{\sc to}\\
$\IFF_{1, 3, 0, 4}$&5&0.033362&317.795578&{\sc to}\\
$\IFF_{1, 3, 0, 4}$&5&0.360447&158.648246&{\sc to}\\
$\IFF_{1, 3, 0, 4}$&5&10.335019&29.803402
&{\sc to}\\
\bottomrule
\end{tabular}
\end{table}

\begin{table}[t]
\centering
\small
\caption{Meaningful tolerances $\varepsilon>0$ that improve running time on the identification (friend of foe) benchmark (Part 3 of 5).}
\label{tab:gaps_iff3}
\begin{tabular}{lrrrr}
\toprule
Model & $\horizon$ & $\varepsilon$ & Time \cite{bovymulti} (s) ($\varepsilon>0$) &Time \cite{bovymulti} (s) ($\varepsilon=0$) \\
\midrule
$\IFF_{1, 3, 0, 4}$&6&0.00638&908.686747&error\\
$\IFF_{1, 3, 0, 4}$&6&0.018979&747.082145&error\\
$\IFF_{1, 3, 0, 4}$&6&0.044894&393.328231&error\\
$\IFF_{1, 3, 0, 4}$&6&0.36543&234.580844&error\\
$\IFF_{1, 3, 0, 4}$&6&10.335019&35.527858
&error\\
$\IFF_{1, 3, 2, 0}$&2&0.383022&77.624215&141.91681\\
$\IFF_{1, 3, 2, 0}$&2&10.335019&11.977965
&141.91681\\
$\IFF_{1, 3, 2, 0}$&3&0.132984&413.774795&514.208244\\
$\IFF_{1, 3, 2, 0}$&3&0.317788&115.672065&514.208244\\
$\IFF_{1, 3, 2, 0}$&3&10.335019&17.909951
&514.208244\\
$\IFF_{1, 3, 2, 0}$&4&0.004996&542.509162&error\\
$\IFF_{1, 3, 2, 0}$&4&0.005611&413.818409&error\\
$\IFF_{1, 3, 2, 0}$&4&0.023733&284.820704&error\\
$\IFF_{1, 3, 2, 0}$&4&0.362956&152.84294&error\\
$\IFF_{1, 3, 2, 0}$&4&10.335019&23.841004
&error\\
$\IFF_{1, 3, 2, 0}$&5&0.005733&514.046505&{\sc to}\\
$\IFF_{1, 3, 2, 0}$&5&0.010987&352.233895&{\sc to}\\
$\IFF_{1, 3, 2, 0}$&5&0.360666&191.511139&{\sc to}\\
$\IFF_{1, 3, 2, 0}$&5&10.335019&30.038704
&{\sc to}\\
$\IFF_{1, 3, 2, 0}$&6&0.005239&1265.199023&{\sc to}\\
$\IFF_{1, 3, 2, 0}$&6&0.005559&1101.958233&{\sc to}\\
$\IFF_{1, 3, 2, 0}$&6&0.005679&486.164358&{\sc to}\\
$\IFF_{1, 3, 2, 0}$&6&0.021117&326.146578&{\sc to}\\
$\IFF_{1, 3, 2, 0}$&6&0.360271&166.268666&{\sc to}\\
$\IFF_{1, 3, 2, 0}$&6&10.335019&36.182132
&{\sc to}\\
$\IFF_{1, 3, 2, 4}$&2&0.383022&77.501331&142.514123\\
$\IFF_{1, 3, 2, 4}$&2&10.335019&11.963282
&142.514123\\
$\IFF_{1, 3, 2, 4}$&3&0.132984&414.983334&515.357185\\
$\IFF_{1, 3, 2, 4}$&3&0.317788&115.328902&515.357185\\
$\IFF_{1, 3, 2, 4}$&3&10.335019&17.867992
&515.357185\\
$\IFF_{1, 3, 2, 4}$&4&0.005135&943.774142&{\sc to}\\
$\IFF_{1, 3, 2, 4}$&4&0.00579&814.402909&{\sc to}\\
$\IFF_{1, 3, 2, 4}$&4&0.041855&281.808707&{\sc to}\\
$\IFF_{1, 3, 2, 4}$&4&0.362956&152.863577&{\sc to}\\
$\IFF_{1, 3, 2, 4}$&4&10.335019&23.984778
&{\sc to}\\
$\IFF_{1, 3, 2, 4}$&5&0.033362&319.50252&{\sc to}\\
$\IFF_{1, 3, 2, 4}$&5&0.360447&159.380286&{\sc to}\\
$\IFF_{1, 3, 2, 4}$&5&10.335019&29.55917
&{\sc to}\\
$\IFF_{1, 3, 2, 4}$&6&0.005099&1212.249705&{\sc to}\\
$\IFF_{1, 3, 2, 4}$&6&0.005179&918.751541&{\sc to}\\
$\IFF_{1, 3, 2, 4}$&6&0.020398&754.95113&{\sc to}\\
$\IFF_{1, 3, 2, 4}$&6&0.044894&398.138069&{\sc to}\\
$\IFF_{1, 3, 2, 4}$&6&0.36543&237.619894&{\sc to}\\
$\IFF_{1, 3, 2, 4}$&6&10.335019&35.887154
&{\sc to}\\
$\IFF_{1, 3, 4, 2}$&2&0.383022&90.340224&165.237963\\
$\IFF_{1, 3, 4, 2}$&2&10.335019&13.809389
&165.237963\\
$\IFF_{1, 3, 4, 2}$&3&0.044519&488.611768&error\\
$\IFF_{1, 3, 4, 2}$&3&0.317788&135.421368&error\\
$\IFF_{1, 3, 4, 2}$&3&10.335019&20.737646
&error\\
$\IFF_{1, 3, 4, 2}$&4&0.362956&178.707634&{\sc to}\\
$\IFF_{1, 3, 4, 2}$&4&10.335019&27.602859
&{\sc to}\\
$\IFF_{1, 3, 4, 2}$&5&0.004695&1135.059032&{\sc to}\\
$\IFF_{1, 3, 4, 2}$&5&0.005234&942.89021&{\sc to}\\
$\IFF_{1, 3, 4, 2}$&5&0.039595&410.545458&{\sc to}\\
$\IFF_{1, 3, 4, 2}$&5&0.360686&223.484471&{\sc to}\\
$\IFF_{1, 3, 4, 2}$&5&10.335019&34.628805
&{\sc to}\\
$\IFF_{1, 3, 4, 2}$&6&0.004599&1252.686854&{\sc to}\\
$\IFF_{1, 3, 4, 2}$&6&0.005719&645.729229&{\sc to}\\
$\IFF_{1, 3, 4, 2}$&6&0.035855&421.755427&{\sc to}\\
$\IFF_{1, 3, 4, 2}$&6&0.360351&193.534618&{\sc to}\\
\bottomrule
\end{tabular}
\end{table}

\begin{table}[t]
\centering
\small
\caption{Meaningful tolerances $\varepsilon>0$ that improve running time on the identification (friend of foe) benchmark (Part 4 of 5).}
\label{tab:gaps_iff4}
\begin{tabular}{lrrrr}
\toprule
Model & $\horizon$ & $\varepsilon$ & Time \cite{bovymulti} (s) ($\varepsilon>0$) &Time \cite{bovymulti} (s) ($\varepsilon=0$) \\
\midrule
$\IFF_{1, 3, 4, 2}$&6&10.335019&41.381999
&{\sc to}\\
$\IFF_{2, 3, 0, 2}$&2&1.939408&74.737667&137.343105\\
$\IFF_{2, 3, 0, 2}$&2&10.335019&11.74911
&137.343105\\
$\IFF_{2, 3, 0, 2}$&3&0.293408&208.601651&{\sc to}\\
$\IFF_{2, 3, 0, 2}$&3&1.884998&113.63439&{\sc to}\\
$\IFF_{2, 3, 0, 2}$&3&10.335019&17.764729
&{\sc to}\\
$\IFF_{2, 3, 0, 2}$&4&0.245361&405.335497&{\sc to}\\
$\IFF_{2, 3, 0, 2}$&4&0.283536&278.305535&{\sc to}\\
$\IFF_{2, 3, 0, 2}$&4&1.801674&151.392554&{\sc to}\\
$\IFF_{2, 3, 0, 2}$&4&10.335019&23.751316
&{\sc to}\\
$\IFF_{2, 3, 0, 2}$&5&0.274026&286.240798&{\sc to}\\
$\IFF_{2, 3, 0, 2}$&5&1.631319&159.607277&{\sc to}\\
$\IFF_{2, 3, 0, 2}$&5&10.335019&29.65847
&{\sc to}\\
$\IFF_{2, 3, 0, 2}$&6&0.283859&260.947321&{\sc to}\\
$\IFF_{2, 3, 0, 2}$&6&1.631274&166.078665&{\sc to}\\
$\IFF_{2, 3, 0, 2}$&6&10.335019&35.601055
&{\sc to}\\
$\IFF_{2, 3, 0, 4}$&2&1.939408&74.241271&136.696771\\
$\IFF_{2, 3, 0, 4}$&2&10.335019&11.72567
&136.696771\\
$\IFF_{2, 3, 0, 4}$&3&0.293408&207.000162&{\sc to}\\
$\IFF_{2, 3, 0, 4}$&3&1.884998&112.703548&{\sc to}\\
$\IFF_{2, 3, 0, 4}$&3&10.335019&17.591871
&{\sc to}\\
$\IFF_{2, 3, 0, 4}$&4&0.245361&400.729476&{\sc to}\\
$\IFF_{2, 3, 0, 4}$&4&0.283536&275.260925&{\sc to}\\
$\IFF_{2, 3, 0, 4}$&4&1.801674&150.027988&{\sc to}\\
$\IFF_{2, 3, 0, 4}$&4&10.335019&23.365801
&{\sc to}\\
$\IFF_{2, 3, 0, 4}$&5&0.274026&283.948861&{\sc to}\\
$\IFF_{2, 3, 0, 4}$&5&1.631319&158.464371&{\sc to}\\
$\IFF_{2, 3, 0, 4}$&5&10.335019&29.625754
&{\sc to}\\
$\IFF_{2, 3, 0, 4}$&6&0.283859&259.304075&{\sc to}\\
$\IFF_{2, 3, 0, 4}$&6&1.631274&164.947204&{\sc to}\\
$\IFF_{2, 3, 0, 4}$&6&10.335019&35.405053
&{\sc to}\\
$\IFF_{2, 3, 2, 0}$&2&1.939408&75.681753&138.915399\\
$\IFF_{2, 3, 2, 0}$&2&10.335019&11.963679
&138.915399\\
$\IFF_{2, 3, 2, 0}$&3&0.318409&211.942822&{\sc to}\\
$\IFF_{2, 3, 2, 0}$&3&1.884998&115.205659&{\sc to}\\
$\IFF_{2, 3, 2, 0}$&3&10.335019&18.052294
&{\sc to}\\
$\IFF_{2, 3, 2, 0}$&4&0.274243&409.427494&{\sc to}\\
$\IFF_{2, 3, 2, 0}$&4&0.305782&281.320625&{\sc to}\\
$\IFF_{2, 3, 2, 0}$&4&1.801674&153.524426&{\sc to}\\
$\IFF_{2, 3, 2, 0}$&4&10.335019&24.174249
&{\sc to}\\
$\IFF_{2, 3, 2, 0}$&5&0.306825&451.002362&{\sc to}\\
$\IFF_{2, 3, 2, 0}$&5&0.307405&322.661713&{\sc to}\\
$\IFF_{2, 3, 2, 0}$&5&1.650738&194.726674&{\sc to}\\
$\IFF_{2, 3, 2, 0}$&5&10.335019&30.014687
&{\sc to}\\
$\IFF_{2, 3, 2, 0}$&6&0.294539&361.6288&{\sc to}\\
$\IFF_{2, 3, 2, 0}$&6&0.305259&265.162844&{\sc to}\\
$\IFF_{2, 3, 2, 0}$&6&1.631694&168.737673&{\sc to}\\
$\IFF_{2, 3, 2, 0}$&6&10.335019&35.979194
&{\sc to}\\
$\IFF_{2, 3, 2, 4}$&2&0.383022&77.552645&142.687529\\
$\IFF_{2, 3, 2, 4}$&2&10.335019&11.985888
&142.687529\\
$\IFF_{2, 3, 2, 4}$&3&0.044519&504.559813&error\\
$\IFF_{2, 3, 2, 4}$&3&0.280943&407.277901&error\\
$\IFF_{2, 3, 2, 4}$&3&0.317788&115.656604&error\\
$\IFF_{2, 3, 2, 4}$&3&10.335019&17.916342
&error\\
$\IFF_{2, 3, 2, 4}$&4&0.005215&540.656027&{\sc to}\\
$\IFF_{2, 3, 2, 4}$&4&0.00579&411.262733&{\sc to}\\
$\IFF_{2, 3, 2, 4}$&4&0.041855&282.378471&{\sc to}\\
$\IFF_{2, 3, 2, 4}$&4&0.362956&153.316284&{\sc to}\\
$\IFF_{2, 3, 2, 4}$&4&10.335019&23.988811
&{\sc to}\\
\bottomrule
\end{tabular}
\end{table}

\begin{table}[t]
\centering
\small
\caption{Meaningful tolerances $\varepsilon>0$ that improve running time on the identification (friend of foe) benchmark (Part 5 of 5).}
\label{tab:gaps_iff5}
\begin{tabular}{lrrrr}
\toprule
Model & $\horizon$ & $\varepsilon$ & Time \cite{bovymulti} (s) ($\varepsilon>0$) &Time \cite{bovymulti} (s) ($\varepsilon=0$) \\
\midrule
$\IFF_{2, 3, 2, 4}$&5&0.033362&321.565696&{\sc to}\\
$\IFF_{2, 3, 2, 4}$&5&0.360447&160.516355&{\sc to}\\
$\IFF_{2, 3, 2, 4}$&5&10.335019&30.261006
&{\sc to}\\
$\IFF_{2, 3, 2, 4}$&6&0.004699&982.385689&{\sc to}\\
$\IFF_{2, 3, 2, 4}$&6&0.005279&820.318127&{\sc to}\\
$\IFF_{2, 3, 2, 4}$&6&0.005699&689.00884&{\sc to}\\
$\IFF_{2, 3, 2, 4}$&6&0.006159&559.359959&{\sc to}\\
$\IFF_{2, 3, 2, 4}$&6&0.044894&398.010023&{\sc to}\\
$\IFF_{2, 3, 2, 4}$&6&0.36543&237.51286&{\sc to}\\
$\IFF_{2, 3, 2, 4}$&6&10.335019&35.970266
&{\sc to}\\
$\IFF_{2, 3, 4, 0}$&2&1.939408&75.577512&138.79073\\
$\IFF_{2, 3, 4, 0}$&2&10.335019&11.969358
&138.79073\\
$\IFF_{2, 3, 4, 0}$&3&0.319778&210.292622&{\sc to}\\
$\IFF_{2, 3, 4, 0}$&3&1.884998&114.124223&{\sc to}\\
$\IFF_{2, 3, 4, 0}$&3&10.335019&17.86182
&{\sc to}\\
$\IFF_{2, 3, 4, 0}$&4&0.275622&408.491556&{\sc to}\\
$\IFF_{2, 3, 4, 0}$&4&0.307001&280.316039&{\sc to}\\
$\IFF_{2, 3, 4, 0}$&4&1.801674&152.52919&{\sc to}\\
$\IFF_{2, 3, 4, 0}$&4&10.335019&23.982122
&{\sc to}\\
$\IFF_{2, 3, 4, 0}$&5&0.307265&414.838747&{\sc to}\\
$\IFF_{2, 3, 4, 0}$&5&0.307885&286.888755&{\sc to}\\
$\IFF_{2, 3, 4, 0}$&5&1.799511&159.32949&{\sc to}\\
$\IFF_{2, 3, 4, 0}$&5&10.335019&29.811451
&{\sc to}\\
$\IFF_{2, 3, 4, 0}$&6&0.308439&230.064947&{\sc to}\\
$\IFF_{2, 3, 4, 0}$&6&1.782093&134.438962&{\sc to}\\
$\IFF_{2, 3, 4, 0}$&6&10.335019&36.041919
&{\sc to}\\
$\IFF_{2, 3, 4, 2}$&2&1.939408&75.688371&139.109953\\
$\IFF_{2, 3, 4, 2}$&2&10.335019&11.914726
&139.109953\\
$\IFF_{2, 3, 4, 2}$&3&0.319778&210.793983&{\sc to}\\
$\IFF_{2, 3, 4, 2}$&3&1.884998&114.750093&{\sc to}\\
$\IFF_{2, 3, 4, 2}$&3&10.335019&17.852247
&{\sc to}\\
$\IFF_{2, 3, 4, 2}$&4&0.275622&409.50455&{\sc to}\\
$\IFF_{2, 3, 4, 2}$&4&0.307001&281.18696&{\sc to}\\
$\IFF_{2, 3, 4, 2}$&4&1.801674&153.032512&{\sc to}\\
$\IFF_{2, 3, 4, 2}$&4&10.335019&23.946691
&{\sc to}\\
$\IFF_{2, 3, 4, 2}$&5&0.307265&416.466311&{\sc to}\\
$\IFF_{2, 3, 4, 2}$&5&0.307885&287.952489&{\sc to}\\
$\IFF_{2, 3, 4, 2}$&5&1.799511&160.001563&{\sc to}\\
$\IFF_{2, 3, 4, 2}$&5&10.335019&30.061209
&{\sc to}\\
$\IFF_{2, 3, 4, 2}$&6&0.308439&230.52634&{\sc to}\\
$\IFF_{2, 3, 4, 2}$&6&1.782093&134.488155&{\sc to}\\
$\IFF_{2, 3, 4, 2}$&6&10.335019&35.849211
&{\sc to}\\
\bottomrule
\end{tabular}
\end{table}

\subsection{Ablation study}
Table~\ref{tab:ablation} shows a comparison between our method, the ablated method, and the method from~\cite{bovymulti}.
Our results show that pruning is a significantly effective addition, as the ablated method runs quickly into the memory limit and the process gets killed by the operating system. We use a memory limit of 16GB.

In the table \textsc{TO} is timeout (over 3600s), and \textsc{MO} is memory-out (over 16GB).

\begin{table}[t]
\centering
\small
\caption{Ablation study on the same instances as Table~\ref{tab:rock-sample-main}.}
\label{tab:ablation}
\begin{tabular}{lrrrrrrrr}
\toprule
Model & $|\POMDPStates|$ & $|\POMDPActions|$ & $|\POMDPObs|$ & $n$ & $\horizon$ & Time ours (s) & Time without prunning (s) & Time\cite{bovymulti} (s)\\
\midrule
$\RS_{3, 1, 2}$ & 26 & 7 & 3 & 2 & 4 & 0.233528 & 0.237096 & 1.735581 \\
$\RS_{3, 1, 3}$ & 34 & 8 & 3 & 3 & 3 & 0.058124 & 0.054495 & 2.433545
 \\
$\RS_{4, 1, 2}$ & 32 & 7 & 3 & 2 & 7 &  154.249    & \textsc{mo} & \textsc{to} \\
$\RS_{6, 1, 2}$ & 35 & 7 & 3 & 2 & 6 & 18.6796   & 29.628089 & 43.165601\\
$\RS_{3, 2, 7}$ & 246 & 12 & 3 & 21 & 5 &  459.885    & \textsc{mo} & \textsc{to}  \\
$\RS_{3, 4, 7}$ & 574 & 12 & 3 & 35 & 5 & 783.78 & \textsc{mo} & \textsc{to}\\
$\RS_{3, 6, 7}$ & 192 & 12 & 3 & 7 & 6 & 3476.53 & \textsc{mo} & \textsc{to}\\
$\RS_{3, 1, 2}$ & 26 & 7 & 3 & 2 & 5 & 2.06783  & 4.915408 & 2.280673   \\
$\RS_{4, 1, 2}$ & 32 & 7 & 3 & 2 & 6 & 17.1989 & 59.258818 & 24.329221 \\
\bottomrule
\end{tabular}
\end{table}


\end{document}